\documentclass[12pt]{article}

\usepackage[left=2.7cm,bottom=2.7cm,right=2.7cm,top=2.7cm]{geometry}

\usepackage[utf8]{inputenc} % allow utf-8 input
\usepackage[T1]{fontenc}    % use 8-bit T1 fonts
\usepackage{lmodern}
\usepackage{microtype}
\usepackage{amsmath,amssymb,amsthm,mathtools}
\usepackage{array}
\usepackage{booktabs}
\usepackage{multirow}
\usepackage{graphicx}
\usepackage{xcolor}
\usepackage{authblk}
\usepackage{enumitem}
\usepackage{indentfirst}
\usepackage[bottom]{footmisc}
\usepackage{comment}
\usepackage{algorithm}
\usepackage{algpseudocode}
\usepackage{tikz}
\usepackage{pgfplots}
\usepackage[authoryear,round,sort]{natbib}
\usepackage{doi}
\usepackage{hyperref}
\hypersetup{hidelinks}
\ifdefined\pdfsuppresswarningpagegroup
\fi

\theoremstyle{plain}

\newtheorem{theorem}{Theorem}[section]
\newtheorem{lemma}[theorem]{Lemma}
\newtheorem{corollary}[theorem]{Corollary}

\theoremstyle{definition}
\newtheorem{remark}[theorem]{Remark}

\usetikzlibrary{calc,fit}
\usepgfplotslibrary{groupplots}
\pgfplotsset{compat=1.17}

\algrenewcommand\algorithmicrequire{\textbf{Require:}}
\algrenewcommand\algorithmicensure{\textbf{Ensure:}}

\begin{document}
\title{A Unified Descriptive-Complexity Framework for Model Selection under Correlated Designs}

\author[1]{Yanhang Zhang}
\author[2]{Wei Liu}
\author[1, 3]{Yuhong Yang}

\affil[1]{\footnotesize Yau Mathematical Sciences Center, Tsinghua University.}
\affil[2]{\footnotesize Citigroup.}
\affil[3]{\footnotesize Beijing Institute of Mathematical Sciences and Applications.}

\date{}
\maketitle \sloppy

\begin{abstract}
Model selection becomes particularly challenging under strong
predictor dependence and model-class uncertainty, especially when there are exponentially many models. We propose a
Descriptive-Complexity Information Criterion (DCIC) that regularizes
large candidate model collections through Kraft-admissible code
lengths. Under sub-Weibull noise, we establish selection consistency
through approximation-error separation without relying on RIP-type
conditions, together with nonasymptotic oracle risk bounds that remain
valid under model misspecification. The same coding principle places heterogeneous classes on a
common complexity scale at a small additional class-identification cost. This extension yields class--model recovery under suitable
identifiability conditions and risk adaptation across classes. We further develop a complexity-guided search
path that makes the computation--statistics trade-off explicit. Large
penalties yield polynomial-size retained search regions with high
probability, whereas smaller penalties sharpen the oracle risk
benchmark. Numerical experiments illustrate stable support recovery
and favorable estimation performance under strong dependence and
model-class uncertainty.
\end{abstract}


\section{Introduction}\label{sec:introduction}
Model selection has been a central problem in  statistics and machine learning for decades, giving rise to an enormous literature on its theory, methodology, computation and applications.
Despite thousands of papers published on the subject, in our opinion, several key issues remain barely or insufficiently addressed.

\subsection{Three questions}\label{sub:question}
\begin{enumerate}[label=\textbf{\arabic*.}, leftmargin=2.5em, labelsep=0.6em, itemsep=0.5em]
\item Existing selection consistency results often require restrictive near-orthogonality conditions on the design, including the restricted isometry property (RIP) \citep{candes2005} and the sparse Riesz condition (SRC) \citep{zhang2008}. How can one establish selection consistency and oracle risk guarantees under possibly strong predictor dependence that is commonly seen in high-dimensional data?

\item Most variable-selection procedures in regression work with a prespecified variable dictionary or basis family. In many applications, however, one naturally wants to explore several approximation systems. Is it possible to select across heterogeneous model classes to guarantee risk adaptation and, under correct model specification, achieve joint recovery of the true model class and its within-class model?

\item Scalable model-selection procedures typically use convex optimization or tractable approximations to nonconvex objectives. Best subset selection (BSS) offers strong guarantees for selection consistency and estimation risk but generally involves combinatorial search over exponentially many subsets. Can one develop a method that explicitly links the extent of model-space exploration to an oracle risk benchmark, even under model misspecification?
\end{enumerate}

The first question arises from the difficulty of distinguishing competing models under strong predictor dependence. Many existing analyses unfortunately rule out strongly dependent designs through restrictive design assumptions. In observational studies, however, strong and complex dependence is common, and concerns have been raised about RIP-type conditions \citep{wasserman2012}. Rather than requiring the observed design to satisfy near-orthogonality conditions, a different direction is to characterize when selection consistency remains attainable under the observed dependence structure and to derive meaningful risk guarantees even when exact identification is difficult or irrelevant.

A further question is model-class uncertainty. Competing structural classes may exhibit markedly different approximation behavior and complexity, so restricting selection to a potentially misspecified class can result in substantial approximation error. For example, many methods for sparse additive regression approximate the target function using a prespecified basis family \citep{Ravikumar2009,Meier2009,huang2010}, even though other basis families may be better suited.  These considerations motivate a unified framework for selecting among heterogeneous classes that achieves risk adaptation and, under correct model specification, recovers the true class and its within-class model. To the best of our knowledge, such a framework is absent from the current literature.

A third question concerns the trade-off between statistical accuracy and the computational cost of exploring large model spaces. Some existing data-driven tuning procedures attain optimal rates, typically under favorable design conditions and sparse or weakly sparse approximation regimes \citep{bellec2018SLOPE,ing2020}. These results primarily characterize approximation--estimation trade-offs along a prescribed solution path. Under strong predictor dependence and model misspecification, however, they do not directly quantify how  expanding the explored model region affects the attainable statistical benchmark.

Motivated by these considerations, we seek to develop a model selection
framework that accommodates general predictor dependence by
characterizing the distinguishability of competing models rather than
imposing restrictive near-orthogonality conditions. The framework aims
to provide meaningful risk guarantees under model misspecification and
to compare heterogeneous model classes on a common complexity scale.
We further seek to use its tuning parameter to control an admissible
complexity budget, thereby linking the extent of model-space
exploration to the attainable statistical benchmark.

\subsection{A descriptive-complexity framework and main results}
\label{sub:framework}
We address these questions through the Descriptive-Complexity Information Criterion (DCIC), a unified framework based on Kraft-admissible code lengths. To place the discussion in a concrete setting, let $X\in\mathbb R^{n\times p}$ be a fixed design matrix and consider the general mean model
$$
y=\mu_n+\varepsilon,
$$
where $y\in\mathbb R^n$ is the response vector, $\mu_n\in\mathbb R^n$ is an unknown mean vector, and $\varepsilon=(\varepsilon_1,\ldots,\varepsilon_n)^\top$
consists of independent, mean-zero sub-Weibull random variables satisfying,
for some $\alpha\in(0,2]$ and $\zeta>0$,
\[
\mathbb{E}(\varepsilon_i^2)=\sigma^2,
\qquad
\Pr\!\left(\left|\varepsilon_i/\sigma\right|>t\right)
\leq
2\exp\!\left\{-\left(t/\zeta\right)^\alpha\right\},
\]
for all $t\geq 2$ and $i=1,\ldots,n$
\citep{comminges2021}.
When pursuing the consistency results, we specialize to
$
\mu_n=X\beta^*,
$
where $\beta^*\in\mathbb R^p$ is $s^*$-sparse with support $S^*$. For the risk analysis, we allow $\mu_n$ to be arbitrary and do not require it to be represented by any candidate model or even to admit a linear representation in $X$.

Consider a collection of candidate linear models. For a model $S$ of
size $s$, define
\[
\mathrm{DCIC}(S)
:= \mathrm{RSS}(S)/\sigma^2
+ \eta_n\, s
+ \lambda\, C_S^{2/\alpha},
\]
where \(\eta_n>0\), \(\mathrm{RSS}(S)=\bigl\|\bigl(\mathrm{I}_n-P_S\bigr)y\bigr\|_2^2\) is the residual sum of squares of model $S$, \(C_S\) is the descriptive complexity of model \(S\), and \(\lambda \ge 0\) is a tuning parameter.  
The exponent \(2/\alpha\) reflects the tail behavior under sub-Weibull noise and the need to control deviations uniformly over a coded model class. In particular, \(\alpha=2\) corresponds to the sub-Gaussian noise setting. 
We assume that \(C_S\) satisfies the Kraft inequality
\[
\sum_{S \in \mathcal{M}_n} e^{-C_S} \leq 1,
\]
where \(\mathcal{M}_n\) is the candidate model class. \(C_S\) can be viewed as a description length assigned to the model index, and the Kraft inequality ensures global admissibility of the resulting code.  
This coding interpretation is
classical in information-theoretic model selection and minimum
description length
\citep{Barron1991,barron1998,yang1999selection}. Its broader
connections with oracle risk theory are reviewed in
Section~\ref{sub:literature}.

The same Kraft-admissible complexity provides a unified answer to the
three questions in Section~\ref{sub:question}. First, we establish a
general selection consistency result for a broad class of model collections through approximation-error separation, without
requiring RIP/SRC-type conditions, and verify the
resulting assumptions under representative strongly correlated
designs. The same criterion also yields nonasymptotic oracle risk
bounds that remain meaningful under model misspecification and imply minimax-adaptive rates in several
canonical settings.

Second, we extend DCIC to selection over $M$ heterogeneous candidate
classes. By encoding the class label at an additional cost of order
$\log M$, DCIC places models from all candidate classes on a common
descriptive-complexity scale while preserving a global Kraft
inequality. This yields guarantees for both exact class--model recovery
and risk adaptation to the best approximation--complexity trade-off
across classes. We illustrate this extension through adaptive
basis-family selection in sparse additive models.

Third, we use the same descriptive-complexity structure to develop a
complexity-guided search strategy with an explicit
computation--statistics trade-off. Along the decreasing-$\lambda$ path, larger penalties yield polynomial-size retained regions with high probability, whereas smaller penalties expand the search and sharpen the oracle risk benchmark under model misspecification. We further establish
finite-sample control of false positives and false negatives along the
path and identify an exact-recovery region under suitable beta-min
conditions.

\subsection{Motivating illustrations of the computation--statistics trade-off}\label{sub:illustration}
We give two single-run illustrations of the trade-off, with detailed
simulation protocols deferred to Section~\ref{sec:experiments}.
Figure~\ref{fig:runtime_sparse} reports the average squared error (ASE)
and Matthews correlation coefficient (MCC) under a well-specified sparse
model, whereas Figure~\ref{fig:runtime_dense} reports ASE, model size,
and the number of selected strong signals under a misspecified dense-signal
model. In both figures, markers show the terminal performance of
representative competitors, while DCIC traces a decreasing-$\lambda$ path.
The piecewise time axis uses seconds initially and minutes thereafter.
\begin{figure}[H]
\centering
\includegraphics[scale=0.39]{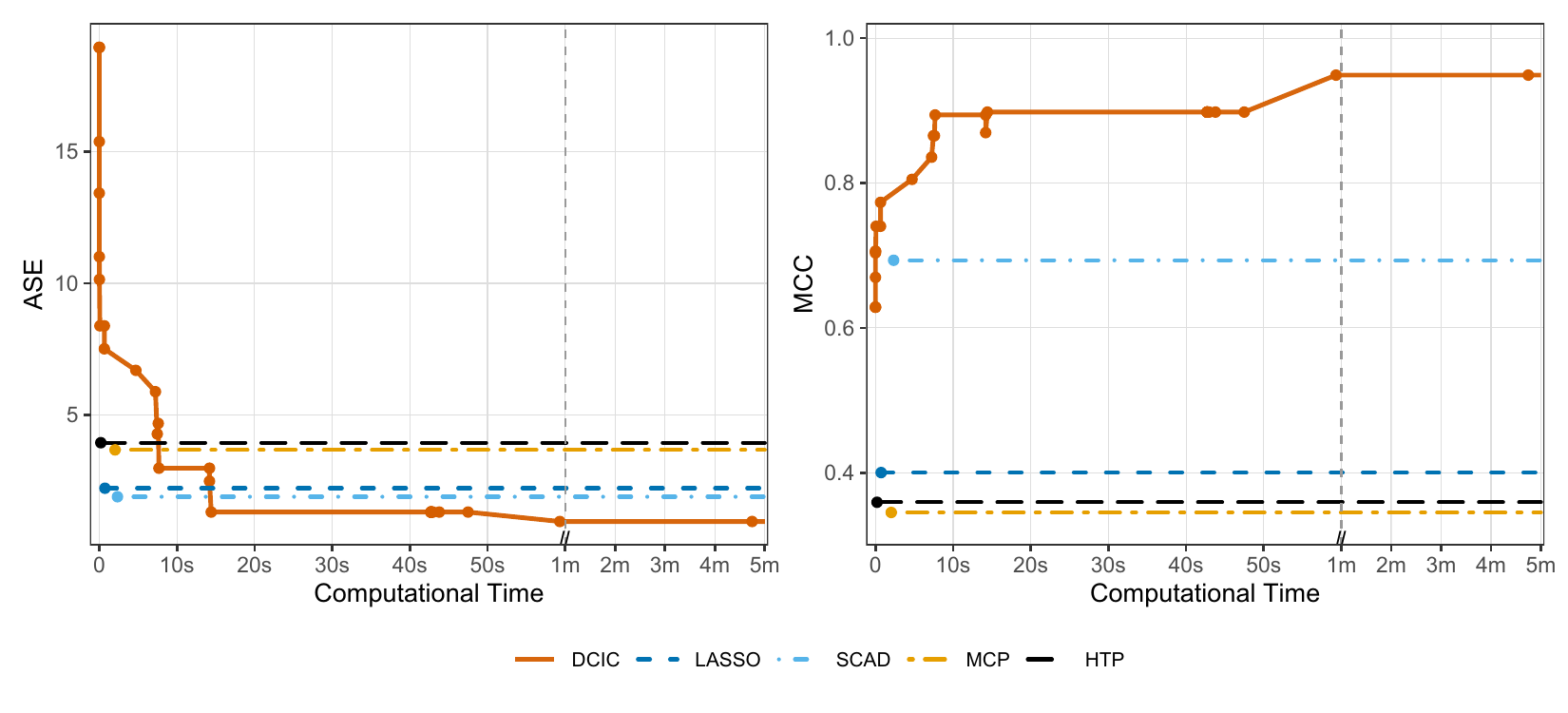}
\caption{
Computation--statistics trade-off under a well-specified, strongly
correlated design with $n=200$, $p=1000$, $s^*=20$, $\sigma=2$,
$\rho_{\mathrm{in}}=\rho_{\mathrm{out}}=0.9$, and $w=0.5$.
}
\label{fig:runtime_sparse}
\end{figure}
\begin{figure}[H]
\centering
\includegraphics[scale=0.39]{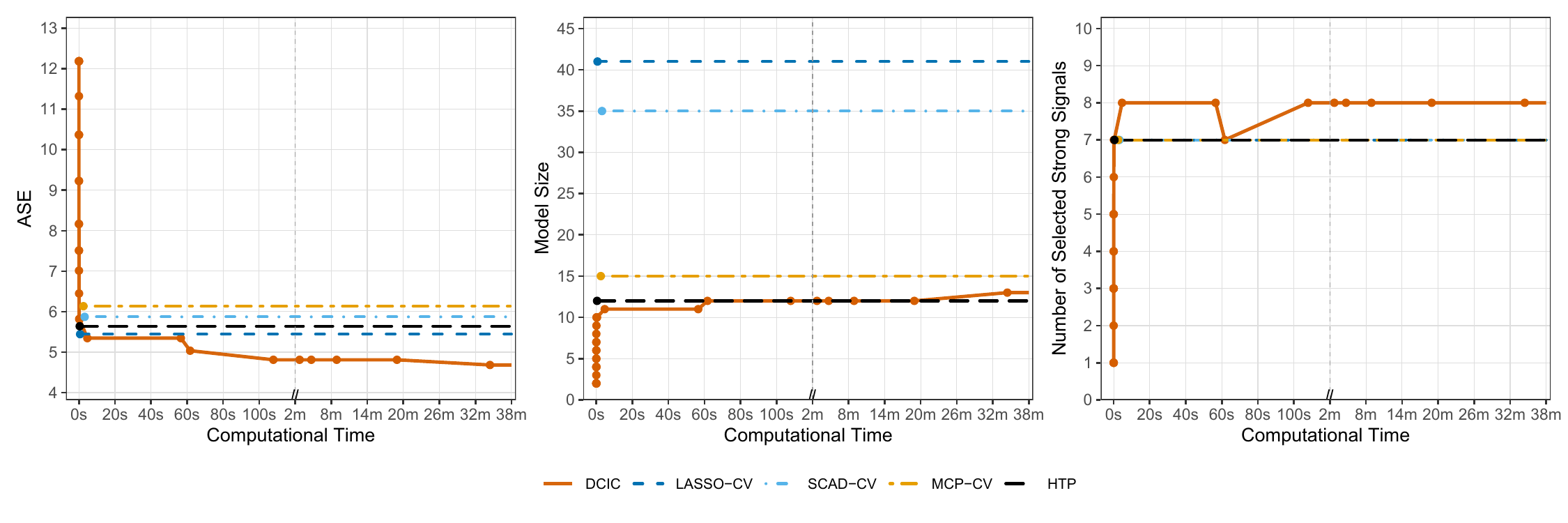}
\caption{
Computation--statistics trade-off under a misspecified dense-signal
model with $n=200$, $p=1000$, $\sigma=2$,
$\rho_{\mathrm{in}}=\rho_{\mathrm{out}}=0.7$, and $w=0.5$.
The coefficient vector contains $10$ strong, $50$ moderate, and $940$
weak signals. The variables are partitioned into ten equal-sized
blocks, each containing one strong and five moderate signals. Signal
magnitudes are drawn independently from
$\operatorname{Unif}(1,1.3)$,
$\operatorname{Unif}(0.2,0.4)$, and
$\operatorname{Unif}(0,0.1)$, respectively, with independent random
signs.
}
\label{fig:runtime_dense}
\end{figure}
The illustrations show how the DCIC path improves statistical performance
as the explored model region expands. In the well-specified setting, early
DCIC solutions are already competitive, while later path points further
improve support recovery. Under model misspecification, LASSO and SCAD tend
to select larger models, whereas MCP and HTP remain relatively sparse. Along
the DCIC path, the selected model gradually expands and ASE decreases.
Taken together, the illustrations suggest that DCIC provides more favorable
support recovery in the well-specified setting and identifies more strong
variables under model misspecification, at the cost of additional computation.
In both illustrations, DCIC attains lower ASE than the competing methods at
sufficiently small values of $\lambda$.
Section~\ref{sub:tradeoff} formalizes this pathwise computation--statistics
trade-off.

\subsection{Related literature}\label{sub:literature}
Model selection through penalized empirical criteria has long been
analyzed using oracle inequalities that characterize
approximation--estimation trade-offs
\citep{barron1999risk,birge2001gaussian,Birg2006}. This theory is closely
connected to information-theoretic coding and minimum description
length. Kraft-admissible code lengths provide uniform control over
large candidate collections and yield general oracle risk bounds
\citep{Barron1991,barron1998,yang1998,yang1999selection}.
Related coding-based penalties have also been studied for structured
sparse estimation \citep{huang2011JMLR}. In Gaussian
regression, complexity-penalized least squares further yields sharp
oracle inequalities and minimax-adaptive rates over rich model lists
\citep{verzelen2012minimax,wang2014adaptive}. These results primarily
address risk adaptation, and do not by themselves imply
selection consistency.

Selection consistency has been extensively studied through
information criteria. Classical BIC is consistent in fixed dimensions \citep{schwarz1978estimating}, but
penalties depending only on model dimension may fail to control the exponentially large model lists that
arise in high dimensions \citep{yang1998}, motivating various
BIC extensions
\citep{chen2008extended,wang2013,zhang2023}. Our goal is to combine complexity-based uniform control with selection consistency under possibly strong predictor dependence while retaining oracle risk guarantees under misspecification.

Earlier, \citet{shen2012,shen2013} developed a degree-of-separation theory for 
BSS under general designs and showed that, when the sparsity level is specified, 
the resulting separation boundary is necessary and attainable up to constants. 
Recent work further developed this separation-based perspective for BSS under 
correlated designs \citep{guo2021,zhu2024sure,roy2025ejs,gao2025}. 
In particular, \citet{guo2021} recast a closely related projection-gap condition 
as an identifiability margin and showed that, when $s^*$ is known, it is 
essentially necessary and sufficient for BSS consistency. 
\citet{roy2025ejs} clarified its dependence on design geometry and derived additional 
sufficient conditions, while \citet{zhu2024sure} established no false discoveries 
with high probability along the early BSS path for model sizes below $s^*$. 
\citet{gao2025} studied the optimal sample complexity of BSS with known and unknown 
$s^*$. When $s^*$ is unknown, their proposed criterion depends on typically 
unobservable quantities, including design-specific structure and minimum signal strength. 
Existing results mainly address all-subset selection, often with known sparsity or 
additional structural information. In contrast, DCIC accommodates unknown sparsity 
and broader model collections, including all-subset, group, and double-sparse selection, 
through a common Kraft-admissible complexity framework.

On the computational side, modern mixed-integer optimization (MIO) has substantially
improved the tractability of subset selection and its variants
\citep{bertsimas2016, bertsimas2020, Hazimeh2023, mazumder2023subset}.
Branch-and-bound algorithms produce feasible incumbents together with explicit
optimality gaps, allowing optimization error to be incorporated into statistical
guarantees. Many standard cardinality-based formulations assign the same
selection cost to all models of a given size. From a complementary perspective, DCIC assigns Kraft-admissible,
model-specific complexities and lets the tuning parameter determine a nonuniform
complexity budget. This provides a statistical coding principle for analyzing how
the retained search region and the oracle benchmark trade off along the tuning path,
rather than a replacement for modern MIO solvers.

More broadly, a growing literature studies model uncertainty and adaptation across multiple candidate classes. Bayesian model averaging addresses model uncertainty
through posterior averaging \citep{Hoeting1999}, while forecast combination aims to achieve optimal performance offered by the candidate procedures adaptively or even improve the best individual performance by aggregating them instead of selecting a single one
\citep{Yang2004}. \citet{donoho1994ideal} studied adaptive denoising over a
library of orthonormal bases and derived a near-oracle risk bound relative to the best basis. In nonparametric regression,
\citet{Sklar2013} selects basis functions from multiple libraries to
improve flexibility, but the resulting hybrid representations may
be difficult to interpret. High-dimensional additive
methods mainly focus on estimation within prescribed function classes
\citep{Tan2019}, rather than selection across heterogeneous classes.
Since different families may exhibit substantially different approximation
behavior, it remains unclear how to compare them on a common scale while
obtaining both selection and risk guarantees.

Finally, a substantial literature studies statistical guarantees along solution paths for model selection. \citet{bellec2018SLOPE} derives sharp oracle inequalities for the LASSO under misspecification and develops a Lepski-type procedure attaining minimax adaptation in well-specified sparse models. \citet{ing2020} analyzes the prediction error of the orthogonal greedy algorithm (OGA) under various sparsity regimes and selects among candidates along the full path, whereas \citet{stank2024} develops an early-stopping procedure with oracle inequalities along the nested OGA path, allowing dense signals within a linear model. Complementing these works, we develop a complexity-guided search strategy under possibly strong dependence and model misspecification, whose $\lambda$-path links the retained search size to an oracle risk benchmark and permits termination according to the computational budget.
\subsection{Notation and Organization}

Let $[p]:=\{1,\ldots,p\}$, and let $\mathcal M_n\subseteq 2^{[p]}$
denote the candidate model collection. For $S\in\mathcal M_n$, write
$s:=|S|$, let $X_S$ be the submatrix of $X$ indexed by $S$, and let
$P_S$ be the orthogonal projection onto $\operatorname{Col}(X_S)$, with
$r_S:=\operatorname{rank}(X_S)$. Define the normalized approximation error
\[
\operatorname{APP}(S)
:=
\|(\mathrm{I}_n-P_S)\mu_n\|_2^2/\sigma^2,
\]
where $\mathrm{I}_n$ is the $n\times n$ identity matrix.
Let $c_1,c_2,\ldots$ denote positive constants whose values may vary from line to line. 
%Unless otherwise stated, these constants depend only on $\alpha$ and $\zeta$, and this dependence is suppressed for simplicity.

The remainder of the paper is organized as follows.
Sections~\ref{sec:consistency} and~\ref{sec:risk} establish selection consistency
and risk bounds for DCIC under general predictor dependence.
Section~\ref{sec:multiclass} extends the framework to model-class uncertainty,
and Section~\ref{sec:strategy} develops a complexity-guided search strategy.
Section~\ref{sec:experiments} reports simulations, and Section~\ref{sec:conclusion}
provides a summary of our work. Technical proofs and additional simulations are deferred to the appendices.

\section{Model selection consistency of DCIC}\label{sec:consistency}
We study exact recovery in the well-specified setting. Assume that there
exists a true model $S^*\in\mathcal M_n$ with $s^*:=|S^*|$.
Assume that the dictionary index set $[p]$ is partitioned into nonempty
selection atoms $\mathcal U=\{U_1,\ldots,U_q\}$
and every $S\in\mathcal M_n$ is an admissible union of atoms. Let
$
\mathcal U(S):=\{U\in\mathcal U:U\subseteq S\}$
and
$
S=\bigcup_{U\in\mathcal U(S)}U.
$
Let $\mathcal U^*:=\mathcal U(S^*)$ and $q^*:=|\mathcal U^*|$.
The atoms specify the basic selection units, while $\mathcal M_n$ specifies the admissible unions of atoms.
This atom-level formulation yields a general consistency framework for structured model collections representable as admissible atom unions.
The examples below correspond to different choices of atoms and admissible unions.

We partition the models into the overfitted, underfitted, and
wrong classes,
\[
\mathcal M_{\mathrm{sup}}
:=
\{S\in\mathcal M_n:S^*\subsetneq S\},
\qquad
\mathcal M_{\mathrm{sub}}
:=
\{S\in\mathcal M_n:S\subsetneq S^*\},
\]
and
\[
\mathcal M_{\mathrm w}
:=
\mathcal M_n
\setminus
\bigl(
\mathcal M_{\mathrm{sup}}\cup\mathcal M_{\mathrm{sub}}\cup\{S^*\}
\bigr),
\]
respectively.
Define
$
\Delta^-(S)
:=
|\mathcal U^*\setminus\mathcal U(S)|$,
$
\Delta^+(S)
:=
|\mathcal U(S)\setminus\mathcal U^*|
$
and
$\widetilde r_S:=\operatorname{rank}(X_{S \cap S^*})$.
For $\ell\in [q^*]$, define
$
N(\ell)
\;:=\;
\Big|\big\{\, S\in\mathcal M_{\mathrm{sub}}:\ \Delta^-(S)=\ell \,\big\}\Big|.
$
Define 
$
u_S=\eta_n(s-s^*)+\lambda(C_S^{2/\alpha}-C_{S^*}^{2/\alpha}).$
For $S \in \mathcal{M}_\mathrm{w}$, define $B_{S} = \frac{1}{4}\lambda D_S^{2/\alpha}+\frac{1}{4}\log\log n+2(r_S-\tilde r_S),$ where
$$
D_{S} =  \Biggl\{\log\binom{q^*}{\Delta^-(S)}\binom{q-q^*}{\Delta^+(S)}\Biggr\}\wedge C_S.$$
Here $D_{S}$ uses the sharper of a local combinatorial count and the global Kraft complexity to control the multiplicity of wrong models.

\subsection{Model selection consistency of DCIC}

\begin{theorem}\label{thm:complexity}
Assume that $\eta_n \to \infty$ and that the descriptive complexity satisfies the Kraft inequality
$
\sum_{S\in\mathcal M_n} e^{-C_S}\le 1.
$
There exists a constant $\lambda_{\mathrm{sel}}$ depending only on $\alpha$ and $\zeta$ such that, for every $\lambda \geq \lambda_{\mathrm{sel}}$, if the following conditions hold:
\begin{itemize}
    \item[(\romannumeral 1)]
For every \(S\in\mathcal M_{\mathrm{sup}}\), \(C_S\ge C_{S^*}\). In addition, there
exist constants \(0<\tilde c_1<c_1\) such that, for every \(j\ge 1\),
\begin{equation}\label{eq:assumption1}
\sum_{S\in\mathcal M_{\mathrm{sup}}, \Delta^+(S)=j}
\exp\!\left\{
-c_1\lambda^{\frac{\alpha}{2}}\bigl(C_S-C_{S^*}\bigr)
\right\}
\le
\exp\!\left\{\tilde c_1(\eta_n j)^{\frac{\alpha}{2}}\right\}.
\end{equation}

    \item[(\romannumeral 2)]
    There exists a sufficiently large constant $c_2>1$ such that, for every
    $S\in\mathcal M_{\mathrm{sub}}$,
    \begin{equation}\label{eq:subAPP_thm}
    \mathrm{APP}(S)\ge c_2\Bigg(
    \eta_n(s^*-s)
    +\lambda\bigl(C_{S^*}^{\frac{2}{\alpha}}-C_S^{\frac{2}{\alpha}}\bigr)_+
    +\log^{\frac{2}{\alpha}} N\Bigl(\Delta^-(S)\Bigr)
    \Bigg).
    \end{equation}

    \item[(\romannumeral 3)]
    For every $S\in\mathcal M_w$ satisfying $2B_S>u_S$, we have
    $\mathrm{APP}(S)\ge 2(2B_S-u_S).$
\end{itemize}
Then
\[
\Pr\!\left(
\min_{S\in\mathcal M_n,\;S\neq S^*}\mathrm{DCIC}(S)>\mathrm{DCIC}(S^*)
\right)\to 1
\qquad \text{as } n\to\infty.
\]
\end{theorem}
\begin{remark}\label{re:kraft} 
The three assumptions in Theorem~\ref{thm:complexity} correspond to overfitted,
underfitted, and wrong competitors relative to \(S^*\).
Assumption (\romannumeral 1) controls \(S\in\mathcal M_{\rm sup}\) through
Kraft-type summability induced by the complexity gap \(C_S-C_{S^*}\).
Assumption (\romannumeral 2) excludes \(S\in\mathcal M_{\rm sub}\) through
lower bounds on \(\mathrm{APP}(S)\), reflecting the signal loss from omitted
true atoms. Assumption (\romannumeral 3) treats \(S\in\mathcal M_w\) in two
stages: sufficiently large penalty gaps rule out models automatically, while
borderline cases require additional approximation-error separation.
All results remain valid under
$
\sum_{S\in\mathcal M_n} e^{-C_S}\le K_0
$
for any fixed finite \(K_0>0\), with constants allowed to depend on \(K_0\).
\end{remark}
\begin{corollary}\label{cor:complexity}
Let $\{C_k\}_{k=1}^q$ be nonnegative atomic code lengths satisfying
$
\sum_{k=1}^q e^{-C_k}\le 1 .
$
For each model $S\in\mathcal M_n$, define its descriptive complexity by
$
C_S:=1+\sum_{k: U_k\subseteq S} C_k .
$
For every $\lambda\ge \lambda_{\rm sel}$,
if assumptions \textnormal{(\romannumeral 2)}--\textnormal{(\romannumeral 3)} of
Theorem~\ref{thm:complexity} hold, then
\[
\Pr\left(
\min_{S\in\mathcal M_n,\ S\neq S^*}
\mathrm{DCIC}(S)>\mathrm{DCIC}(S^*)
\right)\to 1
\qquad \text{as } n\to\infty .
\]
\end{corollary}
Thus, additive atom-level codes automatically control overfitted models, leaving only the APP bounds for underfitted and wrong competitors to be verified.

\subsection{Consistency for all-subset selection}\label{sub:subset}
In this setting, each atom corresponds to a single covariate, so the selected atom set is simply \(\mathcal{U}(S)=S\) and the true atom set is \(\mathcal{U}^*=S^*\). In particular,
$
\Delta^{-}(S)=|S^*\setminus S|$
and
$\Delta^{+}(S)=|S\setminus S^*|.
$
Moreover, the counting function becomes
$
N(\ell)=\binom{s^*}{\ell}, \ell=1,\ldots,s^*.
$
We take $\eta_n=\log n$ and $C_S=s\log p$, which is Kraft-admissible
by Remark~\ref{re:kraft}.

\begin{corollary}\label{cor:all_subset}
For every $\lambda \geq \lambda_{\rm sel}$, if the following conditions hold:
\begin{itemize}
    \item[(\romannumeral 1)]
    There exists a sufficiently large constant $c_1>1$ such that, for every
    $S\in\mathcal M_{\mathrm{sub}}$,
    \begin{equation}\label{eq:suball}
    \mathrm{APP}(S)\ge
    c_1\Bigg(
    \Delta^-(S)\log n
    +\lambda\Bigl((s^*\log p)^{\frac{2}{\alpha}}-(s\log p)^{\frac{2}{\alpha}}\Bigr)
    +\log^{\frac{2}{\alpha}}\binom{s^*}{\Delta^-(S)}
    \Bigg).
    \end{equation}

    \item[(\romannumeral 2)]
    There exists a constant $\kappa>0$ such that
    $
    2B_S<u_S
    $ for all $S\in\mathcal M_w \text{ with } s>(1+\kappa)s^*.
    $
    Moreover, for every $S\in\mathcal M_w$ with $s\le (1+\kappa)s^*$,
\begin{equation}\label{eq:wrongall}
\begin{aligned}
\mathrm{APP}(S)\ge{}&
\lambda\left\{
\log
\binom{s^*}{\Delta^-(S)}
\binom{p-s^*}{\Delta^+(S)}
\right\}^{\frac{2}{\alpha}}
+2(s^*-s)\log n
\\
&+2\lambda\left\{
(s^*\log p)^{\frac{2}{\alpha}}
-(s\log p)^{\frac{2}{\alpha}}
\right\}
+\log\log n
+8(r_S-\widetilde r_S).
\end{aligned}
\end{equation}
\end{itemize}
Then
\[
\Pr\!\left(
\min_{S\in\mathcal M_n,\;S\neq S^*}\mathrm{DCIC}(S)>\mathrm{DCIC}(S^*)
\right)\to 1
\qquad \text{as } n\to\infty.
\]
\end{corollary}
\begin{remark}
In the sub-Gaussian regime with $s=s^*$,
\citet{guo2021} showed that the
identifiability margin
\[
\inf_{\substack{S\neq S^*,\,s=s^*}}
\frac{\mathrm{APP}(S)}{\Delta^+(S)}
\gtrsim \log p
\]
is sufficient and nearly necessary for BSS consistency.
Since $\Delta^+(S)=\Delta^-(S)$ when $s=s^*$, this verifies
\eqref{eq:wrongall}, whose lower bound is of order
$\Delta^+(S)\log p$. Hence, the setting studied by \citet{guo2021} is a fixed-size special case
of our APP-based separation framework.
\end{remark}

\begin{remark}
An alternative is $C_S=s\log(ep/s)+\log p$, which matches the EBIC
combinatorial penalty in order when $\alpha=2$, up to constants and
additive terms. Compared with $s\log p$, it penalizes large models less
heavily. If $s^*$ diverges, consistency typically requires
$\log(ep/s^*)\asymp\log p$ or a stronger dimension penalty through
$\eta_n$. When $s^*$ is fixed, the two choices are equivalent in
order, which is the regime considered by
\citet{chen2008extended}.
\end{remark}

\subsection{Consistency for group selection}\label{sub:group}

Let the $p$ predictors be partitioned into $m$ disjoint groups
$G_1,\ldots,G_m$, with $|G_k|=d$. Equivalently, $p = m \times d$ and $s^* = g^*d$.
A candidate model is indexed by its selected group set $G\subseteq[m]$, with
$g:=|G|$. Thus,
$
\mathcal{U}(S)=G, \mathcal{U}^*=G^*,
\Delta^+(S)=|G\setminus G^*|,
\Delta^-(S)=|G^*\setminus G|,
$
and
$N(\ell)=\binom{g^*}{\ell}$ for $\ell\in [g^*]$.
We take $\eta_n = \log n$ and  $C_S=g\log m$, which is Kraft-admissible by Remark \ref{re:kraft}.
\begin{corollary}\label{cor:group_subset}
For every $\lambda \geq \lambda_{\rm sel}$, if the following conditions hold:
\begin{itemize}
    \item[(\romannumeral 1)]
    There exists a sufficiently large constant $c_1>1$ such that, for every
    $S\in\mathcal M_{\mathrm{sub}}$,
    \begin{equation*}
    \mathrm{APP}(S)\ge
    c_1\Bigg(
    (s^*-s)\log n
    +\lambda\Bigl((g^*\log m)^{\frac{2}{\alpha}}-(g\log m)^{\frac{2}{\alpha}}\Bigr)
    +\log^{\frac{2}{\alpha}}\binom{g^*}{\Delta^-(S)}
    \Bigg).
    \end{equation*}

\item[(\romannumeral 2)]
There exists a constant \(\kappa>0\) such that
\(2B_S<u_S\) for all \(S\in\mathcal M_w\) with
\(g>(1+\kappa)g^*\).
Moreover, for every \(S\in\mathcal M_w\) with
\(g\le(1+\kappa)g^*\),
\begin{equation*}
\begin{aligned}
\mathrm{APP}(S)\ge{}&
\lambda\left\{
\log
\binom{g^*}{\Delta^-(S)}
\binom{m-g^*}{\Delta^+(S)}
\right\}^{\frac{2}{\alpha}}
+2(s^*-s)\log n
\\
&+2\lambda\left\{
(g^*\log m)^{\frac{2}{\alpha}}
-(g\log m)^{\frac{2}{\alpha}}
\right\}
+\log\log n
+8(r_S-\widetilde r_S).
\end{aligned}
\end{equation*}
\end{itemize}
Then
\[
\Pr\!\left(
\min_{S\in\mathcal M_n,\;S\neq S^*}\mathrm{DCIC}(S)>\mathrm{DCIC}(S^*)
\right)\to 1
\qquad \text{as } n\to\infty.
\]
\end{corollary}

%%%%%%%%%%%%%%%%%%%%%%%%%%%%%%%%%%%%%%%%%%%%%%%%%%%%%%%%%%%%%%%%%%%%%%%%%%%%
\subsection{Consistency for double-sparse structure}\label{sub:ds}

For a candidate model $S$, define
$
G:=\{k\in[m]:S\cap\mathcal G_k\neq\emptyset\},
\
g:=|G|
$ and $
t_k(S):=|S\cap\mathcal G_k|$
for $ k\in G.
$
Unlike group selection, only a sparse subset of variables may be active within each
selected group. Treating individual variables as the selection atoms while retaining
the group structure in the admissible model collection, we define
\[
C_S
:=
g\log m+s\log d
+2\log(g+1)
+2\sum_{k\in G}\log(t_k+1).
\]
The four terms encode, respectively, the selected groups, the selected variables
within active groups, the number of active groups, and the within-group subset sizes,
while ensuring Kraft admissibility.
Theorem~\ref{thm:complexity} then yields selection consistency under the
corresponding APP-separation conditions. A precise statement and proof are provided in Appendix~\ref{app:double_sparse}.

\subsection{Representative designs with strong dependence}\label{sub:example}
In this section, we verify the APP separation requirement in Corollary~\ref{cor:all_subset} and derive explicit beta-min conditions under three representative correlated designs.

\paragraph*{Scenario 1 (Equicorrelation)}
Assume that the predictors are centered and standardized, and that the sample Gram matrix
$
\widehat\Sigma := X^\top X/n = (1-\omega)\mathrm I_p + \omega \mathbf{1}_p \mathbf{1}_p^\top
$
for some \(\omega \in [0,1)\), where \(\mathbf{1}_p\) denotes the \(p\)-dimensional all-ones vector.

\paragraph*{Scenario 2 (Weak alignment with spurious predictors)}
Assume that
$
\lambda_{\min}\!\bigl(\widehat\Sigma_{S^*S^*}\bigr)\ \ge\ \lambda_{*}\ >\ 0.
$
Moreover, there exist constants $\kappa_0\geq 1$ and $\epsilon_0\in(0,1)$ such that
\[
\sup_{\substack{|S|\le (1+\kappa_0) s^*\\S^*\backslash S \neq \emptyset}}
\frac{\|P_{S} X_{S^*\setminus S}\beta^*_{S^*\setminus S}\|_2^2}{\|X_{S^*\setminus S}\beta^*_{S^*\setminus S}\|_2^2}
\ \le\ 1-\epsilon_0.
\]
This condition is closely related to quantities based on residualized signals considered by \citet{roy2025ejs} for BSS with known sparsity under correlated designs.

\paragraph*{Scenario 3 (Shadow predictors)}
Assume that \(\widehat\Sigma_{S^*S^*}=\mathrm I_{s^*}\). For each \(j\in S^*\), there exists a
shadow predictor \(\tilde j\in (S^*)^c\) such that
\(\widehat\Sigma_{j,\tilde j}=\omega\) for some \(\omega\in(0,1)\). Let
\(\widetilde S=\{\tilde j:j\in S^*\}\), with \(|\widetilde S|=s^*\). These
\(s^*\) signal--shadow pairs are mutually uncorrelated and are also uncorrelated
with all remaining predictors.

These scenarios permit strong predictor dependence under which RIP conditions may fail, while the $\mathrm{APP}(S)$-based separation requirement remains verifiable. Let
$\beta_{\min}:=\min_{j\in S^*}|\beta_j^*|$, and throughout this subsection take
$\alpha=2$ and $\eta_n=\log n$. 

\begin{theorem}\label{thm:example}
For every $\lambda \geq \lambda_{\rm sel}$, under Scenarios 1--3, there exists a constant \(c_\lambda>0\) such that the DCIC minimizer is selection consistent under the following  conditions:
$$
\textnormal{(\romannumeral 1)}\ \beta_{\min}^2 \ge c_\lambda\sigma^2 \frac{\log p+\log n}{(1-\omega)n},
\
\textnormal{(\romannumeral 2)}\ \beta_{\min}^2 \ge c_\lambda\sigma^2 \frac{\log p+\log n}{\epsilon_0\lambda_* n},
\
\textnormal{(\romannumeral 3)}\ \beta_{\min}^2 \ge c_\lambda\sigma^2 \frac{\log p+\log n}{(1-\omega^2)n}.
$$
\end{theorem}
When $\omega$, $\lambda_*$, and $\epsilon_0$ are treated as constants,
the beta-min conditions in Theorem~\ref{thm:example} are of order
$\sigma^2(\log p+\log n)/n$. Whenever $\log n=O(\log p)$, this
reduces to the optimal beta-min scale $\sigma^2\log p/n$ for
selection consistency in subset selection
\citep{butucea2018variable}. When $\omega\asymp 1-1/s^*$, the RIP and
SRC conditions fail in Scenarios 1 and 3, so RIP/SRC-based consistency
analyses are no longer applicable in these strongly correlated regimes
\citep{zhang2010,wang2013,Zhu202014241,zhang2025}. By contrast, the
$\mathrm{APP}(S)$-based separation conditions remain verifiable. In
particular, Theorem~\ref{thm:example} yields
$\beta_{\min}^2\ge
c\sigma^2s^*(\log p+\log n)/n$ for Scenarios 1 and 3, which simplifies
to $c\sigma^2s^*\log p/n$ when $\log n=O(\log p)$. Therefore, DCIC
remains selection consistent in these strongly correlated regimes.

\section{Oracle risk bounds under model misspecification}\label{sec:risk}
In this section, we establish nonasymptotic risk bounds for DCIC under model misspecification. Unlike the consistency analysis in Section \ref{sec:consistency}, we do not require the existence of a distinguished true model $S^*$. This risk perspective is particularly relevant when the candidate collection is misspecified or when strong dependence makes exact model recovery too stringent.

Let $\widehat S$ denote the model selected by DCIC, and for each model $S$,
let $\widehat\mu_S=P_Sy$ denote the least-squares estimator of $\mu_n$.
Define its average squared error by
$
\mathrm{ASE}(S):=\frac{1}{n}\|\mu_n-\widehat\mu_S\|_2^2.
$ Denote
\begin{equation*}\label{eq:Rn_def}
R_n(\mu_n;S)
=\frac{1}{n}\|(\mathrm{I}_n-P_S)\mu_n\|_2^2
+\frac{(s\eta_n-r_S)\sigma^2}{n}
+\frac{\lambda\sigma^2 C_S^{2/\alpha}}{n},
\end{equation*}
and $
R_n^*(\mu_n;\mathcal M_n)=\min_{S\in\mathcal M_n} R_n(\mu_n;S).
$
The quantity $R_n^*(\mu_n;\mathcal M_n)$ characterizes the optimal trade-off between approximation error and model
complexity over the candidate set $\mathcal M_n$ and is referred to as the index of resolvability in the
literature \citep{Barron1991, barron1998, yang1999selection}.

\begin{theorem}\label{thm:risk}
Assume that
$
\sum_{S\in\mathcal M_n}e^{-C_S}\le 1
\ \text{and}\
\eta_n\ge 2.
$
Then there exist positive constants
$\lambda_{\mathrm{risk}}$, $c_{1,\alpha,\zeta}$, and
$c_{2,\alpha,\zeta}$, depending only on $\alpha$ and $\zeta$, such
that, for every $\lambda\ge\lambda_{\mathrm{risk}}$ and every
$\delta\in(0,1)$, with probability at least $1-\delta$,
\[
\mathrm{ASE}(\widehat S)
\le
2R_{n}^*(\mu_n;\mathcal M_n)
+
c_{1,\alpha,\zeta}
\frac{\sigma^2}{n}
\left\{\log\left(4/\delta\right)\right\}^{2/\alpha}.
\]
Consequently,
\[
\mathbb E\!\left\{
\mathrm{ASE}(\widehat S)
\right\}
\le
2R_{n}^*(\mu_n;\mathcal M_n)
+
c_{2,\alpha,\zeta}\frac{\sigma^2}{n}.
\]
\end{theorem}
Theorem~\ref{thm:risk} shows that the DCIC risk is controlled by the best
approximation--complexity trade-off over $\mathcal M_n$. We specialize this
bound to well-specified sparse models with $\eta_n=\log n$ and $\alpha=2$.
For all-subset selection, taking $C_S=s\log(ep/s)+\log p$ and
$\log n=O(\log(ep/s^*))$ gives
\begin{equation*}\label{eq:ase_subset_special}
\mathbb E\!\left\{\mathrm{ASE}(\widehat S)\right\}
\le
c\,\sigma^2\frac{s^*\log(ep/s^*)}{n},
\end{equation*}
which is minimax optimal over $s^*$-sparse linear models without design
assumptions \citep{raskutti2011minimax,wang2014adaptive}.
For group selection, $C_S=g\log(em/g)+\log m$ yields
\[
\mathbb E\!\left\{\mathrm{ASE}(\widehat S)\right\}
\le
c\,\sigma^2
\frac{g^*\log(em/g^*)+g^*d\log n}{n},
\]
matching the minimax rate up to a logarithmic factor in the within-group
estimation term \citep{tsy2011}.
For double-sparse selection, take
$C_S=g\log(em/g)+s\log(egd/s)+2\log(g+1)
+2\sum_{k\in G}\log(t_k+1)$.
If
$s^*\log n\lesssim g^*\log(em/g^*)+s^*\log(eg^*d/s^*)$, then
\[
\mathbb E\!\left\{\mathrm{ASE}(\widehat S)\right\}
\le
c\,\sigma^2
\frac{g^*\log(em/g^*)+s^*\log(eg^*d/s^*)}{n},
\]
which is minimax optimal for double-sparse structures
\citep{cai2019sparse,li2024}.

\section{Model-class uncertainty}
\label{sec:multiclass}
DCIC extends to heterogeneous candidate classes by adding a $\log M$
class-identification cost while preserving Kraft admissibility.

\subsection{Multi-class extension of DCIC}
\label{subsec:multiclass_dcic}

Let $\{\mathcal M_n^{(m)}\}_{m=1}^M$ be candidate model classes with
Kraft-admissible complexities $C_S^{(m)}$, that is,
\[
\sum_{S\in\mathcal M_n^{(m)}} \exp\{-C_S^{(m)}\}\le 1,
\qquad m=1,\ldots,M.
\]
We aggregate candidates across classes using class--model pairs \((m,S)\) and define
\begin{equation*}
\widetilde{\mathcal{M}}_n
:= \bigcup_{m=1}^M \big\{(m,S): S\in\mathcal{M}_n^{(m)}\big\}.
\label{eq:Mn_union}
\end{equation*}
For each pair $(m,S)$, define its descriptive complexity as
\begin{equation}
C_{(m,S)}
:=
C_S^{(m)}+\log M,
\qquad
(m,S)\in\widetilde{\mathcal{M}}_n.
\label{eq:Cms_def}
\end{equation}
By the Kraft inequalities,
\[
\sum_{(m,S)\in\widetilde{\mathcal{M}}_n}
\exp\{-C_{(m,S)}\}
=
\frac{1}{M}
\sum_{m=1}^{M}
\sum_{S\in\mathcal{M}_n^{(m)}}
\exp\{-C_S^{(m)}\}
\le 1.
\]
Thus, the aggregated complexity remains Kraft-admissible. The additional
$\log M$ term represents the cost of identifying one class among $M$
candidate classes.

For each $(m,S)\in\widetilde{\mathcal{M}}_n$, let
$\mathrm{RSS}(m,S)$ and $s_{(m,S)}$ denote its residual sum of squares and number of selected variables, respectively. The multi-class DCIC is
\begin{equation}
\mathrm{DCIC}(m,S)
:= \mathrm{RSS}(m,S)/\sigma^2 + \eta_n\, s_{(m,S)} + \lambda\, C_{(m,S)}^{\,2/\alpha},
\qquad (m,S)\in\widetilde{\mathcal{M}}_n,
\label{eq:dcic_multiclass}
\end{equation}
and we select
\begin{equation}
(\widehat m,\widehat S)\in \arg\min_{(m,S)\in\widetilde{\mathcal{M}}_n}\mathrm{DCIC}(m,S).
\label{eq:select_multiclass}
\end{equation}

Analogously to the single-class decomposition of \(\mathcal{M}_n\) in Section~\ref{sec:consistency}, define
\[
\begin{aligned}
\widetilde{\mathcal{M}}_{\mathrm{sup}}
&:=\bigl\{(m^*,S) \in \widetilde{\mathcal{M}}_n: S^*\subset S,\ S\neq S^*\bigr\},\\
\widetilde{\mathcal{M}}_{\mathrm{sub}}
&:=\bigl\{(m^*,S) \in \widetilde{\mathcal{M}}_n: S\subset S^*,\ S\neq S^*\bigr\}.
\end{aligned}
\]
These are the overfitted and underfitted pairs within the true class $m^*$, while
\[
\widetilde{\mathcal{M}}_{\mathrm{w}}
:=
\widetilde{\mathcal{M}}_n\setminus
\bigl(\widetilde{\mathcal{M}}_{\mathrm{sup}}\cup\widetilde{\mathcal{M}}_{\mathrm{sub}}\cup\{(m^*,S^*)\}\bigr)
\]
contains all remaining wrong pairs, including those with $m\neq m^*$.

Applying Theorem \ref{thm:complexity} to the aggregated list of pairs gives the
multi-class consistency result. A precise formulation is given
in Appendix~\ref{app:multiclass_consistency}.

\begin{corollary}[Multi-class selection consistency]
\label{cor:multiclass_consistency}
Suppose there exists a unique true pair $(m^*,S^*)\in\widetilde{\mathcal{M}}_n$, where $m^*$ denotes the true model class and $S^*$ the true model within class $m^*$. For every $\lambda \geq \lambda_{\rm sel}$, assume that the conditions in Theorem~\ref{thm:complexity} hold on the aggregated candidate list $\widetilde{\mathcal{M}}_n$, indexed by pairs \((m,S)\) with complexity
\(C_{(m,S)}\). Then
\[
\Pr\bigl\{(\widehat m,\widehat S)=(m^*,S^*)\bigr\}\to 1.
\]
\end{corollary}
Corollary~\ref{cor:multiclass_consistency} concerns exact pair recovery only when the true class-model pair is identifiable. When several classes yield statistically indistinguishable representations, risk adaptation is the more natural target.

For each $(m,S)\in\widetilde{\mathcal{M}}_n$, let $\mathrm{ASE}(m,S)$ be the class-wise extension of $\mathrm{ASE}(S)$
in Section \ref{sec:risk}, interpreted within class $m$. 
Similarly, let $R_n(\mu_n;m,S)$ be the class-wise extension of $R_n(\mu_n;S)$ with $C_S$
replaced by $C_{(m,S)}$, and extend the index of resolvability to the multi-class setting by
\[
R_n^*(\mu_n;\widetilde{\mathcal{M}}_n):=\min_{(m,S)\in\widetilde{\mathcal{M}}_n} R_n(\mu_n;m,S).
\] 
Under the same aggregated formulation, the risk bound in Theorem~\ref{thm:risk} extends directly to the multi-class setting.
\begin{corollary}[Multi-class risk bound]
\label{cor:multiclass_risk}
Assume that
$
\sum_{(m,S)\in\widetilde{\mathcal M}_n}
e^{-C_{(m,S)}}\le 1
\ \text{and}\
\eta_n\ge 2.
$
Then, there exist positive constants
$\lambda_{\rm risk}$ and $c_{1,\alpha,\zeta}$ such that whenever
$
\lambda\ge \lambda_{\rm risk},
$
we have
\begin{equation}
\label{eq:multiclass_risk_oracle}
\mathbb E\!\left\{
\mathrm{ASE}(\widehat m,\widehat S)
\right\}
\le
2R_n^*(\mu_n;\widetilde{\mathcal M}_n)
+
c_{1,\alpha,\zeta}\frac{\sigma^2}{n}.
\end{equation}
\end{corollary}
Compared with the single-class setting in Section~\ref{sec:risk}, the multi-class
extension introduces an additional \(\log M\) term in the descriptive complexity,
leading to an extra term of order \((\log M)^{2/\alpha}/n\) up to constants in the risk bound.
This is the price of selecting among \(M\) candidate classes. 

\subsection{Basis-family selection in nonparametric sparse additive models}
\label{sec:basis-multiclass}

We illustrate the multi-class formulation through sparse additive
regression with basis-family uncertainty, where the class label
identifies a basis family and the within-class model specifies the
active components and their basis representations.

Let $(y_i,x_{i1},\ldots,x_{ip})_{i=1}^n$ be observations from
\[
y_i
=
\sum_{j\in G^*} f_j^*(x_{ij})+\varepsilon_i,
\]
where $G^*\subset[p]$ is the unknown active set with
$|G^*|=g^*\ll p$. Assume that each $X_j$ is supported on $[0,1]$
with a marginal density bounded away from zero and infinity,
$\mathbb E\{f_j^*(X_j)\}=0$ for every $j\in G^*$, and the errors
are sub-Gaussian with $\sigma=1$.

Different basis families, such as Fourier, spline, and wavelet
families, induce different approximation regimes and effective
complexities. We therefore select both the basis family and the
active components. Let $\mathcal F=\{1,\ldots,M\}$ index the
candidate families. For each $m\in\mathcal F$, let
$
b_m(\cdot)
=
\bigl(b_{m,1}(\cdot),b_{m,2}(\cdot),\ldots\bigr)
$
be a basis dictionary and let
$b_{m,K}(\cdot)
=
\bigl(b_{m,1}(\cdot),\ldots,b_{m,K}(\cdot)\bigr)
\in\mathbb R^K$
denote its $K$-term truncation.
Let $B_{j,K}^{(m)}\in\mathbb R^{n\times K}$ be the corresponding
design block for the $j$th covariate. The class-specific design is
formed by concatenating these blocks over $j\in[p]$.

A candidate model in $\mathcal M_n^{(m)}$ is indexed by
$S=(G,\Gamma)$, where $G\subset[p]$ is the active component set and
$\Gamma$ specifies its within-family representation. For ordered
families, $\Gamma=K\in\mathbb N$ is the truncation level and
$s_{(m,S)}=|G|K$. For families admitting sparse representations,
$\Gamma=\{I_j\}_{j\in G}$ records the selected basis terms and
$s_{(m,S)}=\sum_{j\in G}|I_j|$.
We assign each within-family list $\mathcal M_n^{(m)}$ a
Kraft-admissible complexity $C_S^{(m)}$ and use the augmented code
$C_{(m,S)}=C_S^{(m)}+\log M$.
Minimizing \eqref{eq:dcic_multiclass} over
$(m,S)\in\widetilde{\mathcal M}_n$ then jointly selects the basis
family, the active components, and their within-family representations.

For each $m\in\mathcal F$, let $r_m(n)$ denote the optimal
approximation--estimation--complexity trade-off for a univariate
component over $\mathcal M_n^{(m)}$, and let $f_{j,m}^{\circ}$ be
admissible componentwise approximants attaining this trade-off up
to constants. The examples below specialize the risk adaptation
result in Corollary~\ref{cor:multiclass_risk}.
For a function $h$, write
$\|h\|_n^2:=n^{-1}\sum_{i=1}^n h^2(x_i)$.

\begin{corollary}
\label{cor:additive_multiclass_oracle}
Fix a constant $\lambda\ge\lambda_{\rm risk}$ and assume that the remaining
conditions of Corollary~\ref{cor:multiclass_risk} hold. Let
$m^\circ\in\arg\min_{m\in\mathcal F} r_m(n)$, and assume that, for some
constant $c_0\ge1$ independent of $n$, $p$, and $g^*$,
\[
\left\|
\sum_{j\in G^*}
\bigl(f_j^*-f_{j,m^\circ}^{\circ}\bigr)
\right\|_n^2
\le
c_0
\sum_{j\in G^*}
\left\|
f_j^*-f_{j,m^\circ}^{\circ}
\right\|_n^2
\]
Let \((\widehat m,\widehat S)\) be the DCIC selector in
\eqref{eq:select_multiclass}. Then
\[
\mathbb E\bigl\{\mathrm{ASE}(\widehat m,\widehat S)\bigr\}
\le
c_\lambda
\left\{
\frac{g^*\log(ep/g^*)}{n}
+
g^*\min_{m\in\mathcal F}r_m(n)
+
\frac{\log M}{n}
\right\}.
\]
\end{corollary}
Suppose that the true components belong to a univariate Sobolev-type smoothness class of order $\beta$, and that there exists a candidate family $m\in\mathcal F$ attaining the optimal one-dimensional rate $r_{m}(n)\asymp n^{-\frac{2\beta}{2\beta+1}}$. Then, Corollary~\ref{cor:additive_multiclass_oracle} matches the minimax rate for sparse additive estimation \citep{raskutti2012} up to the class-selection cost $\log M/n$. In particular, when $M$ grows at most polynomially in $n$, this extra term is of order $(\log n)/n$ and hence does not affect the overall rate of convergence.

\begin{comment}
\begin{remark}
\citet{barron1999risk} derive oracle risk bounds for penalized selection over a countable model list using penalties of order $L_S D_S/n$, where
$\sum_{S\in\mathcal{M}_n}\exp(-L_S D_S)\leq c$.
Thus, candidates from multiple families may be pooled into a single list for risk adaptation. Our formulation instead keeps the family label explicit through \eqref{eq:Cms_def}, making the class-selection cost $\log M/n$ transparent and facilitating the exact-recovery analysis for the class--model pair in Corollary \ref{cor:basis_consistency}.
\end{remark}
\end{comment}

\subsection{Fourier, Spline or Wavelet}\label{sub:example_family}
Consider candidate basis families
\[
\mathcal{F}=\{\text{Fourier},\ \text{Spline},\ \text{Haar Wavelet}\}.
\]

As in Section \ref{sec:basis-multiclass}, ordered family candidates are indexed by \(S=(G,K)\), whereas
Haar candidates are indexed by \(S=(G,\{I_j\}_{j\in G})\). Let \(\{N_t\}_{t\ge1}\) be an
increasing sequence satisfying \(N_t/t\to\infty\). We assume that, for each \(j\in G\), the selected indices in \(I_j\) can be arranged increasingly as
$
1\le \ell_{j,1}<\cdots<\ell_{j,|I_j|}
$
with
$
\ell_{j,t}\le N_t,\ t=1,\ldots,|I_j|.
$

We encode integers using a universal code of length of order $\log^*(t)$,
where $\log^*(t):=\log t+2\log\log(t+2)$ for $t\ge1$
\citep{rissanen1983}. For the Fourier and spline families, we assign
\[
C_S^{(m)}
=
|G|\log\frac{ep}{|G|}
+\log p+\log^*(K)+c_0,
\qquad
m\in\{\mathrm{Fourier},\mathrm{Spline}\}.
\]
For the Haar family, we define
\[
C_S^{(\mathrm{Haar})}
=
|G|\log\frac{ep}{|G|}
+\log p
+\sum_{j\in G}
\left\{
\log^*(|I_j|)
+\sum_{t=1}^{|I_j|}\log N_t
+c_0
\right\}.
\]
For concreteness, we take $N_t=(t+1)^2$, $t\ge1$, and choose $c_0>0$
sufficiently large to ensure Kraft admissibility. With the class-label
augmentation in Section~\ref{subsec:multiclass_dcic}, the aggregated code
satisfies
$
\sum_{(m,S)\in\widetilde{\mathcal M}_n}
\exp\{-C_{(m,S)}\}\le1.
$

Assume that the true regression function has a unique finite representation in one candidate basis family, yielding a unique pair
$(m^*,S^*)\in\widetilde{\mathcal M}_n$.
Let $\mathrm{APP}(m,S)$ denote the corresponding within-class approximation error. The following corollary specializes Corollary~\ref{cor:multiclass_consistency} to exact family and within-family recovery.
\begin{corollary}\label{cor:basis_consistency}
Consider the multi-class selection setting in
Section~\ref{sub:example_family}
with \(\eta_n=\log n\).
Assume there exists a unique true pair \((m^*,S^*)\in\widetilde{\mathcal M}_n\), and $g^* \leq p^\gamma$ for $0<\gamma<1$.
Assume that the following conditions hold:

\noindent
(\romannumeral 1) There exists an absolute constant $K_{\max}$ such that
$K \le K_{\max}$ for every candidate in the ordered families and
$|I_j| \le K_{\max}$ for every $j \in G$ in each Haar candidate.

\noindent
(\romannumeral 2) There exists a sufficiently large constant \(c_1>0\) such that, for every
\((m,S)\in\widetilde{\mathcal M}_{\mathrm{sub}}\),
\[
\mathrm{APP}(m,S)\ge c_1\bigl(s^*-s_{(m,S)}\bigr)(\log n+\lambda\log p);
\]

\noindent
(\romannumeral 3) There exist sufficiently large constants $\kappa>0$ and
$c_2>1$ such that, for every
\((m,S)\in\widetilde{\mathcal M}_{\mathrm{w}}\) with
$
s_{(m,S)}\le (1+\kappa)s^*,
$
we have
\[
\mathrm{APP}(m,S)\ge
c_2\Bigl\{
\bigl(s^*-s_{(m,S)}\bigr)_+\log n+\lambda s^*\log p+\log\log n
\Bigr\}.
\]
Then there exists a constant \(\lambda_{\rm basis}\), depending on \(\zeta\) and \(\gamma\), such that, for every \(\lambda\ge\lambda_{\rm basis}\), 
\[
\Pr\bigl\{(\hat m,\hat S)=(m^*,S^*)\bigr\}\to 1.
\]
\end{corollary}
\begin{remark}
Condition~(\romannumeral 3) admits a direct interpretation for cross-class
candidates. Let
$\mathcal G_{m,S}$ denote the finite-dimensional candidate class induced
by $(m,S)$, and define
\[
\delta_n(m,S)
:=
\frac{1}{n}
\inf_{g\in\mathcal G_{m,S}}
\sum_{i=1}^n
\bigl(f^*(x_i)-g(x_i)\bigr)^2.
\]
By the projection characterization of least squares,
$\delta_n(m,S)=\sigma^2\mathrm{APP}(m,S)/n$.
Hence, for $(m,S)\in\widetilde{\mathcal M}_{\mathrm w}$ with
$m\neq m^*$ and $s_{(m,S)}\le(1+\kappa)s^*$,
condition~(\romannumeral 3) requires
\[
\delta_n(m,S)
\ge
\frac{c_2\sigma^2}{n}
\left\{
\bigl(s^*-s_{(m,S)}\bigr)_+\log n
+\lambda s^*\log p
+\log\log n
\right\}.
\]
Thus, cross-class competitors are excluded when their empirical
approximation gaps dominate the corresponding complexity threshold.
\end{remark}
Outside the exactly specified setting, Corollary~\ref{cor:multiclass_risk} instead guarantees risk adaptation to the best approximation--complexity trade-off over the aggregated basis-family list.

\paragraph*{Periodic Sobolev class}
Under the periodic Sobolev assumption, \citet{tsybakov2009nonparametric} gives
$
r_{\mathrm{Haar}}(n)\asymp n^{-2/3},\
r_{\mathrm{spline}}(n)\asymp n^{-\frac{2\tilde\beta}{2\tilde\beta+1}},\
r_{\mathrm{Fourier}}(n)\asymp n^{-\frac{2\beta}{2\beta+1}},
$
where $\tilde\beta=\beta\wedge 4$ captures the saturation of fixed-order cubic splines when $\beta>4$
\citep{huangsu2021penalizedspline}. The Haar family has fixed regularity, so its approximation order over Sobolev classes saturates
and yields the exponent $2/3$ \citep{devore1998}. Hence, among the three candidate families, Fourier achieves the best one-dimensional rate when $\beta>4$.

\paragraph*{Non-periodic Sobolev class}
Assume that $f_j^*$ belongs to a non-periodic Sobolev class on $[0,1]$ with order $\beta\in(1,4]$. Then,
$
r_{\mathrm{spline}}(n)\asymp n^{-\frac{2\beta}{2\beta+1}},\
r_{\mathrm{Haar}}(n)\asymp n^{-2/3}.
$
In contrast, periodic Fourier truncation can be suboptimal without boundary matching and achieve only
$
r_{\mathrm{Fourier}}(n)\asymp n^{-1/2}
$
over this class \citep{tsybakov2009nonparametric}. Therefore, $\min_m r_m(n)=r_{\mathrm{spline}}(n)$.

\paragraph*{Besov class $B^1_{1,1}([0,1])$} Over this class, the minimax risk among linear estimators is of order $n^{-1/2}$,
whereas nonlinear wavelet methods can attain $n^{-2/3}$ \citep{donoho1998minimax}. Truncated Fourier and spline
least-squares fits are linear in the observations and hence are limited to the $n^{-1/2}$ benchmark. In contrast,
the Haar-wavelet family with within-group subset selection achieves
$
r_{\mathrm{Haar}}(n)\asymp n^{-2/3}\log n,
$
where the extra $\log n$ is the price of searching an enlarged subset list under Kraft coding \citep{yang1999selection}. Consequently,
$\min_m r_m(n)=r_{\mathrm{Haar}}(n)$, and the wavelet family is optimal among the candidates up to a logarithmic factor.

\section{A complexity-guided search strategy}\label{sec:strategy}
We develop a complexity-guided search strategy for all-subset selection under sub-Gaussian noise, with analogous constructions for other model classes. Although the strategy does not alter the worst-case NP-hardness of subset selection \citep{natarajan1995sparse}, complexity-guided pruning improves practical tractability and yields high-probability bounds on the retained search-region size together with finite-sample FP--FN control along the path.

\subsection{Complexity-guided search strategy}\label{subsec:strategy}
For each variable \(i\in[p]\), let \(r_i\in[p]\) denote its rank supplied by a
preliminary screening procedure, with smaller ranks corresponding to more
informative covariates. We assign the rank-based code length
\[
  C_i=\log^*(r_i),\qquad i\in[p].
\]
For every \(S\in\mathcal M_n\), let $ C_S := \sum_{i\in S} C_i . $ The resulting model complexity satisfies the generalized Kraft bound in
Remark~\ref{re:kraft}.
Along the path, the model-specific complexities induce a
$\lambda$-dependent budget that jointly controls the retained search
region and the oracle risk benchmark.
In practice, the ranking is obtained from a preliminary LASSO
\citep{T1996} or marginal correlation.
The induced Kraft code determines the search order and pruning priorities.

Let $\lambda_\zeta>0$ be a sufficiently large constant depending only on
$\zeta$, and let $\Lambda=\{\lambda_1,\ldots,\lambda_L\}
\subset[\lambda_\zeta,\infty)$ be a decreasing grid. The complexity of the corresponding DCIC minimizers
is nondecreasing as $\lambda$ decreases.
\begin{lemma}\label{lem:mono-complexity}
For each $\lambda>0$, let $\widehat S_\lambda$ be the minimizer of DCIC
over $S\in\mathcal M_n$. Then, for any $0<\lambda_1<\lambda_2$,
$C_{\widehat S_{\lambda_1}}\ge C_{\widehat S_{\lambda_2}}$.
\end{lemma}
Next, we introduce two pruning procedures that substantially reduce the candidate search space.

\subsubsection{First-stage pruning of model sizes}
At the current grid point $\lambda_k$, let
$\widetilde B_k:=\mathrm{DCIC}_{\lambda_k}(\widehat S_{\lambda_{k-1}})$
and $\widetilde C_k:=C_{\widehat S_{\lambda_{k-1}}}$.
For notational simplicity, when the current grid point is fixed, we suppress the subscript $k$ and write
$\widetilde B$, $\widetilde C$, and $\lambda$.
Since $\widehat S_{\lambda_{k-1}}$ remains feasible at the current $\lambda$,
any global minimizer $S$ at the current $\lambda$ must satisfy
$
\mathrm{RSS}(S)/\sigma^2
+\eta_n s+\lambda C_S
\leq \widetilde B .
$
Together with $\mathrm{RSS}(S)\geq0$ and Lemma~\ref{lem:mono-complexity}, this yields
\begin{align}
\label{eq:Ts-def}
\widetilde C\leq C_S\leq T_s,
\qquad
T_s:=\frac{\widetilde B-\eta_n s}{\lambda}.
\end{align}

Define
$\bar s:=\max\{s\in[p]:T_s\geq\widetilde C\}$
and
$\underline s:=\min\{s\in[p]:\max_{|S|=s}C_S\geq\widetilde C\}$.
No model of size $s>\bar s$ or $s<\underline s$ can satisfy
\eqref{eq:Ts-def}. Hence, the first-stage pruning restricts the search to model sizes
$\underline s\leq s\leq\bar s$.
\subsubsection{Second-stage pruning of candidate variables}
After the first-stage pruning identifies the admissible range of model sizes, the second-stage pruning further narrows the search space by screening variables for each admissible size \(s\).  

Let \(C_{(1)}\le \cdots \le C_{(p)}\) denote the order statistics of \(\{C_i\}_{i=1}^p\). For  model \(S\), let \(t_1<\cdots<t_s\) be the ordered indices of its selected variables in this ordering. If \(C_S\le T_s\), then for every \(r=1,\ldots,s\), we have
$
  \sum_{\ell=1}^{r-1} C_{(\ell)} + (s-r+1) C_{(t_r)} \le T_s .
$
Accordingly, define
\begin{align*}
  i_{s,r}
  :=
  \max\Bigl\{
  i\in\{r,\ldots,p\}:
  \sum_{\ell=1}^{r-1} C_{(\ell)} + (s-r+1) C_{(\romannumeral 1)} \le T_s
  \Bigr\},
  \qquad r=1,\ldots,s.
\end{align*}
Then every model with \(C_S\le T_s\) must satisfy
$
t_r \le i_{s,r},\ r=1,\ldots,s.
$
Hence, for each admissible size \(s\), it suffices to retain only those size-\(s\) supports satisfying these  bounds.

After these two pruning stages, the remaining optimization is carried out over the reduced family by a depth-first search, with further pruning based on the remaining complexity budget. 

\vspace{-0.2cm}

\begin{theorem}
\label{thm:coarse_poly_regime}
Assume that the model complexity satisfies
\(\sum_{S\in\mathcal M_n} e^{-C_S}\le 1\),
and
\(\|\mu_n\|_2^2\le c_1 n\sigma^2\)
for some constant \(c_1>0\).
Let \(N_\lambda\) denote the number of candidate models retained by the budget constraint \eqref{eq:Ts-def}.
Then, there exist constants \(c_2,c_3>0\) such that
\[
\Pr\!\Bigl(
N_\lambda \le \exp\!\left(c_2 n/\lambda\right)
\Bigr)\ge 1-\exp(-c_3 n).
\]
In particular, if \(\lambda \ge c_4\,n/\log p\) for some constant \(c_4>0\), then
\[
\Pr\!\left(
N_\lambda \le p^{\,c_2/c_4}
\right)\ge 1-\exp(-c_3 n).
\]
\end{theorem}
Theorem~\ref{thm:coarse_poly_regime} shows that larger $\lambda$ yields
more aggressive pruning, with $\lambda\gtrsim n/\log p$ giving a
polynomial-size retained region with high probability, while smaller
$\lambda$ sharpens the oracle benchmark in Section~\ref{sec:risk}.

Let $\tau_n:=\max_{j\in S^*}r_j$ be the rank envelope of the true support.
Since $C_{S^*}\le s^*\log^*(\tau_n)$, a sufficient condition for $S^*$ to
satisfy \eqref{eq:Ts-def} in the polynomial regime is
$s^*\log^*(\tau_n)\le c_1\log p$ for some $c_1>0$.
Under the uniform code $C_i=\log p$, this reduces to $s^*=O(1)$,
restricting the analogous guarantee to the fixed-sparsity regime.

\subsection{FP and FN guarantees along a \texorpdfstring{$\lambda$}{lambda}-path}\label{subsec:fpfn}
Fix $\delta\in(0,1)$.
Define
$
C_{\min}:=\min_{i\in[p]\setminus S^*}C_i,
C_{\max}^*:=\max_{i\in S^*}C_i.
$
Given $\lambda \in \Lambda$, define
\begin{align*}
    \mathcal{L}_{\lambda} := \{S \in \mathcal{M}_{\rm sub}\cup\mathcal{M}_{\rm w}:&(\eta_n+(\lambda-\lambda_{\zeta})C_{\min})\Delta^+(S)-\\
    &(\eta_n+\lambda C^*_{\rm max}+\lambda_{\zeta}\log s^*)\Delta^-(S) \leq \lambda_{\zeta}\log(L/\delta)\},
\end{align*}
and
\begin{align*}
\omega_{\lambda}
:=
\inf_{S \in \mathcal{L}_{\lambda}}
\lambda_{\min}\!\left(\frac{1}{n}
X_{S^*\setminus S}^\top (\mathrm{I}_n-P_S)X_{S^*\setminus S}\right).
\end{align*}
Let the uniform lower bound \(\underline{\omega}:=\min_{\lambda\in\Lambda}\omega_{\lambda}\).
\(\mathcal L_\lambda\) contains the underfitted and wrong models that cannot be excluded by the penalty term alone.
 Over the remaining localized model list, $\underline{\omega}$ provides a uniform lower bound along the path. The following theorem establishes a simultaneous FP--FN coupling over the $\lambda$-path.

\begin{theorem}\label{thm:control}
Assume $\underline{\omega}>0$. Then, with probability at least $1-c\delta$,
the following relationship holds for all $\lambda\in\Lambda$:
\begin{align*}
&\Bigl(\eta_n+(\lambda-\lambda_{\zeta}) C_{\min}\Bigr)
  \Delta^+(\widehat S_{\lambda})\\
&\quad+\Bigl(\frac{n}{2\sigma^2}\underline{\omega}\beta^2_{\min}
  -\eta_n-\lambda C^*_{\max}-\lambda_{\zeta}\log s^*\Bigr)
  \Delta^-(\widehat S_{\lambda})\\
&\qquad\le \lambda_{\zeta}\log\!\Bigl(\frac{L}{\delta}\Bigr).
\end{align*}
\end{theorem}
Under the $\beta_{\min}$ condition, Theorem \ref{thm:control} yields upper bounds for $\Delta^+(\widehat S_{\lambda})$ and $\Delta^-(\widehat S_{\lambda})$.
\begin{corollary}\label{cor:clean}
Assume the conditions of Theorem~\ref{thm:control} and
\begin{align*}
\frac{n}{2\sigma^2}\,\underline{\omega}\,\beta_{\min}^2>\eta_n+\lambda C_{\max}^*+\lambda_{\zeta}\log s^*
\qquad \text{for all }\lambda\in\Lambda.
\end{align*} 
Then, with probability at least $1-c\delta$, the following bounds hold for all
$\lambda\in\Lambda$:
\begin{align*}
\Delta^+(\widehat S_{\lambda})\ \le\frac{\lambda_{\zeta}\log(L/\delta)}{\eta_n+(\lambda-\lambda_{\zeta}) C_{\min}},
\qquad
\Delta^-(\widehat S_{\lambda})\ \le \frac{\lambda_{\zeta}\log(L/\delta)}
{\frac{n}{2\sigma^2}\underline{\omega}\beta_{\min}^2-\eta_n-\lambda C_{\max}^*-\lambda_{\zeta}\log s^*}.
\end{align*}
\end{corollary}

\begin{corollary}\label{cor:exact}
Assume that the conditions in Theorem~\ref{thm:control} hold.
Define
\[
\lambda_-:=\lambda_{\zeta}+\max\Big\{0,\ \frac{\lambda_{\zeta}\log(L/\delta)-\eta_n}{C_{\min}}\Big\},
\qquad
\lambda_+:=\frac{\frac{n}{2\sigma^2}\underline \omega\,\beta_{\min}^2-\eta_n-\lambda_{\zeta}\log(Ls^*/\delta)}{C_{\max}^*}.
\]
If $\lambda_-<\lambda_+$, with probability at least $1-c\delta$, we have
$
\widehat S_{\lambda}=S^*  \text{for all }\lambda\in\Lambda\cap(\lambda_-,\lambda_+).
$
\end{corollary}
Suppose \(\eta_n=\log n\), \(\delta=n^{-c}\) for a fixed \(c>0\), and \(L\) and
\(s^*\) grow at most polynomially in \(n\), so that \(\log(Ls^*/\delta)\asymp
\log n\). If \(\lambda_{\zeta}\log(L/\delta)\le \eta_n\), then \(\lambda_-=\lambda_{\zeta}\). Hence the exact
recovery region in Corollary~\ref{cor:exact} is nonempty under
$
\beta_{\min}^2
\gtrsim
\frac{\sigma^2}{n\underline\omega}\{C_{\max}^*+\log n\}.
$
Moreover, if \(S^*\) is contained in the first \(\tau_n\)
ranked variables, then under the optimal beta-min scale
\(\beta_{\min}^2\asymp \sigma^2\log p/(n\underline\omega)\), the upper endpoint satisfies
$
\lambda_+\asymp \log p/\log^*(\tau_n).
$

\subsection{A concrete \texorpdfstring{$\lambda$}{lambda}-path and the computation--statistics trade-off}\label{sub:tradeoff}
We initialize at
$\lambda_{\max}:=\|y\|_2^2/
(\sigma^2\min_{j\in[p]}C_j)$, for which the null model is globally
optimal.
Given \(\rho\in(0,1)\), we use the geometric grid
$
\lambda_k=\lambda_{\max}\rho^{k-1},\ k\in[L].
$
Larger $\rho$ yields finer resolution, whereas smaller $\rho$ reduces
the number of evaluations but may lead to larger changes in the
budget-constrained region.

We next give a concrete pathwise interpretation with \(\eta_n=\log n\). Suppose \(\|\mu_n\|_2^2\le c n\sigma^2\) and
consider \(\lambda_{\max}\asymp n\).
Theorem~\ref{thm:coarse_poly_regime} yields
$
N_{\lambda_k}
\le
\exp\left\{Cn/\lambda_k\right\}
\le
\exp\{C\rho^{-(k-1)}\}
$
with high probability.
Hence the search region remains polynomial in \(p\) up to the
computational boundary
$
\lambda_k\asymp n/\log p,
$
corresponding to
$
k_{\mathrm{poly}}\asymp \log_{1/\rho}\log p.
$

The same path also has a statistical interpretation. Suppose the true support is
contained in the first \(\tau_n\) ranked variables, and the risk bound in
Section~\ref{sec:risk} gives
\[
\mathbb E\{\operatorname{ASE}(\widehat S_{\lambda_k})\}
\le
C\sigma^2
\left\{
\frac{s^*\log n}{n}
+
\rho^{k-1}s^*\log^*(\tau_n)
\right\}.
\]

If the path is continued beyond the computational boundary to smaller
\(\lambda\), the search may no longer have a polynomial-size certificate, but the
oracle benchmark continues to improve. For simplicity, assume that
\(\log(ep/s^*)\asymp\log p\). The oracle benchmark matches the minimax rate when
$
\lambda_k \lesssim \log p/\log^*(\tau_n) .
$
Equivalently,
$
k
\gtrsim
k_{\mathrm{stat}}
:=
\log_{1/\rho}
\left\{
n\log^*(\tau_n)/{\log p}
\right\}.
$
For all $k\ge k_{\mathrm{stat}}$, the risk bound becomes
\[
\mathbb E\{\operatorname{ASE}(\widehat S_{\lambda_k})\}
\le
C\sigma^2\frac{s^*\log p}{n}.
\]
By
Corollary~\ref{cor:exact}, if
\(\beta_{\min}^2\asymp \sigma^2\log p/(n\underline\omega)\), the solution path enters
the exact-recovery region at the same order as \(k_{\mathrm{stat}}\). Thus, the statistical entry marks both minimax risk adaptation without a beta-min condition and entry into the exact-recovery region when the condition holds. Moreover, if \(\lambda_-=\lambda_{\zeta}\) and
\(\lambda_+\asymp \log p/\log^*(\tau_n)\), then the region contains about
$
\log_{1/\rho}
\left\{
\log p/\log^*(\tau_n)
\right\}
$
grid points. Hence, when \(\log^*(\tau_n)\ll\log p\), the selected support remains
equal to \(S^*\) over a growing segment of the \(\lambda\)-path.

Tables~\ref{tab:dimensional-regimes}
and~\ref{tab:ranking-sensitive-entry}
summarize the dimension-dependent and ranking-dependent regimes.
The first reports the certified polynomial range and statistical
benchmarks under representative growth rates of $p$, while the second
shows how the rank envelope $\tau_n$ affects the statistical entry and
the beta-min scale for exact recovery.
Appendix~\ref{app:pathwise-sim} provides supporting simulations.
Figure~\ref{fig:pathwise} gives a schematic summary.
A smaller $\tau_n$ shifts the statistical entry earlier along the path,
which lies within the certified polynomial region when
$k_{\mathrm{stat}}\lesssim k_{\mathrm{poly}}$, or equivalently
$\log^*(\tau_n)\lesssim(\log p)^2/n$.
Thus, the path may be stopped early for a conservative risk guarantee
or continued to smaller $\lambda$ to sharpen the oracle benchmark and
potentially attain the minimax-optimal risk scale.
Algorithm~\ref{alg:dcic} summarizes the search strategy.

\begin{table}[H]
\centering
\caption{Dimension-dependent scales along the DCIC path. Here \(k_{\rm poly}\) is the endpoint of the certified polynomial range. The last two columns report the target statistical benchmarks, whose ranking-dependent entry points are given in Table~\ref{tab:ranking-sensitive-entry}.}
\label{tab:dimensional-regimes}
\small
\resizebox{\textwidth}{!}{
\begin{tabular}{cccc}
\toprule
\multirow{2}{*}{Growth of \(p\)}
& \multicolumn{1}{c}{Computational certificate}
& \multicolumn{2}{c}{Statistical benchmarks} \\
\cmidrule(lr){2-2}\cmidrule(lr){3-4}
& $k\leq k_{\rm poly}$ 
& ASE benchmark
& Minimax beta-min scale \\
\midrule

\(n^\gamma, \gamma>0\)
& \(k\le \big\lceil \log_{1/\rho}(\gamma\log n)\big\rceil\)
& \(\sigma^2 s^*\log n/n\)
& \(\beta_{\min}^2\gtrsim \sigma^2\log n/(n\underline\omega)\)
\\[1.2ex]

\(\exp(n^\xi),\ 0<\xi<1\)
& \(k\le \big\lceil \log_{1/\rho}(n^\xi)\big\rceil\)
& \(\sigma^2 s^* n^{\xi-1}\)
& \(\beta_{\min}^2\gtrsim \sigma^2 n^{\xi-1}/\underline\omega\)
\\[1.2ex]

\(\exp(n)\)
& \(k\le \big\lceil \log_{1/\rho}(n)\big\rceil\)
& \(\sigma^2 s^*\)
& \(\beta_{\min}^2\gtrsim \sigma^2/\underline\omega\)
\\
\bottomrule
\end{tabular}
}
\end{table}

\begin{table}[H]
\centering
\caption{Effect of ranking quality on the statistical entry and certified beta-min scale along the DCIC path. Here \(k_{\rm stat}\) denotes the entry index at which the oracle
benchmark reaches the minimax ASE scale, and the certified beta-min scale refers to
the minimum signal strength for exact recovery within the certified range $k \leq k_{\rm poly}$.
}
\label{tab:ranking-sensitive-entry}

\begingroup
\setlength{\tabcolsep}{4.0pt}
\renewcommand{\arraystretch}{1.28}
\footnotesize

\begin{tabular}{@{}p{4.5cm}p{4.3cm}p{4.8cm}@{}}
\toprule
\multicolumn{1}{c}{Ranking regime}
& \multicolumn{1}{c}{$k\ge k_{\rm stat}$}
& \multicolumn{1}{c}{Certified beta-min scale}
\\
\midrule

Near-oracle: \(\tau_n\asymp s^*\)
& \(\displaystyle
k\ge \left\lceil
\log_{1/\rho}
\left\{
\frac{n\log^*(s^*)}{\log p}
\right\}
\right\rceil
\)
& \(\displaystyle
\beta_{\min}^2
\gtrsim
\frac{\sigma^2}{\underline\omega}
\left\{
\frac{\log^*(s^*)}{\log p}
+
\frac{\log n}{n}
\right\}
\)
\\
\midrule

Favorable: \(\log^*(\tau_n)=o(\log p)\)
& \(\displaystyle
k\ge \left\lceil
\log_{1/\rho}
\left\{
\frac{n\log^*(\tau_n)}{\log p}
\right\}
\right\rceil
\)
& \(\displaystyle
\beta_{\min}^2
\gtrsim
\frac{\sigma^2}{\underline\omega}
\left\{
\frac{\log^*(\tau_n)}{\log p}
+
\frac{\log n}{n}
\right\}
\)
\\
\midrule

Uninformative: \(\log^*(\tau_n)\asymp \log p\)
& \(\displaystyle
k\ge \left\lceil
\log_{1/\rho}(n)
\right\rceil
\)
& \(\displaystyle
\beta_{\min}^2
\gtrsim
\sigma^2/\underline\omega
\)
\\

\bottomrule
\end{tabular}
\endgroup
\end{table}

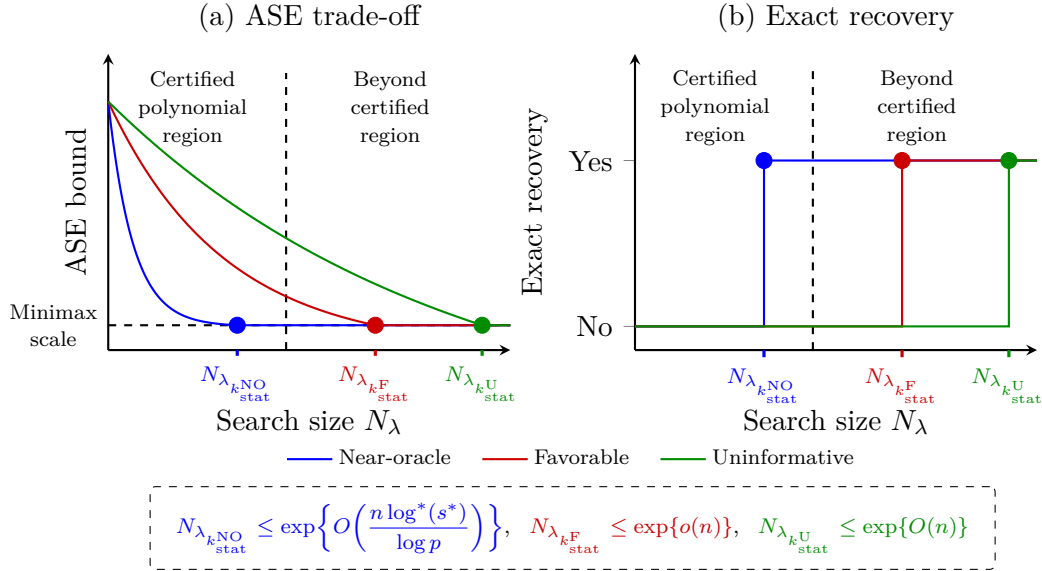
\begin{figure}[htbp]
\centering

\scalebox{1}{%
\begin{tikzpicture}

% =====================================================
% Global parameters
% =====================================================
\pgfmathsetmacro{\xminv}{2.5}
\pgfmathsetmacro{\xmaxv}{50000}

% Common computational boundary
\pgfmathsetmacro{\Npoly}{200}

% Ranking-specific statistical entry points
\pgfmathsetmacro{\NstatBlue}{60}
\pgfmathsetmacro{\NstatRed}{1800}
\pgfmathsetmacro{\NstatGreen}{25000}

% ASE-panel levels
\pgfmathsetmacro{\ystart}{2.45}
\pgfmathsetmacro{\yfloor}{1.28}

% Shape parameters
\pgfmathsetmacro{\kBlue}{5.2}
\pgfmathsetmacro{\kRed}{2.3}
\pgfmathsetmacro{\kGreen}{1.10}

% Ranking-specific entry notation
\def\NentryNO{N_{\lambda_{k_{\mathrm{stat}}^{\mathrm{NO}}}}}
\def\NentryF{N_{\lambda_{k_{\mathrm{stat}}^{\mathrm{F}}}}}
\def\NentryU{N_{\lambda_{k_{\mathrm{stat}}^{\mathrm{U}}}}}

\begin{groupplot}[
    group style={
        group size=2 by 1,
        horizontal sep=1.65cm
    },
    width=6.9cm,
    height=5.5cm,
    xmode=log,
    xmin=\xminv,
    xmax=\xmaxv,
    axis lines=left,
    axis line style={thick},
    tick align=outside,
    xtick=\empty,
    grid=none,
    tick label style={font=\small},
    label style={font=\small},
    clip=false
]

% =====================================================
% Panel (a): ASE trade-off
% =====================================================
\nextgroupplot[
    ymin=1.15,
    ymax=2.70,
    ylabel={ASE bound},
    title={(a) ASE trade-off},
    title style={font=\small, yshift=-2pt},
    ytick=\empty,
    legend to name=sharedlegend,
    legend columns=3,
    legend style={
        draw=none,
        fill=none,
        font=\scriptsize,
        /tikz/every even column/.append style={column sep=0.8em}
    },
    legend cell align={left}
]

% Minimax-scale benchmark
\addplot[
    black,
    dashed,
    thick,
    forget plot
] coordinates {
    (\xminv,\yfloor)
    (\xmaxv,\yfloor)
};

% Computational boundary
\addplot[
    black,
    dashed,
    thick,
    forget plot
] coordinates {
    (\Npoly,1.15)
    (\Npoly,2.60)
};

% Region labels
\node[
    font=\scriptsize,
    anchor=south,
    align=center
]
at (axis cs:20,2.16)
{Certified\\ polynomial\\ region};

\node[
    font=\scriptsize,
    anchor=south,
    align=center
]
at (axis cs:2600,2.16)
{Beyond\\ certified\\ region};

% Minimax-scale label
\node[
    font=\scriptsize,
    anchor=east,
    align=center
]
at (axis cs:\xminv,\yfloor)
{Minimax\\ scale};

% -----------------------------------------------------
% Near-oracle
% -----------------------------------------------------
\addplot[
    blue,
    thick,
    smooth,
    mark=none,
    domain=\xminv:\NstatBlue,
    samples=300
]
{
\yfloor + (\ystart-\yfloor) *
(
    (
        exp(
            -\kBlue *
            (
                (ln(x)-ln(\xminv))
                /
                (ln(\NstatBlue)-ln(\xminv))
            )
        )
        -
        exp(-\kBlue)
    )
    /
    (1-exp(-\kBlue))
)
};
\addlegendentry{Near-oracle}

\addplot[
    blue,
    thick,
    mark=none,
    forget plot,
    domain=\NstatBlue:\xmaxv,
    samples=2
]
{\yfloor};

\addplot[
    blue,
    only marks,
    mark=*,
    mark size=3pt,
    forget plot
] coordinates {
    (\NstatBlue,\yfloor)
};

% -----------------------------------------------------
% Favorable
% -----------------------------------------------------
\addplot[
    red!80!black,
    thick,
    smooth,
    mark=none,
    domain=\xminv:\NstatRed,
    samples=300
]
{
\yfloor + (\ystart-\yfloor) *
(
    (
        exp(
            -\kRed *
            (
                (ln(x)-ln(\xminv))
                /
                (ln(\NstatRed)-ln(\xminv))
            )
        )
        -
        exp(-\kRed)
    )
    /
    (1-exp(-\kRed))
)
};
\addlegendentry{Favorable}

\addplot[
    red!80!black,
    thick,
    mark=none,
    forget plot,
    domain=\NstatRed:\xmaxv,
    samples=2
]
{\yfloor};

\addplot[
    red!80!black,
    only marks,
    mark=*,
    mark size=3pt,
    forget plot
] coordinates {
    (\NstatRed,\yfloor)
};

% -----------------------------------------------------
% Uninformative
% -----------------------------------------------------
\addplot[
    green!55!black,
    thick,
    smooth,
    mark=none,
    domain=\xminv:\NstatGreen,
    samples=300
]
{
\yfloor + (\ystart-\yfloor) *
(
    (
        exp(
            -\kGreen *
            (
                (ln(x)-ln(\xminv))
                /
                (ln(\NstatGreen)-ln(\xminv))
            )
        )
        -
        exp(-\kGreen)
    )
    /
    (1-exp(-\kGreen))
)
};
\addlegendentry{Uninformative}

\addplot[
    green!55!black,
    thick,
    mark=none,
    forget plot,
    domain=\NstatGreen:\xmaxv,
    samples=2
]
{\yfloor};

\addplot[
    green!55!black,
    only marks,
    mark=*,
    mark size=3pt,
    forget plot
] coordinates {
    (\NstatGreen,\yfloor)
};

% -----------------------------------------------------
% Colored entry marks and labels below the x-axis
% -----------------------------------------------------
\draw[
    blue,
    very thick
]
(axis cs:\NstatBlue,1.15)
--
(axis cs:\NstatBlue,1.12);

\node[
    font=\scriptsize,
    text=blue,
    anchor=north,
    align=center
]
at (axis cs:\NstatBlue,1.12)
{$\NentryNO$};

\draw[
    red!80!black,
    very thick
]
(axis cs:\NstatRed,1.15)
--
(axis cs:\NstatRed,1.12);

\node[
    font=\scriptsize,
    text=red!80!black,
    anchor=north,
    align=center
]
at (axis cs:\NstatRed,1.12)
{$\NentryF$};

\draw[
    green!55!black,
    very thick
]
(axis cs:\NstatGreen,1.15)
--
(axis cs:\NstatGreen,1.12);

\node[
    font=\scriptsize,
    text=green!55!black,
    anchor=north,
    align=center
]
at (axis cs:\NstatGreen,1.12)
{$\NentryU$};

% =====================================================
% Panel (b): Exact recovery
% =====================================================
\nextgroupplot[
    ymin=-0.10,
    ymax=1.15,
    ylabel={Exact recovery},
    title={(b) Exact recovery},
    title style={font=\small, yshift=-2pt},
    ytick={0,0.7},
    yticklabels={No,Yes}
]

% Computational boundary
\addplot[
    black,
    dashed,
    thick,
    forget plot
] coordinates {
    (\Npoly,-0.10)
    (\Npoly,1.08)
};

% Region labels
\node[
    font=\scriptsize,
    anchor=south,
    align=center
]
at (axis cs:18,0.72)
{Certified\\ polynomial\\ region};

\node[
    font=\scriptsize,
    anchor=south,
    align=center
]
at (axis cs:2600,0.72)
{Beyond\\ certified\\ region};

% -----------------------------------------------------
% Near-oracle recovery step
% -----------------------------------------------------
\addplot[
    blue,
    thick,
    mark=none,
    forget plot
] coordinates {
    (\xminv,0)
    (\NstatBlue,0)
    (\NstatBlue,0.7)
    (\xmaxv,0.7)
};

\addplot[
    blue,
    only marks,
    mark=*,
    mark size=3pt,
    forget plot
] coordinates {
    (\NstatBlue,0.7)
};

% -----------------------------------------------------
% Favorable recovery step
% -----------------------------------------------------
\addplot[
    red!80!black,
    thick,
    mark=none,
    forget plot
] coordinates {
    (\xminv,0)
    (\NstatRed,0)
    (\NstatRed,0.7)
    (\xmaxv,0.7)
};

\addplot[
    red!80!black,
    only marks,
    mark=*,
    mark size=3pt,
    forget plot
] coordinates {
    (\NstatRed,0.7)
};

% -----------------------------------------------------
% Uninformative recovery step
% -----------------------------------------------------
\addplot[
    green!55!black,
    thick,
    mark=none,
    forget plot
] coordinates {
    (\xminv,0)
    (\NstatGreen,0)
    (\NstatGreen,0.7)
    (\xmaxv,0.7)
};

\addplot[
    green!55!black,
    only marks,
    mark=*,
    mark size=3pt,
    forget plot
] coordinates {
    (\NstatGreen,0.7)
};

% -----------------------------------------------------
% Colored entry marks and labels below the x-axis
% -----------------------------------------------------
\draw[
    blue,
    very thick
]
(axis cs:\NstatBlue,-0.10)
--
(axis cs:\NstatBlue,-0.13);

\node[
    font=\scriptsize,
    text=blue,
    anchor=north,
    align=center
]
at (axis cs:\NstatBlue,-0.13)
{$\NentryNO$};

\draw[
    red!80!black,
    very thick
]
(axis cs:\NstatRed,-0.10)
--
(axis cs:\NstatRed,-0.13);

\node[
    font=\scriptsize,
    text=red!80!black,
    anchor=north,
    align=center
]
at (axis cs:\NstatRed,-0.13)
{$\NentryF$};

\draw[
    green!55!black,
    very thick
]
(axis cs:\NstatGreen,-0.10)
--
(axis cs:\NstatGreen,-0.13);

\node[
    font=\scriptsize,
    text=green!55!black,
    anchor=north,
    align=center
]
at (axis cs:\NstatGreen,-0.13)
{$\NentryU$};

\end{groupplot}

% =====================================================
% Common x-axis labels
% =====================================================
\node[
    font=\small
]
at ($(group c1r1.south)+(0,-0.95cm)$)
{Search size $N_\lambda$};

\node[
    font=\small
]
at ($(group c2r1.south)+(0,-0.95cm)$)
{Search size $N_\lambda$};

% =====================================================
% Shared legend
% =====================================================
\node[
    anchor=north
]
at ($(group c1r1.south)!0.5!(group c2r1.south)+(0,-1.03cm)$)
{\pgfplotslegendfromname{sharedlegend}};

% =====================================================
% Centered dashed formula box
% =====================================================
\coordinate (boxcenter) at
($(group c1r1.south west)!0.5!(group c2r1.south east)$);

\node[
    draw,
    dashed,
    rounded corners=2pt,
    inner xsep=7pt,
    inner ysep=6pt,
    anchor=north,
    align=center
]
at ($(boxcenter)+(0,-1.8cm)$)
{
\scriptsize
\setlength{\tabcolsep}{0pt}
\begin{tabular}{
@{}
c
@{\hspace{0.25cm}}
c
@{\hspace{0.25cm}}
c
@{}
}
\textcolor{blue}{
\(\displaystyle
\NentryNO
\le
\exp\!\left\{
O\!\left(
\frac{n\log^*(s^*)}{\log p}
\right)
\right\}
\)},
&
\textcolor{red!80!black}{
\(\displaystyle
\NentryF
\le
\exp\!\left\{
o\!\left(
n
\right)
\right\}
\)},
&
\textcolor{green!55!black}{
\(\displaystyle
\NentryU
\le
\exp\{O(n)\}
\)}
\end{tabular}
};

\end{tikzpicture}
}%

\caption{
A schematic illustration of the pathwise computation--statistics trade-off.
The dashed line marks the boundary of the polynomial-in-$p$ certified region.
Panel~(a) shows the entry to the minimax ASE benchmark,
while panel~(b) shows the corresponding exact-recovery entry under the beta-min scale. The colored expressions give high-probability
search-size bounds at these entries. An entry lies within the certified region when \(k_{\rm stat}\lesssim k_{\rm poly}\).
}
\label{fig:pathwise}
\end{figure}

\begin{algorithm}[H]
\caption{Complexity-guided search strategy}
\label{alg:dcic}
\begin{algorithmic}[1]
\Require $\{C_i\}_{i=1}^p$;
$\Lambda=\{\lambda_1>\cdots>\lambda_L\}$;
$\sigma$; $T_{\rm stop}$.
\State $\widehat S_{\lambda_1}\gets\emptyset$,
$(\widetilde B,\widetilde C)
\gets\bigl(\mathrm{DCIC}_{\lambda_1}(\emptyset),0\bigr)$,
$\mathrm{unchanged}\gets0$, and
$\widehat\lambda\gets\lambda_L$.
\For{$k=2,\ldots,L$}
  \State Construct $\mathcal M_k$ from
  $(\widetilde B,\widetilde C,\lambda_k)$ using the two pruning stages.
  \State $\widehat S_{\lambda_k}
  \in\arg\min_{S\in\mathcal M_k}
  \mathrm{DCIC}_{\lambda_k}(S)$.
  \If{$\widehat S_{\lambda_k}
  =\widehat S_{\lambda_{k-1}}$}
    \State $\mathrm{unchanged}
    \gets\mathrm{unchanged}+1$
  \Else
    \State $\mathrm{unchanged}\gets0$
  \EndIf
  \If{$\mathrm{unchanged}\ge T_{\rm stop}$}
    \State $\widehat\lambda\gets\lambda_k$;
    \textbf{break}
  \EndIf
  \State $(\widetilde B,\widetilde C)
  \gets
  \bigl(
  \mathrm{DCIC}_{\lambda_k}(\widehat S_{\lambda_k}),
  C_{\widehat S_{\lambda_k}}
  \bigr)$.
\EndFor
\Ensure $\widehat S_{\widehat\lambda}$
\end{algorithmic}
\end{algorithm}

\begin{remark}\label{rem:accelerate}
In practice, we use a LASSO-based preliminary ranking and two acceleration
heuristics: restricting the depth-first search to the first $\lfloor p/5\rfloor$
ranked variables and replacing $T_s$ by a heuristic upper bound that retains $0.8$
of the RSS term. We set $T_{\mathrm{stop}}=8$ in all experiments. Sensitivity analyses for these implementation choices are reported in
Appendix~\ref{sub:sensitivity}.
\end{remark}

\section{Numerical Experiments}\label{sec:experiments}
We evaluate the performance of DCIC under strongly correlated designs and sparse additive models with basis-family uncertainty.

\subsection{All-subset selection under strongly correlated designs}\label{sub:sim_highly}
We compare DCIC with SCAD \citep{fan2001variable} and MCP
\citep{zhang2010}, tuned by either 10-fold cross-validation or HBIC
\citep{wang2013} under known $\sigma$, and with HTP
\citep{foucart2011hard}, using the adaptive procedure proposed by \citet{zhang2025},
and ABESS \citep{Zhu202014241}, both under known $\sigma$.
For DCIC, we construct the $\lambda$-path $\Lambda$ using $\rho=0.9$ with grid length $L=30$.
Given $(\hat\beta,\widehat S)$, we assess performance using
$\mathrm{ASE}=\|X\hat\beta-X\beta^*\|_2^2/n$, false positives (FP),
false negatives (FN), and the Matthews correlation coefficient
\[
\mathrm{MCC}
=
\frac{\mathrm{TP}\,\mathrm{TN}-\mathrm{FP}\,\mathrm{FN}}
{\sqrt{(\mathrm{TP}+\mathrm{FP})(\mathrm{TP}+\mathrm{FN})
(\mathrm{TN}+\mathrm{FP})(\mathrm{TN}+\mathrm{FN})}},
\]
where $\mathrm{TP}=|\widehat S\cap S^*|$,
$\mathrm{TN}=|\widehat S^c\cap(S^*)^c|$,
$\mathrm{FP}=|\widehat S\cap(S^*)^c|$, and
$\mathrm{FN}=|\widehat S^c\cap S^*|$.
We set \(s^* = 10\), \(\sigma = 2\) and  \(p = 500\). The sample size \(n\) increases from 90 to 300 in steps of 30. The nonzero entries of $\beta^*$ are drawn uniformly from $[1,2]$ with random signs. 
We generate the design matrix $X\in\mathbb{R}^{n\times p}$ with independent rows
$X_i \stackrel{\mathrm{i.i.d.}}{\sim} \mathcal{N}(0,\Sigma(w))$ and
\begin{equation*}
\Sigma(w) \;=\; (1-w)\Sigma_{\mathrm{AR}} + w\,\Sigma_{\mathrm{block}},\qquad w\in[0,1].
\end{equation*}
The first component $\Sigma_{\mathrm{AR}}$ is an AR(1) covariance with parameter $\rho_{\mathrm{out}}\in(0,1)$,
\begin{equation*}
(\Sigma_{\mathrm{AR}})_{ij} \;=\; \rho_{\mathrm{out}}^{|i-j|},\qquad 1\le i,j\le p.
\end{equation*}
The second component $\Sigma_{\mathrm{block}}$ is a block-diagonal equi-correlation structure, analogous to Scenario 3 in Section~\ref{sub:example}. Specifically, we partition the coordinates $\{1,\ldots,p\}$ into $10$ consecutive blocks of equal size $50$ and set
\begin{equation*}
(\Sigma_{\mathrm{block}})_{ij} \;=\;
\begin{cases}
1, & i=j,\\[2pt]
\rho_{\mathrm{in}}, & i\neq j \ \text{and $i,j$ belong to the same block},\\[2pt]
0, & \text{otherwise},
\end{cases}
\end{equation*}
where $\rho_{\mathrm{in}}\in(0,1)$ controls the strength of within-block correlation. By construction, $\Sigma(w)$ is positive definite for all $w\in[0,1]$, with $w=0$ corresponding to the AR(1) design and $w=1$ to the block design.   In the experiments, we consider three signal settings:
\begin{itemize}
\item \textbf{Setting 1 (one-per-block signals).} We select one index uniformly at random from each block, and let $S^*$ be the collection of these indices.

\item \textbf{Setting 2 (moderately clustered signals).} We randomly select $5$ blocks, then randomly select $2$ indices from each selected block.

\item \textbf{Setting 3 (highly clustered signals).} We randomly select $2$ blocks, then randomly select $5$ indices from each selected block. 
\end{itemize}
The first two settings are reported in Appendix~\ref{app:highly}. Since Setting~3 is the most
challenging case, with the strongest competition among correlated predictors, we focus on it in the
main text.
For each setting, we set $\rho_{\mathrm{in}}=\rho_{\mathrm{out}}=0.9$ and consider $w\in\{0.8,0.5,0.2\}$. 
Each experiment is repeated 100 times.

\begin{figure}[H]
  \centering
  \includegraphics[scale = 0.55]{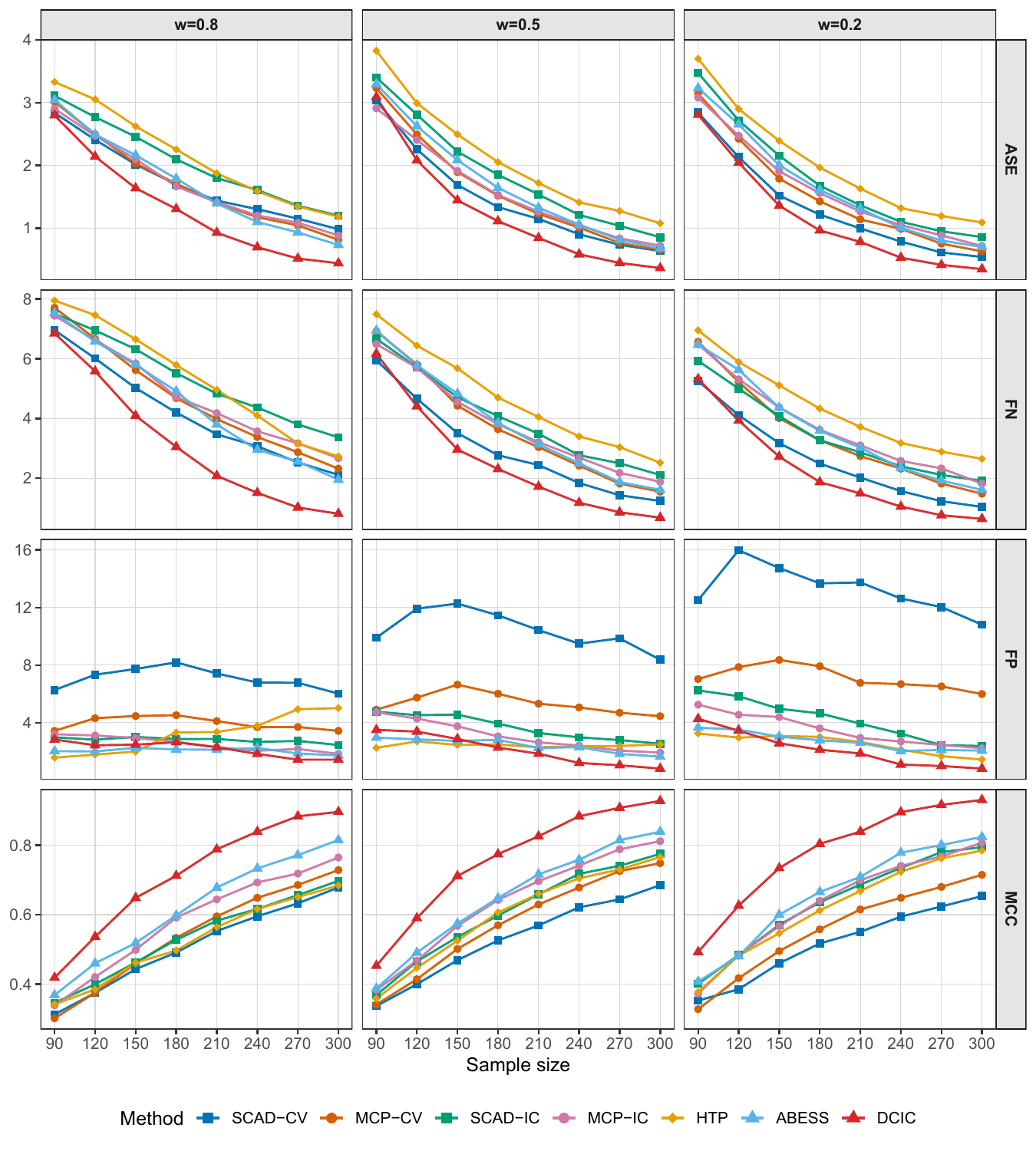}
  \caption{\label{fig:fig3}   Performance metrics under Setting 3 with an increasing sample size.
  The suffix ``-CV'' indicates tuning by 10-fold cross-validation. The suffix ``-IC'' indicates tuning by information criterion.
  }
\end{figure}
Under the most challenging Setting~3 (Figure~\ref{fig:fig3}), severe within-block collinearity substantially degrades the model selection accuracy of all methods. Although no method achieves exact support recovery under this setting, DCIC still exhibits robust performance in both parameter estimation and model selection. Even at small sample sizes, it sustains a high MCC by tightly controlling both FP and FN.
In contrast, HTP and ABESS tend to underselect under severe collinearity, whereas the penalized competitors generally exhibit either higher FP or weaker overall recovery.
\subsection{Empirical evaluation of multi-class basis selection}\label{sub:sim_sam}
We let $n$ range from $90$ to $360$ in steps of $30$, with $p=200$
and $g^*=5$. The variables are partitioned into five blocks of size $40$,
and $G^*$ contains two clusters of sizes $2$ and $3$.
Using the correlation structure in Section~\ref{sub:sim_highly} with
$\rho_{\mathrm{in}}=\rho_{\mathrm{out}}=0.8$ and $w=0.5$, we generate
$Z_i\sim N(0,\Sigma(w))$ and set $x_{ij}=\Phi(Z_{ij})$.
Thus, the covariates are supported on $[0,1]$ with dependence induced by
the corresponding Gaussian copula.

We consider 
$
\mathcal F=\{\text{Fourier},\ \text{Spline},\ \text{Haar}\}.
$
For Fourier, each component is expanded over an ordered trigonometric dictionary
with candidate orders $K\in\{2,4,6\}$.
For spline, each component is expanded over a cubic B-spline dictionary with candidate
dimensions $K\in\{3,4,5\}$.
For Haar, each component is expanded over a three-level Haar wavelet dictionary without
the scaling function. The three levels contain $1$, $2$, and $4$ atoms, respectively.
Accordingly, Fourier and spline are treated as ordered families, whereas Haar is treated
as a subset selection family.

In the Fourier and spline settings, the true order is $K^*=4$.
In the Haar setting, each active component contains exactly two nonzero wavelet atoms.
The active components are generated in a family-specific manner:
\begin{itemize}
  \item \textbf{Fourier truth ($m^*=\text{Fourier}$).}
  For each $j\in G^*$, set
  \[
  \beta_j
  = A_{\mathrm{F}}
  \Big(
  \gamma_{\mathrm{F}}^{3} u(\phi_4)^\top,\ 
  \gamma_{\mathrm{F}}^{2} u(\phi_3)^\top,\ 
  \gamma_{\mathrm{F}} u(\phi_2)^\top,\ 
  u(\phi_1)^\top
  \Big)^\top,
  \qquad
  u(\phi):=(\cos\phi,\sin\phi)^\top,
  \]
  where $A_{\mathrm{F}}=0.6$, $\gamma_{\mathrm{F}}=0.4$, and
  $\phi_\ell\stackrel{i.i.d.}{\sim}\mathrm{Unif}[0,2\pi]$ for $\ell=1,\ldots,4$.

  \item \textbf{Spline truth ($m^*=\text{Spline}$).}
  Define
  \[
  g(x)=\exp\!\left(-\frac{(x-0.1)^2}{2h^2}\right)
  -\alpha_0\exp\!\left(-\frac{(x-0.9)^2}{2h^2}\right),
  \]
  where $h=0.05$ and $\alpha_0=0.8$.
  For each $j\in G^*$, let $\beta_j$ be the least-squares projection of
  $\big(g(x_{1j}),\ldots,g(x_{nj})\big)^\top$ onto the cubic B-spline basis.

  \item \textbf{Haar truth ($m^*=\text{Haar}$).}
For each \(j\in G^*\), one atom is selected uniformly from each of levels 2 and 3, and the corresponding coefficients have independent random signs and magnitudes drawn from \(\mathrm{Unif}[0.5,1]\).
\end{itemize}

To construct comparable multi-class baselines, we combine group selection methods for ordered families with double-sparse
selection methods for subset-selection families. Specifically, we consider group LASSO
\citep{Y2006} combined with sparse group LASSO \citep{simon2013sparse}, group MCP
combined with composite MCP \citep{huang2012selective}, and $L_0$ methods for group
selection \citep{zhang2022} and double-sparse selection \citep{zhang2024}, representing
LASSO-type, MCP-type, and $L_0$-type procedures, respectively. Here we use EBIC \citep{chen2008extended} to select both the
family and the associated within-family model. We also report an oracle benchmark,
namely the ordinary least squares estimator on the true support.

We report the ASE, FP, FN, and MCC, where FP, FN, and MCC
are defined at the group level. Let $R$ denote
the number of simulation replications. For replication $r$, let $\hat m_r$ and $\hat K_r$
denote the estimated family and, for ordered families, the estimated order, respectively.
To assess family and within-family selection accuracy, we also record three family-level metrics: the family identification accuracy (FIA), defined by \(\mathrm{FIA}=R^{-1}\sum_{r=1}^R \mathbf{1}\{\hat m_r=m^*\}\); for ordered families, the order recovery accuracy (ORA), defined by \(\mathrm{ORA}=R^{-1}\sum_{r=1}^R \mathbf{1}\{\hat m_r = m^*, \hat K_r=K^*\}\); and, for the Haar family, the family-aware MCC (FaMCC), defined by \(\mathrm{FaMCC}=R^{-1}\sum_{r=1}^R \mathbf{1}\{\hat m_r=m^*\}\mathrm{MCC}(\hat S_r,S_r^*)\), where \(S_r^*\) and \(\hat S_r\) denote the true and estimated variable-level support sets in replication \(r\). Each experiment is repeated 100 times.

\begin{figure}[H]
  \centering
  \includegraphics[scale = 0.52]{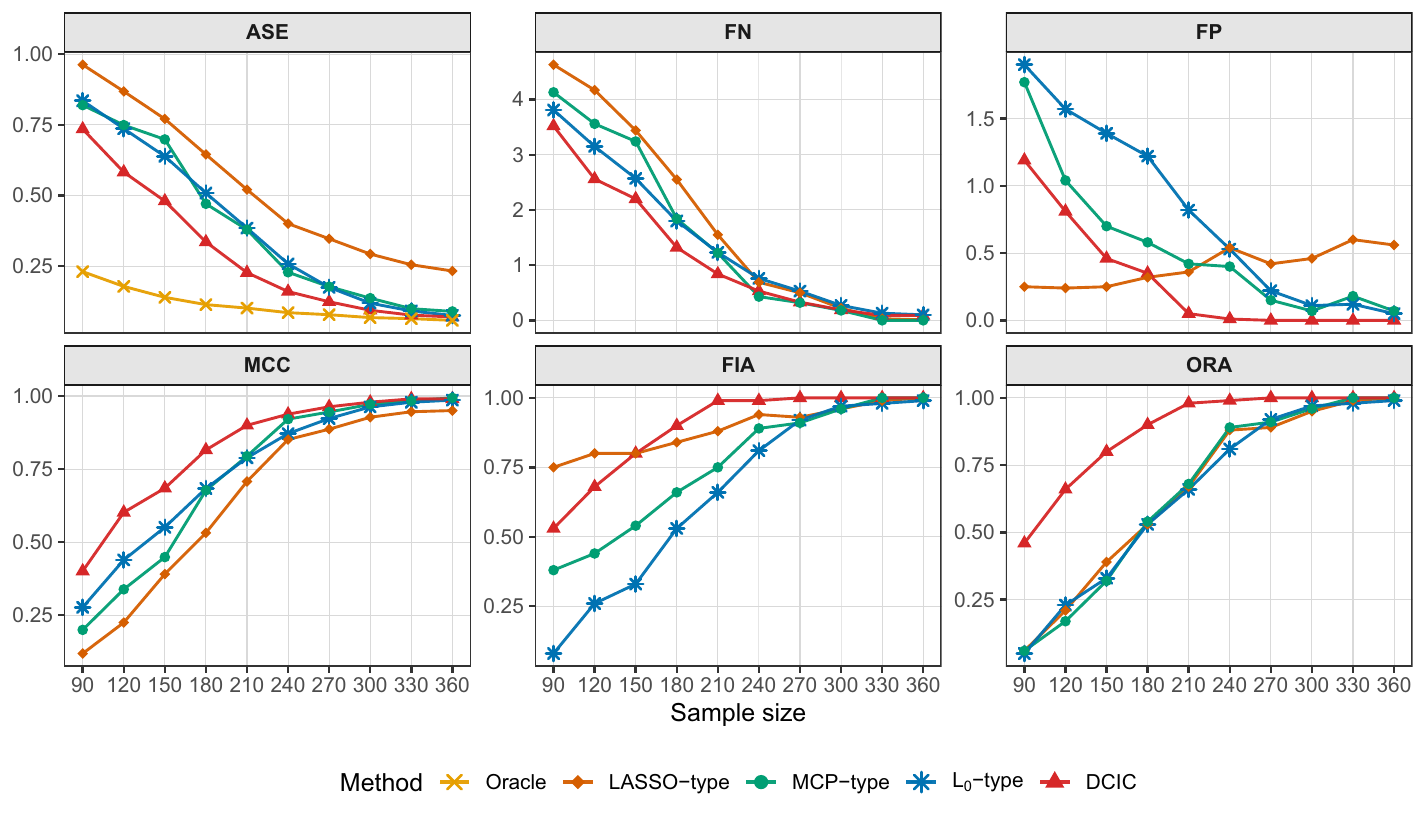}
  \caption{\label{fig:fig4}   Performance metrics under Fourier truth with an increasing sample size.
  }
\end{figure}
Under Fourier and spline truth (Figures~\ref{fig:fig4}--\ref{fig:fig5}), DCIC attains the lowest ASE and highest MCC among the non-oracle competitors, together with rapidly improving FIA and ORA. The advantage is also pronounced in family identification and order recovery: both FIA and ORA increase
rapidly and approach one at moderate sample sizes. In contrast, the LASSO-type
baseline under spline truth reduces FN at the cost of substantially inflated FP, leading to weaker overall recovery.

\begin{figure}[H]
  \centering
  \includegraphics[scale = 0.52]{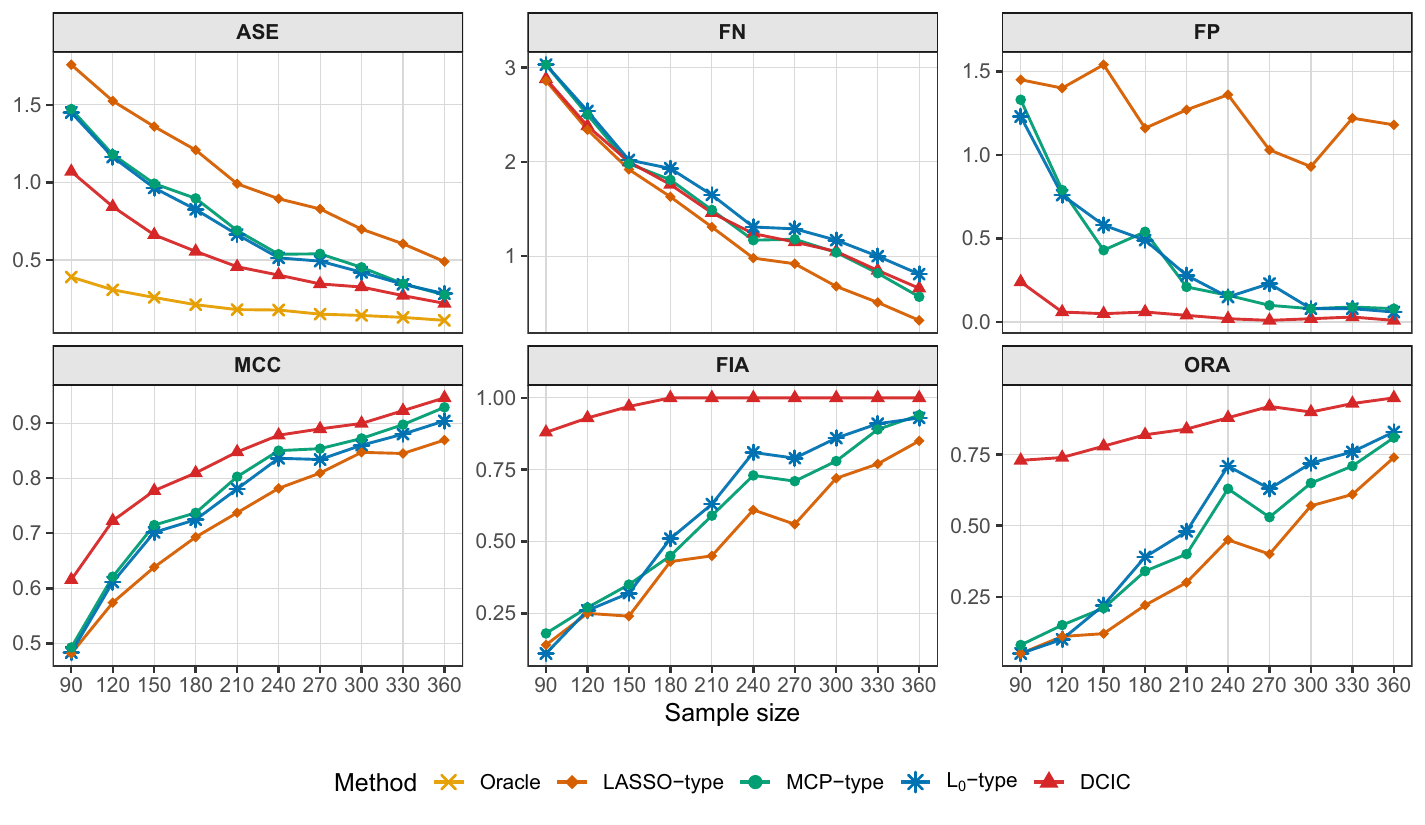}
  \caption{\label{fig:fig5}   Performance metrics under spline truth with an increasing sample size.
  }
\end{figure}

\begin{figure}[H]
  \centering
  \includegraphics[scale = 0.52]{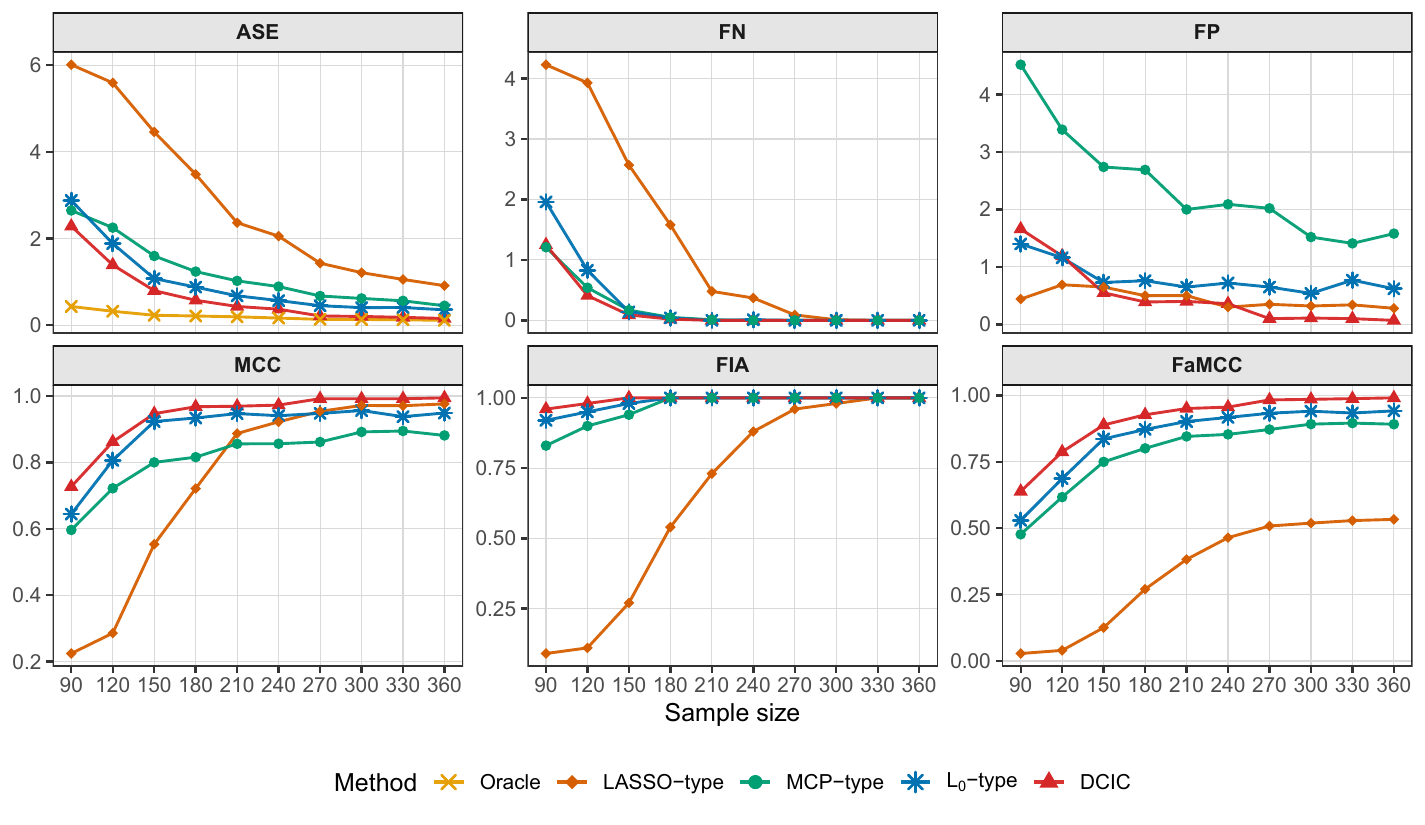}
  \caption{\label{fig:fig6}   Performance metrics under Haar truth with an increasing sample size.
  }
\end{figure}

Under Haar truth (Figure~\ref{fig:fig6}), DCIC attains the highest FaMCC over most sample sizes, indicating accurate family identification and recovery of the within-family sparse structure. Overall, Figures~\ref{fig:fig4}--\ref{fig:fig6} demonstrate that DCIC adapts well across ordered and subset-selection families, consistent with the theory in Section~\ref{sub:example_family}.

\section{Conclusion and Discussion}\label{sec:conclusion}
We introduced the Descriptive-Complexity Information Criterion
(DCIC), which uses Kraft-admissible code lengths to compare models
over large candidate collections under strong predictor dependence and
model-class uncertainty. Under sub-Weibull noise, DCIC achieves
selection consistency through approximation-error separation without
RIP/SRC-type near-orthogonality conditions and yields nonasymptotic
oracle risk bounds under model misspecification. Encoding the class label places heterogeneous model classes on a common
complexity scale, yielding guarantees for exact pair recovery under
identifiability and for risk adaptation across classes. The same complexity induces a
$\lambda$-dependent search budget. Larger values of $\lambda$ yield
polynomial-size retained regions with high probability, whereas
smaller values sharpen the oracle risk benchmark. Along this path, we
also obtain finite-sample control of false positives and false
negatives. The simulations illustrate these properties under strongly
correlated designs and basis-family uncertainty.

Robust criteria may extend DCIC to weaker tail conditions and contaminated observations
\citep{Oliveira2025}, and this direction merits further investigation.

\appendix
% Appendix body extracted from aos-supp.tex.
% Package loading, front matter, and bibliography are handled by DCIC_arxiv.tex.

\section{Proofs of main results}\label{app:proofs}
We begin with a reduction lemma that separates the three classes of competing models.
Let $\widetilde P_S$ be the orthogonal projection onto
$\operatorname{Col}(X_{S\cap S^*})$.
The assumptions of the lemma are as follows:
\begin{itemize}
  \item[(A1)] Overfitted models in $\mathcal M_{\mathrm{sup}}$.
  For every \(S\in\mathcal M_{\mathrm{sup}}\), \(C_S\ge C_{S^*}\). In addition, there exist constants $0<\tilde c_{1}<c_{1}$ such that for all $j\ge1$,
  \begin{equation*}
  \sum_{\substack{S\in\mathcal M_{\mathrm{sup}}\\ \Delta^+(S)=j}}
  \exp\!\left\{
  -c_1
    \lambda^\frac{\alpha}{2}\big(C_{S}-C_{S^*}\big)
  \right\}
  \ \le\
  \exp\!\left\{\tilde c_{1}(\eta_n j)^\frac{\alpha}{2}\right\}.
  \tag{A1}
  \end{equation*}

  \item[(A2)] Underfitted models in $\mathcal M_{\mathrm{sub}}$.
  There exists a sufficiently large constant $c_{2}>1$ such that
  \begin{equation*}
  \mathrm{APP}(S)
  \ \ge\
    c_{2}\Bigg(\eta_n(s^*-s)+\lambda(C_{S^*}^\frac{2}{\alpha}-C_{S}^{\frac{2}{\alpha}})_+
    +\log^\frac{2}{\alpha} N\Bigl(\Delta^-(S)\Bigr)\Bigg).
  \tag{A2}
  \end{equation*}

  \item[(A3)] Wrong models in $\mathcal M_{\mathrm{w}}$.
  For $S\in\mathcal M_{\mathrm{w}}$,
  choose truncation constants $B_{S}$ such that
  \begin{equation*}
  \sum_{S\in\mathcal M_w}
  \Pr\!\left(
  \frac{\varepsilon^\top(P_{S}-\widetilde P_{S})\varepsilon}{\sigma^2}
  > B_{S}\right)
  \longrightarrow 0 .
  \tag{A3a}
  \end{equation*}
  Define the quantity $\Gamma_{S}$ by
  \begin{equation*}
  \Gamma_{S}
  := \mathrm{APP}(S)
     +\eta_n\big(s-s^*\big)+\lambda\big(C^{\frac{2}{\alpha}}_{S}-C^{\frac{2}{\alpha}}_{S^*}\big)
     -B_{S}.
  \end{equation*}
  Assume
  \begin{equation*}
  \sum_{S\in\mathcal M_w}
  \exp\!\left\{
-c_3\
\min\Bigg(\Big(\frac{\Gamma_{S}}{ \sqrt{\mathrm{APP}(S)}}\Big)^{2}, \Big(\frac{\Gamma_{S}}{ \sqrt{\mathrm{APP}(S)}}\Big)^{\alpha}\Bigg)
\right\}
  \longrightarrow 0 .
  \tag{A3b}
  \end{equation*}
\end{itemize}

\begin{lemma}\label{thm:main}
    Under conditions (A1)--(A3), we have
    \begin{equation*}
        \Pr\left(\min_{S \in \mathcal{M}_n,\, S \neq S^*} \mathrm{DCIC}(S)>\mathrm{DCIC}(S^*)\right) \rightarrow 1
        \qquad \text{as } n \rightarrow \infty.
    \end{equation*}
\end{lemma}

\begin{proof}
We divide the proof of Lemma \ref{thm:main} into three cases.
For simplicity of notation, define
\[
u_{S}=\eta_n(s-s^*)+\lambda(C^{\frac{2}{\alpha}}_{S}-C^{\frac{2}{\alpha}}_{{S^*}}),
\qquad \alpha\in(0,2].
\]

{\bf Case 1:}\ $S \in \mathcal{M}_{\mathrm{sup}}$. In this case, note that
\begin{align}\label{eq:sup}
    \mathrm{DCIC}(S)-\mathrm{DCIC}(S^*) = \frac{\varepsilon^\top (P_{S^*}-P_S)\varepsilon}{\sigma^2}+u_{S}.
\end{align}
Since the entries of $\varepsilon = (\varepsilon_1,\ldots,\varepsilon_n)$ are mean-zero sub-Weibull random variables with variance $\sigma^2$, we have
\begin{equation*}
    \mathbb{E}\{\varepsilon^\top (P_{S}-P_{S^*})\varepsilon\} = \sigma^2 (r_{S}-r_{{S^*}}).
\end{equation*}
Moreover, since $S^*\subset S$, the matrix
$P_S-P_{S^*}$ is an orthogonal projection. Hence,
\begin{align*}
    \|P_S-P_{S^*}\|^2_\mathrm{F} = r_S-r_{S^*},\qquad \|P_S-P_{S^*}\|_\mathrm{op} \leq 1.
\end{align*}
Since $r_S-r_{S^*} \le s-s^*$ and $C_S \ge C_{S^*}$ for $S \in \mathcal M_{\mathrm{sup}}$, we have
\begin{align*}
u_S-(r_S-r_{S^*})
&= \eta_n(s-s^*)+\lambda(C_S^{\frac{2}{\alpha}}-C_{S^*}^{\frac{2}{\alpha}})-(r_S-r_{S^*}) \\
&\ge \frac{\eta_n}{2}(s-s^*)+\lambda(C_S^{\frac{2}{\alpha}}-C_{S^*}^{\frac{2}{\alpha}})\\
&>0.
\end{align*}
Then, for all sufficiently large $n$, we have
\begin{align*}
&\Pr(\mathrm{DCIC}(S)-\mathrm{DCIC}(S^*) \leq 0)\\
\leq& \Pr\!\left(\frac{\varepsilon^\top (P_S-P_{S^*})\varepsilon}{\sigma^2} \geq u_S\right) \\
\leq& \Pr\!\left(\frac{\varepsilon^\top (P_S-P_{S^*})\varepsilon}{\sigma^2}-(r_S-r_{S^*}) \geq u_S-(r_S-r_{S^*})\right) \\
\leq& 2\exp\left\{-c_{1,\alpha, \zeta}\min\left\{
\frac{(u_S-(r_S-r_{S^*}))^2}{r_S-r_{S^*}},
\,
\Bigl(u_S-(r_S-r_{S^*})\Bigr)^{\frac{\alpha}{2}}
\right\}\right\} \\
\leq& 2\exp\left\{-c_{1,\alpha, \zeta}\Bigl(u_S-(r_S-r_{S^*})\Bigr)^{\frac{\alpha}{2}}\right\} \\
\leq& 2\exp\left\{-c_{2,\alpha, \zeta}\Bigl((\eta_n-1)(s-s^*)+\lambda(C_S^{\frac{2}{\alpha}}-C_{S^*}^{\frac{2}{\alpha}})\Bigr)^{\frac{\alpha}{2}}\right\} \\
\leq& 2\exp\left\{-c_{2,\alpha, \zeta}\Bigl((\eta_n-1)(s-s^*)\Bigr)^{\frac{\alpha}{2}}\right\} \\
\to& 0,
\end{align*}
where the third inequality holds for all sufficiently large $n$, and the last two inequalities follow from
$r_S-r_{S^*}\le s-s^*$ and $C_S\ge C_{S^*}$.

The preceding bound gives
\begin{align}
\begin{split}\label{ineq:sup}
    &\sum_{S \in \mathcal{M}_{\mathrm{sup}}}\Pr\left(\mathrm{DCIC}(S)-\mathrm{DCIC}(S^*) \leq 0\right)\\
    \leq&\sum_j \sum_{S \in \mathcal{M}_{\mathrm{sup}}(j)}
    2\exp\left\{-c_{2,\alpha, \zeta}\Bigl((\eta_n-1)j+\lambda(C^{\frac{2}{\alpha}}_S-C^{\frac{2}{\alpha}}_{S^*})\Bigr)^{\frac{\alpha}{2}}\right\}\\
    \leq&2\sum_j\exp\{-c_1(\eta_nj)^{\frac{\alpha}{2}}\}
    \sum_{\substack{S\in\mathcal M_{\mathrm{sup}}\\ \Delta^+(S)=j}}\exp\{-c_1\lambda^{\frac{\alpha}{2}}(C_S-C_{S^*})\},
\end{split}
\end{align}
where the first inequality follows from $s-s^*\geq j$, and the second inequality follows from
$
(C^{\frac{2}{\alpha}}_S-C^{\frac{2}{\alpha}}_{S^*})^{\frac{\alpha}{2}} \geq C_S-C_{S^*}
$
and
$
(a+b)^q\geq 2^{q-1}(a^q+b^q)$ for $0<q\le 1,\ a,b>0.
$
By (A1), inequality \eqref{ineq:sup} yields 
\begin{align*}
    \sum_{S \in \mathcal{M}_{\mathrm{sup}}}\Pr\left(\mathrm{DCIC}(S)-\mathrm{DCIC}(S^*) \leq 0\right)
    &\leq\sum_j2\exp\{-(c_1-\tilde c_{1})(\eta_nj)^{\frac{\alpha}{2}}\}\\
    &\leq 2\int_{0}^{\infty}e^{-(c_1-\tilde c_{1})\eta_n^{\frac{\alpha}{2}}x^{\frac{\alpha}{2}}}dx\\
    &= O(\eta_n^{-1}) \to 0.
\end{align*}

{\bf Case 2: $S \in \mathcal{M}_{\mathrm{sub}}$.} In this case, note that
\begin{align}
\begin{split}\label{eq:sub}
    &\mathrm{DCIC}(S)-\mathrm{DCIC}(S^*)\\
    =& \frac{\varepsilon^\top (P_{S^*}-P_S)\varepsilon}{\sigma^2}+\frac{2\mu_n^\top(P_{S^*}-P_S)\varepsilon}{\sigma^2}+\mathrm{APP}(S)+u_{S}\\
    =& \underbrace{\frac{\varepsilon^\top (P_{S^*}-P_S)\varepsilon}{\sigma^2}-(r_{S^*}-r_S)}_{:= Q_{S}}
    +\underbrace{\frac{2\mu_n^\top(P_{S^*}-P_S)\varepsilon}{\sigma^2}}_{:=L_{S}}+\\
    &\underbrace{\mathrm{APP}(S)+u_{S}+(r_{S^*}-r_S)}_{:=T_{S}}.
\end{split}
\end{align}
Under assumption (A2), we have
\begin{align*}
T_S
&= \mathrm{APP}(S)-\eta_n(s^*-s)-\lambda(C_{S^*}^{\frac{2}{\alpha}}-C_S^{\frac{2}{\alpha}})+(r_{S^*}-r_S)\\
&\ge \mathrm{APP}(S)-\eta_n(s^*-s)-\lambda(C_{S^*}^{\frac{2}{\alpha}}-C_S^{\frac{2}{\alpha}})_+\\
&\ge \Bigl(1-c_2^{-1}\Bigr)\mathrm{APP}(S)>0.
\end{align*}
Set $Z_{S} = Q_S+L_{S}$. Note that
\begin{align}\label{event_relation}
    \{Z_S \leq -T_S\} \subset \{|Z_S| \geq T_S\} \subset \{|Q_S|\geq \tfrac{T_S}{2}\}\cup\{|L_S|\geq\tfrac{T_S}{2}\}.
\end{align}
By Lemma \ref{lem:hanson}, we have
\begin{align}\label{eq:qs}
\begin{split}
    \Pr\{|Q_S| > \tfrac{T_S}{2}\}
    &\leq 2\exp\left\{-c_{1,\alpha, \zeta}\min\left\{\frac{T^2_S}{r_{S^*}-r_S},\, T_{S}^{\frac{\alpha}{2}}\right\}\right\}\\
    &\leq 2\exp\{-c_{1,\alpha, \zeta}T_S^{\frac{\alpha}{2}}\}\\
    &\leq 2\exp\!\left\{-c_{2,\alpha, \zeta}\,\mathrm{APP}(S)^{\frac{\alpha}{2}}\right\},
\end{split}
\end{align}
where the second inequality holds for all sufficiently large $n$, and the third inequality follows from
$
T_S\ge \Bigl(1-c_2^{-1}\Bigr)\mathrm{APP}(S).
$

By Lemma \ref{lem:sum}, we have
\begin{align}\label{eq:ls}
\begin{split}
    \Pr\{|L_S| \geq \tfrac{T_S}{2}\}
    &\leq 2\exp\!\left\{
-c_{1,\alpha, \zeta}\
\min\Bigg(\Big(\frac{T_S}{\sqrt{\mathrm{APP}(S)}}\Big)^2, \Big(\frac{T_S}{\sqrt{\mathrm{APP}(S)}}\Big)^{\alpha}\Bigg)
\right\}\\
&\leq 2\exp\!\left\{
-c_{2,\alpha, \zeta}\,\mathrm{APP}(S)^{\frac{\alpha}{2}}
\right\},
\end{split}
\end{align}
where the last inequality follows from $T_S\ge (1-c_2^{-1})\mathrm{APP}(S)$ and $\mathrm{APP}(S) \to \infty$ as $n \to \infty$.

Consequently, combining \eqref{eq:sub}--\eqref{eq:ls}, we have
\begin{align*}
    \Pr\left(\mathrm{DCIC}(S)-\mathrm{DCIC}(S^*)\leq 0\right)
    \leq 4\exp\!\left\{
-c_{3,\alpha, \zeta}\,\mathrm{APP}(S)^{\frac{\alpha}{2}}
\right\}.
\end{align*}
Using assumption (A2), we further have
\[
\mathrm{APP}(S)^{\frac{\alpha}{2}}
\ge c_{4,\alpha, \zeta}\Bigl((\eta_n(s^*-s))^{\frac{\alpha}{2}}+\log N(\ell)\Bigr),
\]
where $c_{4,\alpha, \zeta}>0$ is a constant depending on $c_2$ and $\alpha$. Hence,
\begin{align*}
    \sum_{\substack{S \in \mathcal{M}_{\mathrm{sub}}\\ \Delta^-(S)=\ell}}
    \Pr\left(\mathrm{DCIC}(S)-\mathrm{DCIC}(S^*)\leq 0\right)
    &\leq 4N(\ell)\exp\left\{-c_{5,\alpha, \zeta}\Bigl((\eta_n(s^*-s))^{\frac{\alpha}{2}}+\log N(\ell)\Bigr)\right\}\\
    &\leq4\exp\left\{-c_{5,\alpha, \zeta}(\eta_n\ell)^{\frac{\alpha}{2}}\right\}N(\ell)^{1-c_{5,\alpha, \zeta}},
\end{align*}
where $c_{5,\alpha, \zeta}>1$ is another constant depending on $c_2$ and $\alpha$. When $c_2$ is sufficiently large, we have $c_{5,\alpha, \zeta}>1$, and thus
\begin{align*}
    \sum_{S \in \mathcal{M}_{\mathrm{sub}}}\Pr\left(\mathrm{DCIC}(S)-\mathrm{DCIC}(S^*)\leq 0\right)
    &\leq 4\sum_{\ell=1}^{|\mathcal{U}^*|}\exp\left\{-c_{5,\alpha, \zeta}(\eta_n\ell)^{\frac{\alpha}{2}}\right\}\\
    &\leq 4\int_0^\infty e^{-c_{5,\alpha, \zeta}\eta_n^{\frac{\alpha}{2}} x^{\frac{\alpha}{2}}}\,dx\\
    &= O(\eta_n^{-1})\longrightarrow 0.
\end{align*}

{\bf Case 3: $S \in \mathcal{M}_{\mathrm{w}}$.} In this case, we have
\begin{align}
\begin{split}\label{eq:wrong}
    &\mathrm{DCIC}(S)-\mathrm{DCIC}(S^*)\\
    = &\frac{\varepsilon^\top (P_{S^*}-\widetilde P_S)\varepsilon}{\sigma^2}
    -\frac{\varepsilon^\top ( P_{S}-\widetilde P_S)\varepsilon}{\sigma^2}
    +\frac{2\mu_n^\top(P_{S^*}-P_S)\varepsilon}{\sigma^2}
    +\mathrm{APP}(S)+u_{S}.
\end{split}
\end{align}
Since $P_{S^*}-\widetilde P_S$ is positive semidefinite, we have
\[
\frac{\varepsilon^\top (P_{S^*}-\widetilde P_S)\varepsilon}{\sigma^2}\ge 0.
\]
Hence, on the event
\[
\left\{
\frac{\varepsilon^\top (P_S-\widetilde P_S)\varepsilon}{\sigma^2}\le B_S
\right\},
\]
the inequality $\mathrm{DCIC}(S)-\mathrm{DCIC}(S^*)\leq0$ implies
\[
\frac{2\mu_n^\top(P_{S^*}-P_S)\varepsilon}{\sigma^2}+\Gamma_S\leq0.
\]
Therefore,
\[
\begin{aligned}
\Pr\bigl(\mathrm{DCIC}(S)-\mathrm{DCIC}(S^*)\leq0\bigr)
&\le
\Pr\left(
\frac{\varepsilon^\top (P_S-\widetilde P_S)\varepsilon}{\sigma^2}>B_S
\right)\\
&\quad+
\Pr\left(
\frac{2\mu_n^\top(P_{S^*}-P_S)\varepsilon}{\sigma^2}+\Gamma_S\leq 0
\right).
\end{aligned}
\]
From Lemma \ref{lem:sum}, we have
\begin{align*}
    \Pr\!\Big(
\frac{2\mu_n^\top (P_{S^*}-P_S) \varepsilon}{\sigma^2}+ \Gamma_S
\le 0
\Big)
\ \le&\ 2\exp\!\left\{
-c_3\
\min\Bigg(\Big(\frac{\Gamma_{S}}{ \sqrt{\mathrm{APP}(S)}}\Big)^{2}, \Big(\frac{\Gamma_{S}}{ \sqrt{\mathrm{APP}(S)}}\Big)^{\alpha}\Bigg)
\right\}.
\end{align*}
Thus, under assumption (A3), we have
\begin{align*}
&\sum_{S\in \mathcal{M}_{\mathrm{w}}}\Pr\left(\mathrm{DCIC}(S)-\mathrm{DCIC}(S^*)\leq 0\right)\\
\leq&
\sum_{S\in \mathcal{M}_{\mathrm{w}}}
\Pr\left(
\frac{\varepsilon^\top ( P_{S}-\widetilde P_S)\varepsilon}{\sigma^2}> B_{S}
\right)+
\sum_{S\in \mathcal{M}_{\mathrm{w}}}
\Pr\!\Big(
\frac{2\mu_n^\top (P_{S^*}-P_S) \varepsilon}{\sigma^2}+ \Gamma_S
\le 0
\Big)\\
\longrightarrow&\ 0.
\end{align*}

Combining {\bf Cases 1--3}, we have
\begin{align*}
    &\sum_{S\in \mathcal{M}_{n},\, S\neq S^*}\Pr\left(\mathrm{DCIC}(S)-\mathrm{DCIC}(S^*)\leq 0\right)\\
    \leq& \sum_{S\in \mathcal{M}_{\mathrm{sup}}}\Pr\left(\mathrm{DCIC}(S)-\mathrm{DCIC}(S^*)\leq0\right)
    +\sum_{S\in \mathcal{M}_{\mathrm{sub}}}\Pr\left(\mathrm{DCIC}(S)-\mathrm{DCIC}(S^*)\leq0\right)+\\
    &\sum_{S\in \mathcal{M}_{\mathrm{w}}}\Pr\left(\mathrm{DCIC}(S)-\mathrm{DCIC}(S^*)\leq0\right)\\
    \longrightarrow&\ 0.
\end{align*}
Therefore,
\[
\Pr\left(\min_{S \in \mathcal{M}_n,\, S\neq S^*}\mathrm{DCIC}(S)>\mathrm{DCIC}(S^*)\right)\to 1.
\]
This completes the proof.
\end{proof}

\subsection{Proof of Theorem~\ref{thm:complexity}}
In the proof of Theorem~\ref{thm:complexity}, we verify that all the assumptions of Lemma \ref{thm:main} hold.
\begin{proof}
(A1) and (A2) in Lemma \ref{thm:main} can be verified directly from the assumptions of Theorem~\ref{thm:complexity}.
For simplicity of notation, denote $\ell = \Delta^-(S)$, $j = \Delta^+(S)$ and 
$
D'_S = \log \binom{q^*}{\ell}\binom{q-q^*}{j}.
$

{\bf (Verification of (A3) in Lemma \ref{thm:main})}
Recall that $
B_{S} = \frac{1}{4}\lambda D^\frac{2}{\alpha}_{S}+\frac{1}{4}\log\log n+2(r_S-\tilde r_S)
$
and $\mathbb{E}(\varepsilon^\top(P_S-\tilde{P}_S)\varepsilon/\sigma^2) = r_S-\tilde r_S$.
We have
\begin{align*}
\quad\ 
\Pr\Bigl(
\frac{\varepsilon^\top(P_S-\tilde P_S)\varepsilon}{\sigma^2}
> B_S
\Bigr)
&= 
\Pr\Bigl(
\frac{\varepsilon^\top(P_S-\tilde P_S)\varepsilon}{\sigma^2}-\mathbb{E}(\frac{\varepsilon^\top(P_S-\tilde P_S)\varepsilon}{\sigma^2})
> B_S-(r_S-\tilde r_S)
\Bigr)\\
&\leq 2\exp\Bigl\{-c_{\alpha,\zeta}\min\bigl\{\frac{(B_S-(r_S-\tilde r_S))^2}{r_S-\tilde r_S}, (B_S-(r_S-\tilde r_S))^\frac{\alpha}{2}\bigr\}\Bigr\}
\end{align*}
Since $$\frac{(B_S-(r_S-\tilde r_S))^2}{r_S-\tilde r_S}\geq B_S-(r_S-\tilde r_S)$$ and $B_S-(r_S-\tilde r_S) \geq 1$ for sufficiently large $n$, we have
\begin{align*}
\quad\ 
\Pr\Bigl(
\frac{\varepsilon^\top(P_S-\tilde P_S)\varepsilon}{\sigma^2}
> B_S
\Bigr)
&\leq 2\exp\Bigl\{-c_{\alpha,\zeta} (B_S-(r_S-\tilde r_S))^\frac{\alpha}{2}\Bigr\}
\end{align*}
On the other hand, we have $\left|\left\{S \in \mathcal{M}_{\mathrm{w}}: \Delta^{-}(S)=\ell, \Delta^{+}(S)=j\right\}\right| \leq\binom{ q^*}{\ell}\binom{q-q^*}{j}$ and
\begin{align*}
    \sum_{S\in \mathcal{M}_{\mathrm{w}}} e^{-2D'_S}&\leq\sum_{\ell=1}^{q^*}\sum_{j=1}^{q-q^*}\Biggl\{\binom{q^*}{\ell}\binom{q-q^*}{j}\Biggr\}^{-1}\\
    &=\Biggl\{\sum_{\ell=1}^{q^*}\binom{q^*}{\ell}^{-1}\Biggr\}\times\Biggl\{\sum_{j=1}^{q-q^*}\binom{q-q^*}{j}^{-1}\Biggr\}\\
    &\leq \Biggl\{1+\sum_{\ell=1}^{q^*-1}\binom{q^*}{\ell}^{-1}\Biggr\}\times\Biggl\{1+\sum_{j=1}^{q-q^*-1}\binom{q-q^*}{j}^{-1}\Biggr\}\\
    &\leq 4,
\end{align*}
where the first inequality uses 
$\binom{b}{a}\geq b$ for 
$1\leq a \leq b-1.$

Consequently, we have
\begin{align}
\begin{split}\label{eq:s5}
&\quad\ \sum_{S\in\mathcal M_w}
\Pr\Bigl(
\frac{\varepsilon^\top(P_S-\tilde P_S)\varepsilon}{\sigma^2}
> B_S
\Bigr)\\
&\le \sum_{S\in\mathcal M_w}2\exp\Bigl\{-c_{\alpha,\zeta}\Bigl(\lambda^\frac{\alpha}{2} D_{S}+(\log\log n)^\frac{\alpha}{2}\Bigr)\Bigr\}\\
&\le \sum_{S\in\mathcal M_w}
2\exp\Bigl\{-c_{\alpha,\zeta}\lambda^\frac{\alpha}{2} D_S\Bigr\}
\exp\bigl\{-c_{\alpha,\zeta}((\log\log n)^\frac{\alpha}{2})\bigr\}\\
&\le \sum_{S\in\mathcal M_w}2
\exp\Bigl\{-2\min\{D'_S, C_S\}\Bigr\}
\exp\bigl\{-c_{\alpha,\zeta}((\log\log n)^\frac{\alpha}{2})\bigr\}\\
&\le \Biggl(\sum_{S\in\mathcal M_w}2
\exp\Bigl\{-2D'_S\Bigr\}+\sum_{S\in\mathcal M_w}2
\exp\Bigl\{-2C_S\Bigr\}\Biggr)
\exp\bigl\{-c_{\alpha,\zeta}((\log\log n)^\frac{\alpha}{2})\bigr\}\\
&\le 10\exp\bigl\{-c_{\alpha,\zeta}((\log\log n)^\frac{\alpha}{2})\bigr\}\\
&\longrightarrow\ 0,
\end{split}
\end{align}
where the first inequality uses Lemma \ref{lem:hanson} together with
$
(a+b)^q \geq 2^{q-1}(a^q+b^q), \ q\in(0,1],
$
and the third inequality follows by the choice of $\lambda_{\rm sel}$,
which ensures that $c_{\alpha,\zeta}\lambda_{\rm sel}^{\alpha/2}>2$. This proves (A3a).

Next, we verify assumption (A3b). We first show that
\begin{equation}\label{eq:gammas}
\frac{\Gamma_S}{\sqrt{\mathrm{APP}(S)}} \ge c\,\sqrt{B_S}
\end{equation}
for some universal constant $c>0$.
On one hand, when $u_S \geq 2 B_S$, we have
\[
\Gamma_S = \mathrm{APP}(S)+u_S-B_S \geq \mathrm{APP}(S)+B_S.
\]
Hence,
\[
\frac{\Gamma_S}{\sqrt{\mathrm{APP}(S)}}
\ge \sqrt{\mathrm{APP}(S)}+\frac{B_S}{\sqrt{\mathrm{APP}(S)}}
\ge 2\sqrt{B_S},
\]
where the last inequality follows from the arithmetic--geometric mean inequality.

On the other hand, when $u_S < 2B_S$, the assumption in Theorem~\ref{thm:complexity} implies
$
\mathrm{APP}(S)\ge 2(2B_S-u_S).
$
Therefore,
\[
\Gamma_S = \mathrm{APP}(S)+u_S-B_S
= \frac{1}{2}\mathrm{APP}(S)+\Bigl(\frac{1}{2}\mathrm{APP}(S)+u_S-B_S\Bigr)
\ge \frac{1}{2}\mathrm{APP}(S)+B_S.
\]
Consequently,
\[
\frac{\Gamma_S}{\sqrt{\mathrm{APP}(S)}}
\ge
\frac{1}{2}\sqrt{\mathrm{APP}(S)}+\frac{B_S}{\sqrt{\mathrm{APP}(S)}}
\ge \sqrt{2}\sqrt{B_S},
\]
where the last inequality is obtained by minimizing the right-hand side over $\mathrm{APP}(S)>0$.

Thus, \eqref{eq:gammas} holds with $c=\sqrt{2}$. Using the same argument as in \eqref{eq:s5}, we obtain
\begin{align*}
\sum_{S\in \mathcal{M}_{\mathrm{w}}}
\exp\Bigl\{-c_3\min\Bigg(\Big(\frac{\Gamma_{S}}{ \sqrt{\mathrm{APP}(S)}}\Big)^{2}, \Big(\frac{\Gamma_{S}}{ \sqrt{\mathrm{APP}(S)}}\Big)^{\alpha}\Bigg)\Bigr\}
&\leq
\sum_{S\in \mathcal{M}_{\mathrm{w}}}
\exp\bigl\{-c_3 c^{\alpha} B_S^{\alpha/2}\bigr\}\\
&\longrightarrow\;0.
\end{align*}
This proves (A3b). Consequently, assumption (A3) in Lemma \ref{thm:main} is verified under the assumptions of Theorem~\ref{thm:complexity}, which completes its proof.
\end{proof}

\subsection{Proof of Corollary~\ref{cor:complexity}}
\begin{proof}
We first verify that the induced model complexities satisfy the Kraft inequality required in Theorem~\ref{thm:complexity}. By the definition
$
C_S=1+\sum_{k:U_k\subseteq S} C_k,
$
we have
\begin{align*}
\sum_{S\in\mathcal M_n} e^{-C_S}
&= e^{-1}
\sum_{S\in\mathcal M_n}\exp\!\left(-\sum_{k:U_k \subseteq S} C_k\right)\\
&\le e^{-1}
\prod_{k=1}^q(1+e^{-C_k})\\
&\le e^{-1}
\exp\!\left(\sum_{k=1}^q e^{-C_k}\right)\\
&\le 1,
\end{align*}
where the last inequality follows from 
\(
\sum_{k=1^q} e^{-C_k}\le 1.
\)

Next, we verify assumption (A1) in Theorem~\ref{thm:complexity}.
For $S\in\mathcal M_{\mathrm{sup}}$, since $\mathcal{U}(S^*)\subseteq \mathcal{U}(S)$,
\[
C_S-C_{S^*}
=
\sum_{k:\,U_k\subseteq S,\;U_k\nsubseteq S^*} C_k .
\]
Hence, for each $j\ge 1$,
\begin{align*}
\sum_{\substack{S \in \mathcal{M}_{\mathrm{sup}}\\ \Delta^+(S)=j}}
\exp\{-c\lambda^{\alpha/2}(C_S-C_{S^*})\}
&\le
\sum_{S\in \mathcal{M}_{\mathrm{sup}}}
\exp\left\{
-c\lambda^{\alpha/2}
\sum_{k:\,U_k\subseteq S,\;U_k\nsubseteq S^*} C_k
\right\} \\
&\le
\prod_{k:\,U_k\nsubseteq S^*}
\left(1+e^{-c\lambda^{\alpha/2}C_k}\right) \\
&\le
\exp\left(
\sum_{k:\,U_k\nsubseteq S^*}
e^{-c\lambda^{\alpha/2}C_k}
\right) \\
&\le
\exp\left(
\sum_{k=1}^q e^{-C_k}
\right) \\
&\le e,
\end{align*}
where the third inequality follows by the choice of $\lambda_{\rm sel}$,
which ensures that $c_{\alpha,\zeta}\lambda_{\rm sel}^{\alpha/2}>2$, and the last inequality follows from \(
\sum_{k=1}^q e^{-C_k}\le 1.
\)
Therefore the summation term appearing in (A1) is uniformly bounded in $j$, and since
\(
\eta_n\to\infty
\),
(A1) holds.

The remaining assumptions are exactly (\romannumeral 2)--(\romannumeral 3) of Theorem~\ref{thm:complexity}. Therefore, the conclusion follows directly from that theorem.
\end{proof}

\subsection{Proof of the double-sparse selection result}\label{app:double_sparse}
\label{app:double-sparse}

We first state the precise selection-consistency result deferred from
Section~\ref{sub:ds}. Recall that
\[
N(\ell)=\binom{s^*}{\ell},
\qquad \ell=1,\ldots,s^*.
\]

\begin{corollary}\label{cor:DS}
Let $\eta_n=\log n$ and
\[
C_S
=
g\log m+s\log d+2\log(g+1)
+2\sum_{k\in G}\log(t_k+1).
\]
Assume that $g^*\le d^\gamma$ for some constant $\gamma\ge0$.
For every
$
\lambda\ge
\lambda_{\mathrm{sel}}\vee
c_{\alpha,\zeta}(1+\gamma)^{2/\alpha},
$
with $c_{\alpha,\zeta}>0$ sufficiently large, if the following
conditions hold:
\begin{itemize}
    \item[(\romannumeral 1)]
    There exists a sufficiently large constant $c_1>1$ such that,
    for every $S\in\mathcal M_{\mathrm{sub}}$,
    \[
    \mathrm{APP}(S)\ge
    c_1\left\{
    (s^*-s)\log n
    +\lambda\left(C_{S^*}^{2/\alpha}-C_S^{2/\alpha}\right)
    +\log^{2/\alpha}N\bigl(\Delta^-(S)\bigr)
    \right\}.
    \]

    \item[(\romannumeral 2)]
    For every $S\in\mathcal M_w$ satisfying $2B_S>u_S$,
    \[
    \mathrm{APP}(S)\ge2(2B_S-u_S).
    \]
\end{itemize}
Then
\[
\Pr\left\{
\min_{\substack{S\in\mathcal M_n\\S\neq S^*}}
\mathrm{DCIC}(S)>\mathrm{DCIC}(S^*)
\right\}
\to1
\qquad\text{as }n\to\infty.
\]
\end{corollary}
\begin{remark}
The growth condition $g^*\le d^\gamma$ is used only to verify the overfitted-model summability under general sub-Weibull noise. In the sub-Gaussian case $\alpha=2$, this condition and the corresponding $\gamma$-dependent lower bound on $\lambda$ can be removed by taking $\eta_n=A_\eta\log n$ with a sufficiently large fixed constant $A_\eta$, provided that $g^*\le n$. This modification does not change the order of the separation conditions.
\end{remark}
\begin{proof}
We first verify the Kraft inequality. Let
\[
A_d
:=
\sum_{r=1}^d
\binom{d}{r}\frac{d^{-r}}{(r+1)^2}
\le
\frac14\left\{\left(1+\frac1d\right)^d-1\right\}
\le
\frac{e-1}{4}.
\]
Summing over all
double-sparse supports gives
\begin{align*}
\sum_{S\in\mathcal M_n}e^{-C_S}
&\le
\sum_{g=1}^m
\binom{m}{g}\frac{m^{-g}}{(g+1)^2}A_d^g \\
&\le
\left(1+\frac{A_d}{m}\right)^m-1
\le
e^{A_d}-1
<1.
\end{align*}
Define $S_k^*:=S^*\cap\mathcal G_k$ for $k\in G^*$. By
Vandermonde's identity,
\[
N(\ell)
=
\sum_{\substack{
0\le \ell_k\le |S_k^*|\\
\sum_{k\in G^*}\ell_k=\ell}}
\prod_{k\in G^*}\binom{|S_k^*|}{\ell_k}
=
\binom{s^*}{\ell},
\qquad
\ell=1,\ldots,s^*.
\]

It remains to verify assumption~{\rm(A1)} in
Theorem~\ref{thm:complexity}. Let $c_0=c_0(\alpha,\zeta)>0$ be the
constant multiplying
$\lambda^{\alpha/2}(C_S-C_{S^*})$ in the overfitted-model
exponential bound in the proof of Theorem~\ref{thm:complexity}.
Fix $j\ge1$ and write $a:=c_0\lambda^{\alpha/2}$.
By choosing $c_{\alpha,\zeta}$ sufficiently large, the lower bound
on $\lambda$ in Corollary~\ref{cor:DS} ensures that $a>1+\gamma$.

For every $S\in\mathcal M_{\mathrm{sup}}$ with
$\Delta^+(S)=j$, let $t:=g-g^*$. Then $0\le t\le j$.
Since $S^*\subset S$, all remaining terms in the code are
nondecreasing, and hence
\[
C_S-C_{S^*}\ge t\log m+j\log d.
\]
Therefore,
\begin{align*}
\sum_{\substack{
S\in\mathcal M_{\mathrm{sup}}\\
\Delta^+(S)=j}}
\exp\{-a(C_S-C_{S^*})\} \le
\sum_{t=0}^{\min\{m-g^*,j\}}
\binom{m-g^*}{t}m^{-at}d^{-aj}
\sum_{j_1+j_2=j}
\binom{g^*d}{j_1}\binom{td}{j_2}.
\end{align*}
By Vandermonde's identity,
\[
\sum_{j_1+j_2=j}
\binom{g^*d}{j_1}\binom{td}{j_2}
=
\binom{(g^*+t)d}{j},
\]
and
\[
d^{-aj}\binom{(g^*+t)d}{j}
\le
(1+d^{-a})^{(g^*+t)d}.
\]
Consequently,
\begin{align*}
\sum_{\substack{
S\in\mathcal M_{\mathrm{sup}}\\
\Delta^+(S)=j}}
\exp\{-a(C_S-C_{S^*})\}
\le
(1+d^{-a})^{g^*d}
\left\{
1+m^{-a}(1+d^{-a})^d
\right\}^{m-g^*}.
\end{align*}
Since $a>1+\gamma$, we have
\[
(1+d^{-a})^{g^*d}
\le
\exp\{g^*d^{1-a}\}
\le
\exp\{d^{1+\gamma-a}\}
\le e.
\]
Moreover,
since $
(1+d^{-a})^d
\le
\exp\{d^{1-a}\}
\le e,
$
we have
\[
\left\{
1+m^{-a}(1+d^{-a})^d
\right\}^{m-g^*}
\le
\exp\{e\,m^{1-a}\}
\le e^e.
\]
These two inequalities imply that
\[
\sum_{\substack{
S\in\mathcal M_{\mathrm{sup}}\\
\Delta^+(S)=j}}
\exp\left\{
-c_0\lambda^{\alpha/2}(C_S-C_{S^*})
\right\}
\le e^{e+1}.
\]
Since $j\ge1$ and $\eta_n\to\infty$, for any fixed
$0<\widetilde c_1<c_0$,
\[
e^{e+1}
\le
\exp\left\{
\widetilde c_1(\eta_nj)^{\alpha/2}
\right\}
\]
for all sufficiently large $n$. Thus, assumption~{\rm(A1)} of
Theorem~\ref{thm:complexity} holds with $c_1=c_0$.

The remaining assumptions in Theorem~\ref{thm:complexity} are
exactly those imposed in Corollary~\ref{cor:DS}. The conclusion
therefore follows from Theorem~\ref{thm:complexity}.
\end{proof}

\subsection{Proof of Theorem~\ref{thm:example}}

\begin{proof}
We first establish lower bounds for \(\mathrm{APP}(S)\) over
\(S\in\mathcal M_{\mathrm{sub}}\cup\mathcal M_{\mathrm w}\).
Fix a candidate model \(S\) such that
\(\Delta^-(S)=|S^*\setminus S|\ge1\), and let
\[
T:=S\cap S^*,
\qquad
I:=S^*\setminus S.
\]
Since \(\mu_n=X_{S^*}\beta^*_{S^*}\) and
\(X_T\beta_T^*\in\operatorname{Col}(X_S)\), we have
\begin{align*}
(\mathrm{I}_n-P_S)\mu_n
&=(\mathrm{I}_n-P_S)X_I\beta_I^*,
\qquad
\mathrm{APP}(S)
=
\frac{\|(\mathrm{I}_n-P_S)X_I\beta_I^*\|_2^2}{\sigma^2}.
\end{align*}

\medskip
\noindent\textbf{Scenario 1.}
For every index set \(A\), \(\widehat\Sigma_{AA}\) is equicorrelated and
\(\lambda_{\min}(\widehat\Sigma_{AA})=1-\omega\).
For every \(S\in\mathcal M_{\mathrm{sub}}\cup\mathcal M_{\mathrm w}\),
\begin{align*}
\|(\mathrm{I}_n-P_S)X_I\beta_I^*\|_2^2
&=
\min_{\gamma\in\mathbb R^{|S|}}
\|X_I\beta_I^*-X_S\gamma\|_2^2 \\
&=
n\min_{\gamma}
\theta^\top\widehat\Sigma_{AA}\theta \\
&\ge
n(1-\omega)\min_{\gamma}\|\theta\|_2^2 \\
&\ge
n(1-\omega)\|\beta_I^*\|_2^2,
\end{align*}
where \(A:=I\cup S\) and
\(\theta:=(\beta_I^*,-\gamma)\in\mathbb R^{|A|}\).
Therefore,
\begin{align}
\mathrm{APP}(S)
&\ge
\frac{n(1-\omega)\|\beta_I^*\|_2^2}{\sigma^2}
\ge
\frac{n(1-\omega)\Delta^-(S)\beta_{\min}^2}{\sigma^2}.
\label{app_s1}
\end{align}

\medskip
\noindent\textbf{Scenario 2.}

\noindent\emph{(\romannumeral 1) \(S\in\mathcal M_{\mathrm{sub}}\).}
In this case, \(P_S=P_T\), and
\begin{align*}
\|(\mathrm{I}_n-P_S)X_I\beta_I^*\|_2^2
&=
\|(\mathrm{I}_n-P_T)X_I\beta_I^*\|_2^2 \\
&=
n(\beta_I^*)^\top
\widehat\Sigma_{I\mid T}\beta_I^*,
\end{align*}
where
\[
\widehat\Sigma_{I\mid T}
:=
\widehat\Sigma_{II}
-
\widehat\Sigma_{IT}
\widehat\Sigma_{TT}^{-1}
\widehat\Sigma_{TI}.
\]
Let \(A:=\widehat\Sigma_{S^*S^*}\). Then
\(\widehat\Sigma_{I\mid T}
=
\bigl((A^{-1})_{II}\bigr)^{-1}\), and hence
\[
\lambda_{\min}(\widehat\Sigma_{I\mid T})
=
\frac{1}{\lambda_{\max}((A^{-1})_{II})}
\ge
\frac{1}{\lambda_{\max}(A^{-1})}
=
\lambda_{\min}(A)
\ge
\lambda_*.
\]
Therefore,
\begin{align*}
\mathrm{APP}(S)
&\ge
\frac{n\lambda_*\|\beta_I^*\|_2^2}{\sigma^2}
\ge
\frac{n\lambda_*\Delta^-(S)\beta_{\min}^2}{\sigma^2},
\qquad
S\in\mathcal M_{\mathrm{sub}}.
\end{align*}

\smallskip
\noindent\emph{(\romannumeral 2) \(S\in\mathcal M_{\mathrm w}\) with
\(s\le2s^*\).}
Since \(\kappa_0\ge1\), we have
\(s\le2s^*\le(1+\kappa_0)s^*\), so the assumption in
Scenario~2 applies. Hence,
\[
\|(\mathrm{I}_n-P_S)X_I\beta_I^*\|_2^2
=
\|X_I\beta_I^*\|_2^2
-
\|P_SX_I\beta_I^*\|_2^2
\ge
\epsilon_0\|X_I\beta_I^*\|_2^2.
\]
Moreover, since \(I\subset S^*\), Cauchy interlacing yields
\[
\lambda_{\min}(\widehat\Sigma_{II})
\ge
\lambda_{\min}(\widehat\Sigma_{S^*S^*})
\ge
\lambda_*.
\]
Consequently,
\[
\|X_I\beta_I^*\|_2^2
=
n(\beta_I^*)^\top
\widehat\Sigma_{II}\beta_I^*
\ge
n\lambda_*\|\beta_I^*\|_2^2.
\]
Hence,
\begin{align}
\mathrm{APP}(S)
&\ge
\frac{\epsilon_0n\lambda_*
\|\beta_I^*\|_2^2}{\sigma^2}
\ge
\frac{\epsilon_0n\lambda_*
\Delta^-(S)\beta_{\min}^2}{\sigma^2},
\qquad
S\in\mathcal M_{\mathrm w},
\quad
s\le2s^*.
\label{app_s2}
\end{align}

\medskip
\noindent\textbf{Scenario 3.}

\noindent\emph{(\romannumeral 1) \(S\in\mathcal M_{\mathrm{sub}}\).}
Then \(S\subset S^*\) and \(X_I^\top X_S=0\), so
\(P_SX_I\beta_I^*=0\). Therefore,
\begin{align*}
\mathrm{APP}(S)
&=
\frac{\|X_I\beta_I^*\|_2^2}{\sigma^2}
=
\frac{n\|\beta_I^*\|_2^2}{\sigma^2} \\
&\ge
\frac{n\Delta^-(S)\beta_{\min}^2}{\sigma^2},
\qquad
S\in\mathcal M_{\mathrm{sub}}.
\end{align*}

\smallskip
\noindent\emph{(\romannumeral 2) \(S\in\mathcal M_{\mathrm w}\).}
For each \(j\in I\), if \(\widetilde j\in S\), then
\[
P_Sx_j=\omega x_{\widetilde j},
\qquad
\|x_j-P_Sx_j\|_2^2=n(1-\omega^2).
\]
If \(\widetilde j\notin S\), then
\[
P_Sx_j=0,
\qquad
\|x_j-P_Sx_j\|_2^2=n.
\]
Thus,
\[
\|x_j-P_Sx_j\|_2^2
\ge
n(1-\omega^2),
\qquad j\in I.
\]
Using the orthogonality across distinct signal--shadow pairs,
\[
\|(\mathrm{I}_n-P_S)X_I\beta_I^*\|_2^2
\ge
n(1-\omega^2)
\sum_{j\in I}(\beta_j^*)^2
\ge
n(1-\omega^2)
\Delta^-(S)\beta_{\min}^2.
\]
Therefore,
\begin{align}
\mathrm{APP}(S)
&\ge
\frac{n(1-\omega^2)
\Delta^-(S)\beta_{\min}^2}{\sigma^2},
\qquad
S\in\mathcal M_{\mathrm w}.
\label{app_s3}
\end{align}

We next verify the conditions of
Corollary~\ref{cor:all_subset}.
Throughout the remainder of the proof, \(c_\lambda>0\) denotes a
constant that may depend on the fixed value of \(\lambda\).
Write \(\ell:=\Delta^-(S)\) and \(j:=\Delta^+(S)\).

For \(S\in\mathcal M_{\mathrm{sub}}\), we have
\(s^*-s=\ell\), and
\[
\log\binom{s^*}{\ell}
\le
\ell\log p.
\]
Since \(\alpha=2\), the right-hand side of
\eqref{eq:suball} is bounded by
$
c_\lambda\ell(\log p+\log n).
$

For the wrong-model condition in
Corollary~\ref{cor:all_subset}, take its splitting constant to be
\(\kappa=1\).
We first verify that sufficiently large wrong models are excluded
automatically. Since
\(D_S\le C_S=s\log p\) and
\(r_S-\widetilde r_S\le s\), we have
\begin{align*}
2B_S-u_S
\le{}&
\frac{\lambda}{2}s\log p
+\frac12\log\log n
+4s \\
&-(s-s^*)\log n
-\lambda(s-s^*)\log p \\
={}&
-(s-s^*)\log n
+\lambda\left(s^*-\frac{s}{2}\right)\log p
+4s+\frac12\log\log n.
\end{align*}
If \(s>2s^*\), then
\(s-s^*>s/2\) and \(s^*-s/2<0\). It follows that
\[
2B_S-u_S
<
-\frac{s}{2}\log n
+4s+\frac12\log\log n
<0
\]
for all sufficiently large \(n\). Hence,
\(2B_S<u_S\) for every
\(S\in\mathcal M_{\mathrm w}\) with \(s>2s^*\).

It remains to consider
\(S\in\mathcal M_{\mathrm w}\) with \(s\le2s^*\).
Since \(r_S-\widetilde r_S\le j\), the additional rank term in
\eqref{eq:wrongall} satisfies
\(8(r_S-\widetilde r_S)\le8j\).

If \(s\le s^*\), then \(j\le\ell\). Using
\[
\log\binom{s^*}{\ell}\le\ell\log p,
\qquad
\log\binom{p-s^*}{j}\le j\log p,
\]
the right-hand side of \eqref{eq:wrongall} is bounded by
\begin{align*}
&\lambda(\ell+j)\log p
+2(\ell-j)\log n
+2\lambda(\ell-j)\log p
+\log\log n
+8j \\
&\hspace{3cm}\le
c_\lambda\ell(\log p+\log n)
\end{align*}
for all sufficiently large \(n\).

If \(s>s^*\), let \(d:=s-s^*=j-\ell>0\).
The right-hand side of \eqref{eq:wrongall} is then bounded by
\begin{align*}
&\lambda(\ell+j)\log p
-2d\log n
-2\lambda d\log p
+\log\log n
+8j \\
={}&
2\lambda\ell\log p
+8\ell+\log\log n
-d\{\lambda\log p+2\log n-8\} \\
\le{}&
c_\lambda\ell\log p \\
\le{}&
c_\lambda\ell(\log p+\log n)
\end{align*}
for all sufficiently large \(n\).

Therefore, for every underfitted or wrong model that is not
excluded automatically, the required lower bound is at most
\[
c_\lambda\Delta^-(S)(\log p+\log n).
\]
Combining this bound with
\eqref{app_s1}--\eqref{app_s3} yields, respectively,
\[
\beta_{\min}^2
\ge
\frac{c_\lambda\sigma^2(\log p+\log n)}
{(1-\omega)n},
\qquad
\beta_{\min}^2
\ge
\frac{c_\lambda\sigma^2(\log p+\log n)}
{\epsilon_0\lambda_*n},
\qquad
\beta_{\min}^2
\ge
\frac{c_\lambda\sigma^2(\log p+\log n)}
{(1-\omega^2)n}.
\]
Corollary~\ref{cor:all_subset} therefore gives the desired
selection consistency.
\end{proof}

\subsection{Proof of Theorem~\ref{thm:risk}}
To simplify notation in this proof,
let $S^{\circ}$ be the minimizer of $R_n(\mu_n;\mathcal M_n)$, and define
\[
\mathrm{rem}_1(S):=\varepsilon^\top(\mathrm{I}_n-P_S)\mu_n,
\qquad
\mathrm{rem}_2(S):=\sigma^2 r_S-\varepsilon^\top P_S\varepsilon.
\]
\begin{proof}
By the definition of DCIC
\begin{align*}
\sigma^2\mathrm{DCIC}(S)
&=\mathrm{RSS}(S)+\sigma^2\eta_n s+\lambda\sigma^2 C_S^{\frac{2}{\alpha}}\\
&=\|(\mathrm{I}_n-P_S)\mu_n\|_2^2
+2\varepsilon^\top(\mathrm{I}_n-P_S)\mu_n
+\varepsilon^\top(\mathrm{I}_n-P_S)\varepsilon\\
&\quad +\sigma^2\eta_n s+\lambda\sigma^2 C_S^{\frac{2}{\alpha}}\\
&=nR_n(\mu_n;S)+2\mathrm{rem}_1(S)+\mathrm{rem}_2(S)+\varepsilon^\top\varepsilon.
\end{align*}
We first show that, with probability at least $1-\delta$, the following two inequalities hold simultaneously for every $S \in \mathcal{M}_n$:
\begin{equation}\label{eq:rem1}
|\mathrm{rem}_1(S)|
\le \frac{1}{12}nR_n(\mu_n;S)+c_{1,\alpha,\zeta}\sigma^2\Bigl(\log(4/\delta)\Bigr)^\frac{2}{\alpha},
\end{equation}
and
\begin{equation}\label{eq:rem2}
|\mathrm{rem}_2(S)|
\le \frac{1}{12}nR_n(\mu_n;S)+c_{2,\alpha,\zeta}\sigma^2\Bigl(\log(4/\delta)\Bigr)^\frac{2}{\alpha}.
\end{equation}
We verify these two inequalities at the end of the proof.

By the optimality of $\widehat S$, we have $\sigma^2 \mathrm{DCIC}(\widehat S) \leq \sigma^2 \mathrm{DCIC}(S^\circ)$. Consequently, 
$$
nR_n(\mu_n;\widehat S)+2\mathrm{rem}_1(\widehat S)+\mathrm{rem}_2(\widehat S)+\varepsilon^\top\varepsilon \leq nR_n^*(\mu_n;\mathcal M_n)+2\mathrm{rem}_1(S^\circ)+\mathrm{rem}_2(S^\circ)+\varepsilon^\top\varepsilon.
$$
On the event where \eqref{eq:rem1}--\eqref{eq:rem2} hold, the preceding inequality yields 
\begin{equation}\label{eq:R}
    nR_n(\mu_n;\widehat S) \leq \frac{5}{3}nR_n^*(\mu_n;\mathcal M_n)+c_{3,\alpha,\zeta}\sigma^2\Bigl(\log(4/\delta)\Bigr)^\frac{2}{\alpha}.
\end{equation}
Note that $\mathrm{ASE}(S)=\frac1n\|(\mathrm{I}_n-P_S)\mu_n\|_2^2+\frac1n\varepsilon^\top P_S\varepsilon$.
Hence
\begin{align*}
n\,\mathrm{ASE}(S)+\sigma^2(\eta_n s-2r_S)+\lambda\sigma^2 C_S^{\frac{2}{\alpha}}
=nR_n(\mu_n;S)-\mathrm{rem}_2(S).
\end{align*}
Since $s \geq r_S$ and $\eta_n \geq 2$, we have 
\begin{align}\label{eq:R2}
\begin{split}
n\,\mathrm{ASE}(\hat S) &\leq nR_n(\mu_n;\hat S)+|\mathrm{rem}_2(\hat S)|\\
&\leq \frac{13}{12}nR_n(\mu_n;\hat S)+c_{4,\alpha,\zeta}\sigma^2\Bigl(\log(4/\delta)\Bigr)^\frac{2}{\alpha},
\end{split}
\end{align}
where the second inequality follows from \eqref{eq:rem2}.

Consequently, combining \eqref{eq:R} and \eqref{eq:R2}, with probability at least  $1-\delta$, we have
\begin{equation}
    \mathrm{ASE}(\widehat S) \leq 2 R_n^*(\mu_n;\mathcal M_n)+c_{5,\alpha,\zeta}\frac{\sigma^2\Bigl(\log(4/\delta)\Bigr)^\frac{2}{\alpha}}{n}.
\end{equation}
Let $Z:=\left(\mathrm{ASE}(\widehat S)-2R_n^*(\mu_n;\mathcal M_n)\right)_+$.
The previous inequality implies that for all $\delta\in(0,1)$,
\[
\Pr\Big(Z> c_{5,\alpha,\zeta}\frac{\sigma^2}{n}\big(\log(4/\delta)\big)^{\frac{2}{\alpha}}\Big)\le \delta.
\]
Set $\delta=4e^{-t}$ and integrate:
\begin{align*}
\mathbb{E}[Z]
&\leq \int_0^\infty \Pr(Z \geq z)\,dz\\
&\leq c_{5,\alpha,\zeta}\frac{\sigma^2}{n}\left((\log 4)^\frac{2}{\alpha}+\frac{8}{\alpha}\int_{\log 4}^\infty t^{\frac{2}{\alpha}-1}\,e^{-t}\,dt\right)\\
&= c_{6,\alpha,\zeta}\frac{\sigma^2}{n}.
\end{align*}
Therefore,
\[
\mathbb{E}\big\{\mathrm{ASE}(\widehat S)\big\}
\le 2\,R_n^*(\mu_n;\mathcal M_n)+c_{6,\alpha, \zeta}\frac{\sigma^2}{n}.
\]

Next, we verify \eqref{eq:rem1} and \eqref{eq:rem2}. 

\medskip
\noindent
\textbf{Control of \(\mathrm{rem}_1(S)\).}
Lemma \ref{lem:sum} implies that, for some \(c_{\alpha,\zeta}>0\) and all \(t\ge0\),
\[
\Pr\bigl(|\mathrm{rem}_1(S)|\ge t\bigr)
\le
2\exp\!\left\{
-c_{\alpha, \zeta}
\min\!\left(
\Bigl(\frac{t}{\sigma\|(\mathrm{I}_n-P_S)\mu_n\|_2}\Bigr)^2,
\Bigl(\frac{t}{\sigma\|(\mathrm{I}_n-P_S)\mu_n\|_2}\Bigr)^\alpha
\right)
\right\}.
\]
Let
$
t=\sigma \|(\mathrm{I}_n-P_S)\mu_n\|_2\, x_S^{1/\alpha}
 $.
Then, if $x_S \geq 1$, we have
\[
\min\!\left(
\Bigl(\frac{t}{\sigma\|(\mathrm{I}_n-P_S)\mu_n\|_2}\Bigr)^2,
\Bigl(\frac{t}{\sigma\|(\mathrm{I}_n-P_S)\mu_n\|_2}\Bigr)^\alpha
\right)
=
\min(x_S^{2/\alpha},x_S)
=
x_S.
\] 
Let $c_{1,\alpha,\zeta} := \max\{1, 1/c_{\alpha,\zeta}\}$, and $x_S:=c_{1, \alpha, \zeta}(C_S+\log (4/\delta))$. 
Since $x_S\geq 1$, we have
\[
\Pr\bigl(|\mathrm{rem}_1(S)|\ge \sigma \|(\mathrm{I}_n-P_S)\mu_n\|_2\, x_S^{1/\alpha}\bigr)\le 2e^{-C_S-\log(4/\delta)}.
\]
Taking a union bound over $S \in \mathcal{M}_n$ and using the Kraft inequality gives
\begin{align*}
    &\Pr\bigl(\exists S \in \mathcal{M}_n: |\mathrm{rem}_1(S)| \geq \sigma \|(\mathrm{I}_n-P_S)\mu_n\|_2\, x_S^{1/\alpha}\bigr)\\
    \leq& 2e^{-\log(4/\delta)}\sum_{S\in \mathcal{M}_n}e^{-C_S}\\
    \leq& \frac{\delta}{2}.
\end{align*}
Moreover, by Young's inequality,
\begin{align*}
\sigma\|(\mathrm{I}_n-P_S)\mu_n\|_2\cdot  x_S^{1/\alpha}
&\le
\frac{1}{12}\|(\mathrm{I}_n-P_S)\mu_n\|_2^2+3\sigma^2c^2_{1,\alpha,\zeta}x_S^{2/\alpha}\\
&\le \frac{1}{12}\|(\mathrm{I}_n-P_S)\mu_n\|_2^2+c'_{1,\alpha,\zeta}\sigma^2 C^{\frac{2}{\alpha}}_S+c'_{1,\alpha,\zeta}\sigma^2\Bigl(\log(4/\delta)\Bigr)^{\frac{2}{\alpha}}\\
&\le \frac{1}{12}\|(\mathrm{I}_n-P_S)\mu_n\|_2^2+\frac{1}{12}\lambda\sigma^2 C^{\frac{2}{\alpha}}_S+c'_{1,\alpha,\zeta}\sigma^2\Bigl(\log(4/\delta)\Bigr)^{\frac{2}{\alpha}}\\
&\leq \frac{1}{12}n R_n(\mu_n;S)+c'_{1,\alpha,\zeta}\sigma^2\Bigl(\log(4/\delta)\Bigr)^{\frac{2}{\alpha}},
\end{align*}
where the third inequality follows from $\lambda \geq 12 c'_{1,\alpha,\zeta}$, and the last inequality follows from the definition of $R_n(\mu_n;S)$.
Therefore, with probability at least $1-\frac{\delta}{2}$, for every $S \in \mathcal{M}_n$, we have 
 $$|\mathrm{rem}_1(S)|\leq \frac{1}{12}n R_n(\mu_n;S)+c'_{1,\alpha,\zeta}\sigma^2\Bigl(\log(4/\delta)\Bigr)^{\frac{2}{\alpha}}.$$
\medskip
\noindent
\textbf{Control of \(\mathrm{rem}_2(S)\).}
Note $\mathbb{E}(\varepsilon^\top P_S \varepsilon) = \sigma^2 \text{trace}(P_S) = \sigma^2 r_S$.
By Lemma~\ref{lem:hanson}, for some \(c_{\alpha,\zeta}>0\) and all \(t\ge0\),
\[
\Pr\bigl(|\mathrm{rem}_2(S)|\ge t\bigr)
\le
2\exp\!\left\{
-c_{\alpha,\zeta}\min\!\left(
\frac{t^2}{\sigma^4r_S},
\Bigl(\frac{t}{\sigma^2}\Bigr)^{\alpha/2}
\right)
\right\}.
\]
Let
$t=\sigma^2\bigl(\sqrt{r_Sx_S}+x_S^{2/\alpha}\bigr).
$
Then, we have
\[
\frac{t^2}{\sigma^4r_S}\ge x_S,
\qquad
\Bigl(\frac{t}{\sigma^2}\Bigr)^{\alpha/2}\ge x_S.
\]
Set $c_{2,\alpha,\zeta} := \max\{1, 1/c_{\alpha,\zeta}\}$, and $x_S:=c_{2, \alpha, \zeta}(C_S+\log (4/\delta))$. Then,
\[
\Pr\bigl(|\mathrm{rem}_2(S)|\ge t\bigr)\le 2e^{-C_S-\log (4/\delta)}.
\]
Taking a union bound over $S \in \mathcal{M}_n$ and using the Kraft inequality gives
\begin{align*}
    &\Pr\bigl(\exists S \in \mathcal{M}_n: |\mathrm{rem}_2(S)| \geq \sigma^2\bigl(\sqrt{r_Sx_S}+x_S^{2/\alpha}\bigr)\bigr)\\
    \leq& 2e^{-\log(4/\delta)}\sum_{S\in \mathcal{M}_n}e^{-C_S}\\
    \leq& \frac{\delta}{2}.
\end{align*}
By Young's inequality and $x_S \geq 1$, we have
\[
\sqrt{r_Sx_S}\le \frac{1}{12}r_S+3x_S\le \frac{1}{12}r_S+3x_S^{2/\alpha}.
\]
Consequently, we obtain
\begin{align*}
  t&\leq \frac{1}{12}\sigma^2r_S+4\sigma^2x_S^{2/\alpha}\\
  &\leq \frac{1}{12}\sigma^2(\eta_ns -r_S)+c'_{2,\alpha,\zeta}\sigma^2 C^{\frac{2}{\alpha}}_S+c'_{2,\alpha,\zeta}\sigma^2\Bigl(\log(4/\delta)\Bigr)^{\frac{2}{\alpha}}\\
    &\leq \frac{1}{12}\sigma^2(\eta_ns -r_S)+\frac{1}{12}\lambda\sigma^2 C^{\frac{2}{\alpha}}_S+c'_{2,\alpha,\zeta}\sigma^2\Bigl(\log(4/\delta)\Bigr)^{\frac{2}{\alpha}}\\
  &\leq \frac{1}{12} n R_n(\mu_n;S)+c'_{2,\alpha,\zeta}\sigma^2\Bigl(\log(4/\delta)\Bigr)^{\frac{2}{\alpha}},
\end{align*}
where the second inequality follows from $\eta_n\geq2$, and the third inequality follows from $\lambda \geq 12 c'_{2,\alpha,\zeta}$. Therefore, with probability at least $1-\frac{\delta}{2}$, for every $S \in \mathcal{M}_n$, we have 
 $$|\mathrm{rem}_2(S)|\leq \frac{1}{12}n R_n(\mu_n;S)+c'_{2,\alpha,\zeta}\sigma^2\Bigl(\log(4/\delta)\Bigr)^{\frac{2}{\alpha}}.$$

Finally, choose $\lambda_{\rm risk}>12 \max\{c'_{1,\alpha,\zeta}, c'_{2,\alpha,\zeta}\}$. Then, for every $\lambda \geq \lambda_{\rm risk}$, with probability at least \(1-\delta\),
both \eqref{eq:rem1} and \eqref{eq:rem2} hold simultaneously for all \(S\in\mathcal M_n\).
\end{proof}

\subsection{Precise formulation and proof of Corollary~\ref{cor:multiclass_consistency} using class-model pairs}\label{app:multiclass_consistency}

For \((m,S)\in\widetilde{\mathcal{M}}_n\), let \(X_S^{(m)}\) denote the associated design submatrix under class \(m\), and let
$
V_S^{(m)}:=\operatorname{Col}(X_S^{(m)})\subset\mathbb R^n
$
be its column space. Denote by \(P_S^{(m)}\) the orthogonal projection onto \(V_S^{(m)}\), and let
$
r_{(m,S)}:=\operatorname{rank}(P_S^{(m)}).
$
In addition, define
\[
\mathrm{APP}(m,S):=\frac{\|(\mathrm{I}_n-P_S^{(m)})\mu_n\|_2^2}{\sigma^2}.
\]
We divide the wrong model class $\widetilde{\mathcal{M}}_{\mathrm{w}}$ into $\widetilde{\mathcal{M}}_{\mathrm{w}}^{\mathrm{in}} = \{(m, S) \in \widetilde{\mathcal{M}}_{\mathrm{w}}: m = m^*\}$ and $\widetilde{\mathcal{M}}_{\mathrm{w}}^{\mathrm{out}} = \{(m, S) \in \widetilde{\mathcal{M}}_{\mathrm{w}}: m \neq m^*\}$.
For \((m^*,S)\in\widetilde{\mathcal{M}}_{\mathrm{sup}}\cup\widetilde{\mathcal{M}}_{\mathrm{sub}}\cup \widetilde{\mathcal{M}}_{\mathrm{w}}^{\mathrm{in}}\), define
\[
\widetilde{\Delta}^+(S):=\bigl|\mathcal{U}^{(m^*)}(S)\setminus \mathcal{U}^{(m^*)}(S^*)\bigr|,
\qquad
\widetilde{\Delta}^-(S):=\bigl|\mathcal{U}^{(m^*)}(S^*)\setminus \mathcal{U}^{(m^*)}(S)\bigr|,
\]
where $\mathcal{U}^{m}(S)$ denotes the selected atom set of model $S$ within class $m$. 
For \(\ell\ge1\),
\[
\widetilde{N}(\ell):=
\bigl|
\{(m,S)\in\widetilde{\mathcal{M}}_{\mathrm{sub}}:\widetilde{\Delta}^-(S)=\ell\}
\bigr|.
\]
For every \((m,S)\in\widetilde{\mathcal{M}}_n\), let \(\widetilde P_S^{(m)}\) be the orthogonal projection onto
$
V_S^{(m)}\cap V_{S^*}^{(m^*)},
$
and
$
\widetilde r_{(m,S)}:=\operatorname{rank}(\widetilde P_S^{(m)}).
$
Define
$
u_{(m,S)}
:=
\eta_n\bigl(s_{(m,S)}-s^*\bigr)
+
\lambda\Bigl(C_{(m,S)}^{2/\alpha}-C_{(m^*,S^*)}^{2/\alpha}\Bigr),
$
and
\[
B_{(m,S)}
:=
\frac14\lambda D_{(m,S)}^{2/\alpha}
+
\frac14\log\log n
+
2\bigl(r_{(m,S)}-\widetilde r_{(m,S)}\bigr),
\]
where $q^*=|\mathcal{U}^{(m^*)}(S^*)|$, $q = |\mathcal{U}^{(m^*)}|$, and
\[
D_{(m,S)}
=
\begin{cases}
\displaystyle
\log\left\{
\binom{q^*}{\widetilde{\Delta}^{-}(S)}
\binom{q-q^*}{\widetilde{\Delta}^{+}(S)}
\right\}\wedge C_{(m,S)},
&
(m,S)\in\widetilde{\mathcal{M}}_{\mathrm{w}}^{\mathrm{in}},
\\[1.2ex]
C_{(m,S)},
&
(m,S)\in\widetilde{\mathcal{M}}_{\mathrm{w}}^{\mathrm{out}}.
\end{cases}
\]
The following corollary provides a precise formulation of Corollary~\ref{cor:multiclass_consistency} in terms of class-model pairs.

\begin{corollary}
Suppose there exists a unique true pair \((m^*,S^*)\in\widetilde{\mathcal{M}}_n\), and for every $\lambda \geq \lambda_{\rm sel}$, assume the following conditions hold:

\begin{itemize}
\item[(\romannumeral 1)] For every \((m, S)\in \widetilde{\mathcal M}_{\mathrm{sup}}\), \(C_{(m^*, S)}\ge C_{(m^*,S^*)}\). In addition, there exist constants \(0<\widetilde c_1<c_1\) such that for all \(j\ge1\),
\[
\sum_{\substack{(m,S)\in\widetilde{\mathcal{M}}_{\mathrm{sup}}\\ \widetilde\Delta^+(S)=j}}
\exp\!\bigl\{-c_1\lambda^{\alpha/2}\bigl(C_{(m,S)}-C_{(m^*,S^*)}\bigr)\bigr\}
\le
\exp\!\bigl\{\widetilde c_1(\eta_n j)^{\alpha/2}\bigr\};
\]

\item[(\romannumeral 2)]There exists a sufficiently large constant $c_2>1$ such that, for every \((m,S)\in\widetilde{\mathcal{M}}_{\mathrm{sub}}\),
\[
\mathrm{APP}(m,S)\ge
c_2\Bigg(
\eta_n\bigl(s^*-s_{(m,S)}\bigr)
+\lambda\Bigl(C_{(m^*,S^*)}^{2/\alpha}-C_{(m,S)}^{2/\alpha}\Bigr)_+
+\log^{2/\alpha} \widetilde{N}\bigl(\widetilde\Delta^-(S)\bigr)
\Bigg);
\]

\item[(\romannumeral 3)] For every \((m,S)\in\widetilde{\mathcal{M}}_{\mathrm{w}}\) satisfying
\(2B_{(m,S)}>u_{(m,S)}\),
\[
\mathrm{APP}(m,S)\ge 2\bigl(2B_{(m,S)}-u_{(m,S)}\bigr).
\]
\end{itemize}

Then
\[
\Pr\bigl((\hat m,\hat S)=(m^*,S^*)\bigr)\to1.
\]
\end{corollary}

\begin{proof}
First, the aggregated complexity satisfies the Kraft inequality on \(\widetilde{\mathcal{M}}_n\):
\[
\sum_{(m,S)\in\widetilde{\mathcal{M}}_n} e^{-C_{(m,S)}}\le1.
\]

Next, assumptions (\romannumeral 1) and (\romannumeral 2) are the pair-indexed extensions of assumptions
(A1) and (A2) in Lemma~\ref{thm:main}, with the single-class candidate \(S\) replaced by
the class-model pair \((m,S)\). Therefore, it remains to verify assumption (A3) in
Lemma~\ref{thm:main}.

\medskip
\noindent
{\bf Verification of (A3a).}
Since \(\widetilde P_S^{(m)}\) is the orthogonal projection onto
$
V_S^{(m)}\cap V_{S^*}^{(m^*)}\subset V_S^{(m)},
$
the matrix
$
P_S^{(m)}-\widetilde P_S^{(m)}
$
is an orthogonal projection of rank
$
r_{(m,S)}-\widetilde r_{(m,S)}.
$
Hence
\[
\|P_S^{(m)}-\widetilde P_S^{(m)}\|_F^2
=
r_{(m,S)}-\widetilde r_{(m,S)},
\qquad
\|P_S^{(m)}-\widetilde P_S^{(m)}\|_{\mathrm{op}}\leq1.
\]
As in the proof of Theorem~\ref{thm:complexity}, we obtain
\begin{align*}
&\sum_{(m,S)\in\widetilde{\mathcal{M}}_{\mathrm{w}}}
\Pr\!\left(
\frac{\varepsilon^\top(P_S^{(m)}-\widetilde P_S^{(m)})\varepsilon}{\sigma^2}
>
B_{(m,S)}
\right)\\
&\le
\sum_{(m,S)\in\widetilde{\mathcal{M}}_{\mathrm{w}}}
2\exp\!\left\{
-c_{1,\alpha, \zeta}
\bigl(\lambda D_{(m,S)}^{2/\alpha}+\log\log n+r_{(m,S)}-\tilde r_{(m,S)}\bigr)^{\alpha/2}
\right\}\\
&\le
\sum_{(m,S)\in\widetilde{\mathcal{M}}_{\mathrm{w}}}
2\exp\!\left\{
-c_{2,\alpha, \zeta}\Bigl(
\lambda^{\alpha/2}D_{(m,S)}+(\log\log n)^{\alpha/2}
\Bigr)
\right\}\\
&\le
2\exp\!\bigl\{-c_{2,\alpha, \zeta}(\log\log n)^{\alpha/2}\bigr\}
\Biggl(\sum_{(m,S)\in\widetilde{\mathcal{M}}_{\mathrm{w}}^{\mathrm{in}}} e^{-2D_{(m,S)}}+\sum_{(m,S) \in \widetilde{\mathcal{M}}_{\mathrm{w}}^{\mathrm{out}}} e^{-2D_{(m,S)}}\Biggr)\\
&\to0,
\end{align*}
where the third inequality follows by the choice of $\lambda_{\rm sel}$, and the last inequality follows from the same argument as in the proof of Theorem~\ref{thm:complexity}, together with the Kraft inequality.

\medskip
\noindent
{\bf Verification of (A3b).}
Define
\[
\Gamma_{(m,S)}
:=
\mathrm{APP}(m,S)+u_{(m,S)}-B_{(m,S)},
\qquad (m,S)\in\widetilde{\mathcal{M}}_{\mathrm{w}}.
\]
As in the proof of inequality \eqref{eq:gammas}, we have
\[
\frac{\Gamma_{(m,S)}}{\sqrt{\mathrm{APP}(m,S)}}\ge c\sqrt{B_{(m,S)}}
\]
for some universal constant \(c>0\).

Using Lemma~\ref{lem:sum}, we obtain
\begin{align*}
&\ \sum_{(m,S)\in\widetilde{\mathcal{M}}_{\mathrm{w}}}
\Pr\!\left(
\frac{2\mu_n^\top(P_{S^*}^{(m^*)}-P_S^{(m)})\varepsilon}{\sigma^2}
+\Gamma_{(m,S)}
\le0
\right)\\
&\le
2\sum_{(m,S)\in\widetilde{\mathcal{M}}_{\mathrm{w}}}
\exp\!\left\{
-c_{3,\alpha, \zeta}
\left(
\frac{\Gamma_{(m,S)}}{\sqrt{\mathrm{APP}(m,S)}}
\right)^\alpha
\right\}\\
&\le
2\sum_{(m,S)\in\widetilde{\mathcal{M}}_{\mathrm{w}}}
\exp\!\bigl\{-c_{4,\alpha, \zeta}B_{(m,S)}^{\alpha/2}\bigr\}\\
&\le
2\exp\!\bigl\{-c_{5,\alpha, \zeta}(\log\log n)^{\alpha/2}\bigr\}
\Biggl(\sum_{(m,S)\in\widetilde{\mathcal{M}}_{\mathrm{w}}^{\mathrm{in}}} e^{-2D_{(m,S)}}+\sum_{(m,S) \in \widetilde{\mathcal{M}}_{\mathrm{w}}^{\mathrm{out}}} e^{-2D_{(m,S)}}\Biggr)\\
&\to0.
\end{align*}
This proves (A3b), and hence assumption (A3) of Lemma~\ref{thm:main} holds on the aggregated list.

Therefore, all assumptions of Lemma~\ref{thm:main} are satisfied for the pair-indexed candidate list
\(\widetilde{\mathcal{M}}_n\). It follows that
\[
\Pr\!\left(
\min_{\substack{(m,S)\in\widetilde{\mathcal{M}}_n\\ (m,S)\neq(m^*,S^*)}}
\mathrm{DCIC}(m,S)
>
\mathrm{DCIC}(m^*,S^*)
\right)\to1,
\]
which is equivalent to
\[
\Pr\bigl((\hat m,\hat S)=(m^*,S^*)\bigr)\to1.
\]
This completes the proof.
\end{proof}

\subsection{Proof of Corollary~\ref{cor:basis_consistency}}
\begin{proof}
Let
\[
\widetilde{\mathcal M}_{n}
=
\bigcup_{m=1}^{M}
\{(m,S):S\in\mathcal M_n^{(m)}\},
\qquad M=3.
\]
Fix any $\lambda\ge\lambda_{\mathrm{basis}}$.
By Corollary~\ref{cor:multiclass_consistency}, it suffices to verify
the three conditions in Theorem~\ref{thm:complexity} on
$\widetilde{\mathcal M}_{n}$.

\medskip
\noindent
\textbf{Case 1: Overfitted pairs.}
For every $(m,S)\in\widetilde{\mathcal M}_{\mathrm{sup}}$, we have
$m=m^*$, and hence
\[
C_{(m,S)}-C_{(m^*,S^*)}
=
\bigl(C_S^{(m^*)}+\log M\bigr)
-
\bigl(C_{S^*}^{(m^*)}+\log M\bigr)
=
C_S^{(m^*)}-C_{S^*}^{(m^*)}.
\]
Thus, the additive term $\log M$ cancels. Since the within-class
overfitted summability condition for the code lengths in
Section~\ref{sub:example_family} is unchanged by aggregation,
assumption~(\romannumeral 1) of Theorem~\ref{thm:complexity} holds on
$\widetilde{\mathcal M}_{n}$.

\medskip
\noindent
\textbf{Case 2: Underfitted pairs.}
For every $(m^*,S)\in\widetilde{\mathcal M}_{\mathrm{sub}}$, condition~(\romannumeral 1)
of the corollary gives
\[
\widetilde{\Delta}^{-}(S)
\le
s^*-s_{(m^*,S)}
\le
K_{\max}\widetilde{\Delta}^{-}(S).
\]
Moreover, for some constant $c_3>0$,
\[
\bigl(
C_{(m^*,S^*)}-C_{(m^*,S)}
\bigr)_+
\le
c_3\widetilde{\Delta}^{-}(S)\log p
\le
c_3\bigl(s^*-s_{(m^*,S)}\bigr)\log p.
\]
We also have
\[
\log
\widetilde N\bigl(\widetilde{\Delta}^{-}(S)\bigr)
\le
\bigl(s^*-s_{(m^*,S)}\bigr)\log p.
\]
Therefore, for some constant $c_4>0$,
\begin{align*}
&\eta_n\bigl(s^*-s_{(m^*,S)}\bigr)
+\lambda
\bigl(
C_{(m^*,S^*)}-C_{(m^*,S)}
\bigr)_+
+\log
\widetilde N\bigl(\widetilde{\Delta}^{-}(S)\bigr)
\\
&\qquad\le
\bigl(s^*-s_{(m^*,S)}\bigr)\log n
+\lambda c_3
\bigl(s^*-s_{(m^*,S)}\bigr)\log p
+\bigl(s^*-s_{(m^*,S)}\bigr)\log p
\\
&\qquad\le
c_4
\bigl(s^*-s_{(m^*,S)}\bigr)
\bigl(\log n+\lambda\log p\bigr).
\end{align*}
By assumption~(\romannumeral 2),
\[
\mathrm{APP}(m^*,S)
\ge
c_1
\bigl(s^*-s_{(m^*,S)}\bigr)
\bigl(\log n+\lambda\log p\bigr).
\]
Taking $c_1$ sufficiently large verifies assumption~(\romannumeral 2) of
Theorem~\ref{thm:complexity} for all
$(m^*,S)\in\widetilde{\mathcal M}_{\mathrm{sub}}$.

\medskip
\noindent
\textbf{Case 3: Wrong pairs.}
Fix $(m,S)\in\widetilde{\mathcal M}_{\mathrm w}$, and let $G$ denote
its active component set. By assumption~(\romannumeral 1), each active component
contributes at least one and at most $K_{\max}$ selected basis terms.
Therefore,
\[
|G|
\le s_{(m,S)}
\le K_{\max}|G|.
\]
The code lengths in Section~\ref{sub:example_family} imply that, for
some constant $c_5>0$,
\[
C_{(m^*,S^*)}\le c_5s^*\log p,
\qquad
C_{(m,S)}\ge |G|\log\frac{ep}{|G|}.
\]

We first consider $(m,S)$ satisfying
$s_{(m,S)}>(1+\kappa)s^*$. Since
$|G|\ge s_{(m,S)}/K_{\max}$, we have
\[
|G|>
\frac{(1+\kappa)s^*}{K_{\max}}.
\]
The function $x\mapsto x\log(ep/x)$ is nondecreasing on $(0,p]$.
Thus, whenever such a candidate exists,
\[
C_{(m,S)}
\ge
\frac{(1+\kappa)s^*}{K_{\max}}
\log\left\{
\frac{epK_{\max}}{(1+\kappa)s^*}
\right\}.
\]
Moreover, assumption~(\romannumeral 1) and $g^*\le p^\gamma$ imply
$s^*\le K_{\max}g^*\le K_{\max}p^\gamma$. Hence, for every fixed
$\kappa$ and all sufficiently large $p$,
\[
\log\left\{
\frac{epK_{\max}}{(1+\kappa)s^*}
\right\}
\ge
\frac{1-\gamma}{2}\log p,
\]
and therefore
\[
C_{(m,S)}
\ge
\frac{(1+\kappa)(1-\gamma)}
     {2K_{\max}}
s^*\log p.
\]
Taking the constant $\kappa$ in assumption~(\romannumeral 3) sufficiently large
so that
\[
\frac{(1+\kappa)(1-\gamma)}
     {2K_{\max}}
\ge 2c_5,
\]
we obtain
$
C_{(m,S)}\ge 2C_{(m^*,S^*)}
$
for every $(m,S)\in\widetilde{\mathcal M}_{\mathrm w}$ satisfying
$s_{(m,S)}>(1+\kappa)s^*$.

Since $D_{(m,S)}\le C_{(m,S)}$ and
$r_{(m,S)}-\widetilde r_{(m,S)}\le s_{(m,S)}$, it follows that
\begin{align*}
2B_{(m,S)}-u_{(m,S)}
&=
\frac{\lambda}{2}D_{(m,S)}
-\bigl(s_{(m,S)}-s^*\bigr)\log n
+\lambda C_{(m^*,S^*)}
-\lambda C_{(m,S)}
\\
&\qquad
+\frac12\log\log n
+4\bigl(r_{(m,S)}-\widetilde r_{(m,S)}\bigr)
\\
&\le
-\bigl(s_{(m,S)}-s^*\bigr)\log n
+\lambda C_{(m^*,S^*)}
-\frac{\lambda}{2}C_{(m,S)}
\\
&\qquad
+\frac12\log\log n
+4s_{(m,S)}
\\
&\le
-\bigl(s_{(m,S)}-s^*\bigr)\log n
+\frac12\log\log n
+4s_{(m,S)}.
\end{align*}
Because $s_{(m,S)}>(1+\kappa)s^*$,
\[
s_{(m,S)}-s^*
>
\frac{\kappa}{1+\kappa}s_{(m,S)}.
\]
Consequently,
\[
2B_{(m,S)}-u_{(m,S)}
\le
-\frac{\kappa}{1+\kappa}s_{(m,S)}\log n
+4s_{(m,S)}
+\frac12\log\log n
<0
\]
for all sufficiently large $n$. Hence,
$
2B_{(m,S)}<u_{(m,S)}
$
for every $(m,S)\in\widetilde{\mathcal M}_{\mathrm w}$ with
$s_{(m,S)}>(1+\kappa)s^*$.

It remains to consider
$(m,S)\in\widetilde{\mathcal M}_{\mathrm w}$ satisfying
$s_{(m,S)}\le(1+\kappa)s^*$. For such a pair,
\begin{align*}
2\bigl(2B_{(m,S)}-u_{(m,S)}\bigr)
&=
\lambda D_{(m,S)}
+2\bigl(s^*-s_{(m,S)}\bigr)\log n
+2\lambda C_{(m^*,S^*)}
-2\lambda C_{(m,S)}
\\
&\qquad
+\log\log n
+8\bigl(r_{(m,S)}-\widetilde r_{(m,S)}\bigr)
\\
&\le
2\bigl(s^*-s_{(m,S)}\bigr)_+\log n
+2\lambda C_{(m^*,S^*)}
+\log\log n
+8s_{(m,S)}
\\
&\le
c_7\left\{
\bigl(s^*-s_{(m,S)}\bigr)_+\log n
+\lambda s^*\log p
+\log\log n
\right\}
\end{align*}
for some constant $c_7>0$, where the first inequality follows from
$D_{(m,S)}\le C_{(m,S)}$, and the last inequality uses
$C_{(m^*,S^*)}\le c_5s^*\log p$,
$s_{(m,S)}\le(1+\kappa)s^*$, and
$\lambda\ge\lambda_{\mathrm{basis}}>0$.

By assumption~(\romannumeral 3),
\[
\mathrm{APP}(m,S)
\ge
c_2\left\{
\bigl(s^*-s_{(m,S)}\bigr)_+\log n
+\lambda s^*\log p
+\log\log n
\right\}
\]
for every $(m,S)\in\widetilde{\mathcal M}_{\mathrm w}$ with
$s_{(m,S)}\le(1+\kappa)s^*$. Taking $c_2$ sufficiently large gives
\[
\mathrm{APP}(m,S)
\ge
2\bigl(2B_{(m,S)}-u_{(m,S)}\bigr)
\]
whenever $2B_{(m,S)}>u_{(m,S)}$. Thus, assumption~(\romannumeral 3) of
Theorem~\ref{thm:complexity} is verified.

\medskip
All three conditions of Theorem~\ref{thm:complexity} therefore hold
on $\widetilde{\mathcal M}_{n}$. Corollary~\ref{cor:multiclass_consistency}
yields
\[
\Pr\bigl\{(\widehat m,\widehat S)=(m^*,S^*)\bigr\}\to1.
\]
Since the choice of $\lambda\ge\lambda_{\mathrm{basis}}$ was
arbitrary, this completes the proof.
\end{proof}

\subsection{Proof of Corollary~\ref{cor:additive_multiclass_oracle}}

\begin{proof}
Let $S^\circ\in\mathcal M_n^{(m^\circ)}$ be the admissible
candidate formed by the active set $G^*$ and the componentwise
approximants $\{f^\circ_{j,m^\circ}:j\in G^*\}$. By the multi-class
risk bound,
\[
\mathbb E\bigl\{\mathrm{ASE}(\widehat m,\widehat S)\bigr\}
\le
2R_n^*(\mu_n;\widetilde{\mathcal M}_n)+\frac{c}{n}
\le
2R_n(\mu_n;m^\circ,S^\circ)+\frac{c}{n}.
\]
Since the model space associated with $(m^\circ,S^\circ)$ contains
$\sum_{j\in G^*}f^\circ_{j,m^\circ}$, the projection property gives
\[
\frac{1}{n}
\bigl\|
(\mathrm{I}_n-P_{(m^\circ,S^\circ)})\mu_n
\bigr\|_2^2
\le
\left\|
\sum_{j\in G^*}
\bigl(f_j^*-f^\circ_{j,m^\circ}\bigr)
\right\|_n^2.
\]
By the assumptions in Corollary 4.3, the right-hand side is bounded by
$
c_0
\sum_{j\in G^*}
\bigl\|
f_j^*-f^\circ_{j,m^\circ}
\bigr\|_n^2.
$

By the definition of $r_{m^\circ}(n)$ and the construction of
$S^\circ$, the componentwise approximation, estimation, and
within-family complexity costs are bounded by
$c g^*r_{m^\circ}(n)$, while the support and class codes contribute
$c g^*\log(ep/g^*)/n$ and $c\log M/n$, respectively. Since $\alpha=2$ and $\sigma=1$ in the present setting,
it follows that
\[
R_n(\mu_n;m^\circ,S^\circ)
\le
c_\lambda\left\{
\frac{g^*\log(ep/g^*)}{n}
+
g^*r_{m^\circ}(n)
+
\frac{\log M}{n}
\right\}.
\]

Combining the preceding inequalities and using
$r_{m^\circ}(n)=\min_{m\in\mathcal F}r_m(n)$ gives
\[
\mathbb E\bigl\{\mathrm{ASE}(\widehat m,\widehat S)\bigr\}
\le
c_\lambda\left\{
\frac{g^*\log(ep/g^*)}{n}
+
g^*\min_{m\in\mathcal F}r_m(n)
+
\frac{\log M}{n}
\right\}.
\]
The remainder term $c/n$ is absorbed into the first term because
$g^*\log(ep/g^*)\ge1$. This completes the proof.
\end{proof}

\subsection{Proof of Lemma~\ref{lem:mono-complexity}}
\begin{proof}
Given $0<\lambda_1<\lambda_2$, let $\widehat{S}_{\lambda_1} \in \arg \min\limits_{S \in \mathcal{M}_n} \mathrm{DCIC}_{\lambda_1}(S)$ and $\widehat{S}_{\lambda_2} \in \arg \min\limits _{S \in \mathcal{M}_n} \mathrm{DCIC}_{\lambda_2}(S)$. By optimality, we have

$$
\begin{aligned}
& \frac{\operatorname{RSS}\left(\widehat{S}_{\lambda_1}\right)}{\sigma^2}+\eta_n\left|\widehat{S}_{\lambda_1}\right|+\lambda_1 C_{\widehat{S}_{\lambda_1}} \leq \frac{\operatorname{RSS}\left(\widehat{S}_{\lambda_2}\right)}{\sigma^2}+\eta_n\left|\widehat{S}_{\lambda_2}\right|+\lambda_1 C_{\widehat{S}_{\lambda_2}}, \\
& \frac{\operatorname{RSS}\left(\widehat{S}_{\lambda_2}\right)}{\sigma^2}+\eta_n\left|\widehat{S}_{\lambda_2}\right|+\lambda_2 C_{\widehat{S}_{\lambda_2}} \leq \frac{\operatorname{RSS}\left(\widehat{S}_{\lambda_1}\right)}{\sigma^2}+\eta_n\left|\widehat{S}_{\lambda_1}\right|+\lambda_2 C_{\widehat{S}_{\lambda_1}}.
\end{aligned}
$$
Adding the two inequalities and using simple algebra, we have
$$
\left(\lambda_2-\lambda_1\right)\left(C_{\widehat{S}_{\lambda_2}}-C_{\widehat{S}_{\lambda_1}}\right) \leq 0.
$$
Since $\lambda_2>\lambda_1$, we conclude $C_{\widehat{S}_{\lambda_1}} \geq C_{\widehat{S}_{\lambda_2}}$.
\end{proof}

\subsection{Proof of Theorem~\ref{thm:coarse_poly_regime}}
\begin{proof}
Let \(B_{0}=\mathrm{DCIC}_{\lambda}(\varnothing)=\|y\|_2^2/\sigma^2\).
We first show that \(B_{0}=O_{\mathbb P}(n)\). Since \(y=\mu_n+\varepsilon\), we have
$
\|y\|_2^2
\le
2\|\mu_n\|_2^2+2\|\varepsilon\|_2^2.
$
By assumption, \(\|\mu_n\|_2^2\le c_1n\sigma^2\). Moreover, since the noise is sub-Gaussian, the variables \(\varepsilon_i^2\) are sub-exponential. Hence, there exist constants \(c_2, c_3>0\) such that
\[
\Pr\!\left(\sum_{i=1}^n \varepsilon_i^2 \le c_2 \sigma^2 n\right)\ge 1-\exp(-c_3n).
\]
Therefore, on an event of probability at least \(1-\exp(-c_3n)\),
\[
\|y\|_2^2
\le
2c_1n\sigma^2+2c_2\sigma^2 n
\le
c_4 n\sigma^2
\]
for some constant \(c_4>0\), and thus
$
B_{0}
=
\|y\|_2^2/\sigma^2
\le
c_4 n.
$

Now let \(\mathcal V_\lambda\) denote the collection of candidate models satisfying \eqref{eq:Ts-def}, so that \(N_\lambda=|\mathcal V_\lambda|\).
By construction of the search, every model \(S\in\mathcal V_\lambda\) satisfies
$
C_S
\le
B_{0}/\lambda.
$
Hence, on the event \(\{B_{0}\le c_4 n\}\),
\[
C_S\le \frac{c_4 n}{\lambda},
\qquad \forall\, S\in\mathcal V_\lambda.
\]
Therefore,
\[
e^{-C_S}\ge \exp\!\left(-\frac{c_4 n}{\lambda}\right),
\qquad \forall\, S\in\mathcal V_\lambda,
\]
and so
\[
N_\lambda \exp\!\left(-\frac{c_4 n}{\lambda}\right)
\le
\sum_{S\in\mathcal V_\lambda} e^{-C_S}.
\]
Using the Kraft inequality,
\[
\sum_{S\in\mathcal V_\lambda} e^{-C_S}
\le
\sum_{S\in\mathcal M_n} e^{-C_S}
\le 1.
\]
Thus,
\[
N_\lambda \le \exp\!\left(\frac{c_4 n}{\lambda}\right)
\]
with probability at least \(1-\exp(-c_3 n)\).
Finally, if
$
\lambda \ge c_5\,n/\log p,
$
 we have
\[
N_\lambda
\le
\exp\!\left(\frac{c_4}{c_5}\log p\right)
=
p^{\,c_4/c_5}.
\]
This proves the claim.
\end{proof}

\subsection{Proof of Theorem~\ref{thm:control}}
\begin{proof}
The constant $\lambda_\zeta>0$ will be chosen below depending on $\zeta$. The proof treats the three classes of competitors separately, using the decompositions \eqref{eq:sup}, \eqref{eq:sub}, and \eqref{eq:wrong}.

\paragraph*{Case 1: $S\in\mathcal M_{\mathrm{sup}}$}
Fix \(S\in\mathcal M_{\mathrm{sup}}\) with \(\Delta^+(S)=j\). For simplicity, denote $x_S = \log(L/\delta)+C_S-C_{S^*}$. Note that $\mathbb{E}\{\varepsilon^\top(P_S-P_{S^*})\varepsilon\}=\sigma^2(r_S-r_{S^*})$.
By Lemma \ref{lem:hanson}, there exists a constant $c'_\zeta>0$ such that
\begin{align*}
&\Pr\left(\frac{\varepsilon^\top (P_{S}-P_{S^*})\varepsilon}{\sigma^2}-(r_S-r_{S^*})
\geq c'_\zeta\bigl(\sqrt{(r_S-r_{S^*})x_S}+x_S\bigr)\right)\\
\leq{}\;&
2\exp\Bigl\{-\min\Bigl(\frac{(\sqrt{(r_S-r_{S^*})x_S}+x_S)^2}{r_S-r_{S^*}}, \sqrt{(r_S-r_{S^*})x_S}+x_S\Bigr)\Bigr\}\\
\leq{}\;&
2\exp\{-x_S\}.
\end{align*}
On the other hand, since $
r_S-r_{S^*} \leq \Delta^+(S)= j
$ and $x_S \geq C_S-C_{S^*}\geq jC_{\min}$, we have
\begin{align*}
    \sqrt{(r_S-r_{S^*})x_S}+(r_S-r_{S^*})+x_S \leq \frac{3}{2}\Bigl(\frac{1}{C_{\min}}+1\Bigr)x_S.
\end{align*}
Note that $C_{\min}\geq \log^*(1) \geq c_0 >0$. We choose $\lambda_\zeta$ so that
$
\lambda_\zeta \ge
\frac{3}{2}\left(\frac{1}{c_0}+1\right)\max\{c'_\zeta,1\}.
$ Consequently, we have
\begin{align*}
&\Pr\left(\frac{\varepsilon^\top (P_{S}-P_{S^*})\varepsilon}{\sigma^2}
\geq \lambda_{\zeta} x_S\right)\\
\leq{}\;&\Pr\left(\frac{\varepsilon^\top (P_{S}-P_{S^*})\varepsilon}{\sigma^2}
\geq c'_\zeta\bigl(\sqrt{(r_S-r_{S^*})x_S}+x_S\bigr)+(r_S-r_{S^*})\right)\\
\leq{}\;&
2\exp\{-(\log(L/\delta)+C_S-C_{S^*})\}.
\end{align*}
Then, combining the above inequality and \eqref{eq:sup}, with high probability, we have
\begin{align}
\begin{split}\label{eq:sup_upper}
\mathrm{DCIC}(S)-\mathrm{DCIC}(S^*)\geq{}\;& \eta_n j+(\lambda-\lambda_{\zeta})(C_S-C_{S^*})-\lambda_{\zeta}\log(L/\delta)\\
\geq{}\;& \eta_n j+(\lambda-\lambda_{\zeta})jC_{\min}-\lambda_{\zeta}\log(L/\delta).
\end{split}
\end{align}
Consequently, from \eqref{eq:sup_upper}, $\mathrm{DCIC}(S)-\mathrm{DCIC}(S^*)>0$ is guaranteed whenever
\[
j > \frac{\lambda_{\zeta}\log(L/\delta)}{\eta_n+(\lambda-\lambda_{\zeta}) C_{\min}}.
\]
By the union bound, we have
\begin{align*}
&\Pr\Bigl(\forall\, S\in \mathcal{M}_{\mathrm{sup}}\ \text{with}\ j\geq \frac{\lambda_{\zeta}\log(L/\delta)}{\eta_n+(\lambda-\lambda_{\zeta}) C_{\min}},\ \mathrm{DCIC}(S)-\mathrm{DCIC}(S^*)>0\Bigr)\\
\geq{}\;&
1-\sum_{j\geq \frac{\lambda_{\zeta}\log(L/\delta)}{\eta_n+(\lambda-\lambda_{\zeta}) C_{\min}}}\sum_{\substack{S\in \mathcal{M}_{\mathrm{sup}}\\\Delta^+(S) = j}}
\Pr\bigl(\mathrm{DCIC}(S)-\mathrm{DCIC}(S^*)\leq 0\bigr)\\
\geq{}\;&
1-\sum_{j\geq \frac{\lambda_{\zeta}\log(L/\delta)}{\eta_n+(\lambda-\lambda_{\zeta}) C_{\min}}}\sum_{\substack{S\in \mathcal{M}_{\mathrm{sup}}\\\Delta^+(S) = j}}
2\exp\{-(\log(L/\delta)+C_S-C_{S^*})\}\\
\geq{}\;&
1-2(\delta/L)\sum_{j\geq 1}\sum_{\substack{S\in \mathcal{M}_{\mathrm{sup}}\\\Delta^+(S) = j}}\exp\{-(C_S-C_{S^*})\}\\
\geq{}\;&
1-2e\,(\delta/L),
\end{align*}
where the last inequality follows from $$\sum_{j\geq 1}\sum_{\substack{S\in \mathcal{M}_{\mathrm{sup}}\\\Delta^+(S) = j}}\exp\{-(C_S-C_{S^*})\} \leq \exp \left\{\sum_{i \notin S^*} e^{-C_i}\right\}\leq e.$$

\paragraph*{Case 2: $S\in\mathcal M_{\mathrm{sub}}$}
Fix \(S\in\mathcal M_{\mathrm{sub}}\) with \(\Delta^-(S)=\ell\). Note that
$
P_{S^*}-P_{S}\succeq 0,
$
which implies $\varepsilon^\top(P_{S^*}-P_{S})\varepsilon \geq 0$.
By the choice of $\lambda_\zeta$, assume that
$c_{\zeta}\lambda_\zeta\ge 2$, where $c_{\zeta}>0$ is the
constant in Lemma \ref{lem:sum} with $\alpha=2$.
Choosing
$
t=\sqrt{\mathrm{APP}(S)}\times \sqrt{2\lambda_{\zeta}(\log(L/\delta) + \ell\log(es^*/\ell))}
$,
we have
\begin{align*}
&\Pr\left(\left|\frac{2\mu_n^\top(P_{S^*}-P_S)\varepsilon}{\sigma^2}\right|
\geq \frac{1}{2}\mathrm{APP}(S)+\lambda_{\zeta}\left(\log(L/\delta)+\ell\log(es^*/\ell)\right)\right)\\
\leq{}\;&
\Pr\left(\left|\frac{2\mu_n^\top(P_{S^*}-P_S)\varepsilon}{\sigma^2}\right|
\geq \sqrt{\mathrm{APP}(S)}\times\sqrt{2\lambda_{\zeta}(\log(L/\delta)+\ell\log(es^*/\ell))}\right)\\
\leq{}\;&
2\exp\left\{-2\left(\log(L/\delta)+\ell\log\frac{es^*}{\ell}\right)\right\}.
\end{align*}
Then, combining the above two inequalities and \eqref{eq:sub}, with high probability, we have
\begin{align}
\begin{split}\label{eq:sub_upper}
&\mathrm{DCIC}(S)-\mathrm{DCIC}(S^*)\\\geq{}\;& \frac{1}{2}\mathrm{APP}(S)-\eta_n \ell-\lambda(C_{S^*}-C_S)-\lambda_{\zeta}\log(L/\delta)-\lambda_{\zeta}\ell\log(es^*/\ell)\\
\geq{}\;& \frac{n}{2\sigma^2}\omega_{\lambda} \ell \beta_{\min}^2-\eta_n \ell-\lambda\ell C^*_{\max}-\lambda_{\zeta}\log(L/\delta)-\lambda_{\zeta}\ell\log s^*,
\end{split}
\end{align}
where the second inequality follows from $\mathcal{M}_{\rm sub} \subseteq \mathcal{L}_{\lambda}$.

Consequently, from \eqref{eq:sub_upper}, $\mathrm{DCIC}(S)-\mathrm{DCIC}(S^*)>0$ is guaranteed whenever
\begin{equation}\label{eq:subl}
\left(\frac{n}{2\sigma^2}\omega_{\lambda} \beta_{\min}^2-\eta_n-\lambda C^*_{\max}-\lambda_{\zeta}\log s^*\right)\ell > \lambda_{\zeta}\log(L/\delta)
\end{equation}
holds.
By the union bound, we have
\begin{align*}
&\Pr\Bigl(\forall\, S\in \mathcal{M}_{\mathrm{sub}}\ \text{with}\ \ell\ \text{satisfying}~\eqref{eq:subl},\ \mathrm{DCIC}(S)-\mathrm{DCIC}(S^*)>0\Bigr)\\
\geq{}\;&
1-\sum_{\ell\ \text{satisfies} \eqref{eq:subl}}\sum_{\substack{S\in \mathcal{M}_{\mathrm{sub}}\\\Delta^-(S)=\ell}}
\Pr\bigl(\mathrm{DCIC}(S)-\mathrm{DCIC}(S^*)\leq 0\bigr)\\
\geq{}\;&
1-\sum_{\ell\ \text{satisfies}~\eqref{eq:subl}}\sum_{\substack{S\in \mathcal{M}_{\mathrm{sub}}\\\Delta^-(S)=\ell}}
2\exp\Bigl\{-2\Bigl(\log(L/\delta)+\ell\log (es^*/\ell)\Bigr)\Bigr\}\\
\geq{}\;&
1-2(\delta/L)^{2}\sum_{\ell=1}^{s^*}\binom{s^*}{\ell}\exp\left\{-2\ell\log (es^*/\ell)\right\}\\
\geq{}\;&
1-2e\,(\delta/L)^{2}\sum_{\ell=1}^{s^*}\exp\left\{-\ell\log (es^*/\ell)\right\}\\
\geq{}\;&
1-2e\delta/L.
\end{align*}

\paragraph*{Case 3: $S\in\mathcal M_{\mathrm{w}}$}
Fix \(S\in\mathcal M_w\) with \(\Delta^+(S)=j\) and \(\Delta^-(S)=\ell\). For simplicity, denote $x'_S = \log(L/\delta)+\ell\log(es^*/\ell)+C_{S\backslash S^*}$. Note that
\[
\|P_{S^*}-\tilde P_{S}\|_F^2 \leq \ell
\qquad\text{and}\qquad
\|P_{S}-\tilde P_{S}\|_F^2 \leq j.
\]
Since
$
P_{S^*}-\tilde P_{S}\succeq 0,
$
we have $\varepsilon^\top(P_{S^*}-\tilde P_{S})\varepsilon \geq 0$.
Arguing as in {\bf Case 1} and {\bf Case 2}, and using the choice of \(\lambda_\zeta\), we have
\begin{align*}
&\Pr\left(\frac{\varepsilon^\top (P_{S}-\tilde P_{S})\varepsilon}{\sigma^2}
\geq \frac{\lambda_{\zeta}}{2} x'_S\right)\\
\leq{}\;&
\Pr\left(\frac{\varepsilon^\top (P_{S}-\tilde P_{S})\varepsilon}{\sigma^2}
\geq  c'_\zeta\bigl(\sqrt{(r_S-\tilde r_{S})x'_S}+x'_S\bigr)+(r_S-\tilde r_{S})\right)\\
\leq{}\;&
2\exp\{-2x'_S\},
\end{align*}
and
\begin{align*}
&\Pr\left(\left|\frac{2\mu_n^\top(P_{S^*}-P_S)\varepsilon}{\sigma^2}\right|
\geq \frac{1}{2}\left(\mathrm{APP}(S)+\lambda_{\zeta} x'_S\right)\right)\\
\leq{}\;&
\Pr\left(\left|\frac{2\mu_n^\top(P_{S^*}-P_S)\varepsilon}{\sigma^2}\right|
\geq \frac{1}{2}\sqrt{\mathrm{APP}(S)}\times\sqrt{\lambda_{\zeta} x'_S}\right)\\
\leq{}\;&
2\exp\left\{-2x'_S\right\}.
\end{align*}
Combining the above inequalities with~\eqref{eq:wrong} shows that, with high
probability,
\begin{align*}
&\mathrm{DCIC}(S)-\mathrm{DCIC}(S^*)\\
    = {}\;&\frac{\varepsilon^\top (P_{S^*}-\widetilde P_S)\varepsilon}{\sigma^2}
    -\frac{\varepsilon^\top ( P_{S}-\widetilde P_S)\varepsilon}{\sigma^2}
    +\frac{2\mu_n^\top(P_{S^*}-P_S)\varepsilon}{\sigma^2}
    +\mathrm{APP}(S)+u_{S}\\
\geq{}\;&\frac{1}{2}\mathrm{APP}(S)+ \eta_n j-\eta_n\ell+\lambda(C_S-C_{S^*})-\lambda_{\zeta}\log(L/\delta)-\lambda_{\zeta} C_{S\backslash S^*}-\lambda_{\zeta}\ell\log(es^*/\ell).
\end{align*}

If $S\notin \mathcal{L}_{\lambda}$, we have $\mathrm{DCIC}(S)-\mathrm{DCIC}(S^*) >0$.
On the other hand, if $S\in \mathcal{L}_{\lambda}$, we have
\begin{align*}
&\mathrm{DCIC}(S)-\mathrm{DCIC}(S^*)\\
\geq{}\;& \left(\frac{n}{2\sigma^2}\omega_{\lambda}\beta^2_{\min} -\eta_n-\lambda C^*_{\max}-\lambda_{\zeta}\log s^*\right)\ell+(\eta_n+(\lambda-\lambda_{\zeta})C_{\min})j-\lambda_{\zeta}\log(L/\delta).
\end{align*}
Consequently, $\mathrm{DCIC}(S)-\mathrm{DCIC}(S^*)>0$ is guaranteed whenever
\begin{align}\label{eq:constraint}
\Big(\eta_n+(\lambda-\lambda_{\zeta}) C_{\min}\Big)j+\left(\frac{n}{2\sigma^2}\omega_{\lambda}\beta^2_{\min}-\eta_n-\lambda C^*_{\max}-\lambda_{\zeta}\log s^*\right)\ell> \lambda_{\zeta}\log (L/\delta).
\end{align}
Then, by the union bound, we have
\begin{align*}
&\Pr\Bigl(\forall\, S\in \mathcal{M}_{\mathrm{w}}\ \text{satisfying}\ \eqref{eq:constraint},\ \mathrm{DCIC}(S)-\mathrm{DCIC}(S^*)>0\Bigr)\\
\geq{}\;&
1-\sum_{\substack{S\in \mathcal{M}_{\mathrm{w}}\\ \Delta^+(S),\,\Delta^-(S)\ \text{satisfy}\ \eqref{eq:constraint}}}
\Pr\bigl(\mathrm{DCIC}(S)-\mathrm{DCIC}(S^*)\leq 0\bigr)\\
\geq{}\;&
1-\sum_{\ell=1}^{s^*}\sum_{\substack{S\in \mathcal{M}_{\mathrm{w}}\\ \Delta^-(S)=\ell}}
4\exp\left\{-2\left(\log(L/\delta)+\ell\log (es^*/\ell)+C_{S\backslash S^*}\right)\right\}\\
\geq{}\;&
1-4e\,(\delta/L)^{2}\sum_{\ell=1}^{s^*}\exp\left\{-\ell\log (es^*/\ell)\right\}
\sum_{\substack{S\in \mathcal{M}_{\mathrm{w}}\\ \Delta^-(S)=\ell}}\exp\{-2C_{S\backslash S^*}\}\\
\geq{}\;&
1-ce(\delta/L),
\end{align*}
where the last inequality follows from the Kraft condition.

Combining the above bounds for each fixed $\lambda\in\Lambda$ and applying a union bound over $\Lambda$, we conclude that, with probability at least $1-c\delta$, the stated inequality holds simultaneously for all $\lambda\in\Lambda$.
\end{proof}

\section{Auxiliary Lemmas}\label{app:auxiliary}
Recall that
\[
\|Z\|_{\psi_\alpha}
:=
\inf\left\{K>0:
\mathbb E\exp\left(\frac{|Z|^\alpha}{K^\alpha}\right)\le 2
\right\}.
\]
Under the tail assumption in Section~\ref{sub:framework}, there exists a constant
$C_{\alpha,\zeta}>0$ such that
\[
\|\varepsilon_i\|_{\psi_\alpha}
\le C_{\alpha,\zeta}\sigma,
\qquad i=1,\ldots,n.
\]
Indeed, with
$
K_{\alpha,\zeta}
=
\max\left\{\zeta,\frac{2}{(\log 2)^{1/\alpha}}\right\},
$
we have
\[
\Pr\left(\frac{|\varepsilon_i|}{\sigma}>t\right)
\le
2\exp\left\{-\left(\frac{t}{K_{\alpha,\zeta}}\right)^\alpha\right\},
\qquad t\ge0,
\]
and the claim follows from the standard tail--Orlicz equivalence.

\begin{lemma}[\citet{ejp2021,sambale2023some}]\label{lem:hanson}
Let $\varepsilon=(\varepsilon_1,\ldots,\varepsilon_n)^\top$ be independent, mean-zero
sub-Weibull $(\alpha,\zeta)$ random variables for all $i$ and some $\alpha\in(0,2]$.
Let $A=(a_{ij})\in\mathbb{R}^{n\times n}$ be a symmetric matrix. Then, for every $t\ge 0$,
\begin{equation*}
\Pr\!\Big(
\big|\varepsilon^\top A\varepsilon-\mathbb{E}(\varepsilon^\top A\varepsilon)\big|
\ge t
\Big)
\ \le\
2\exp\!\left\{
-c_{\alpha,\zeta}
\min\!\left(
\frac{t^2}{\sigma^4\|A\|_{\mathrm{F}}^2},\
\Big(\frac{t}{\sigma^2\|A\|_{\mathrm{op}}}\Big)^{\alpha/2}
\right)
\right\},
\end{equation*}
where $\|A\|_F$ is the Frobenius norm and
$\|A\|_{\mathrm{op}}$ is the operator norm.
\end{lemma}

\begin{lemma}\label{lem:sum}
Let $\varepsilon=(\varepsilon_1,\ldots,\varepsilon_n)^\top$ be independent, mean-zero
sub-Weibull $(\alpha,\zeta)$ random variables for all $i$ and some $\alpha\in(0,2]$.
Then, for every deterministic vector $a \in \mathbb{R}^n$ and $t>0$, we have
\begin{equation*}
\Pr\!\Big(
\big|a^\top\varepsilon\big|
\ge t
\Big)
\ \le\
2\exp\!\left\{
-c_{\alpha,\zeta}
\min\!\left(
\left(\frac{t}{\sigma\|a\|_2}\right)^2,\,
\left(\frac{t}{\sigma\|a\|_2}\right)^\alpha
\right)
\right\}.
\end{equation*}
\end{lemma}

\begin{proof}
By the Bernstein-type inequality for linear forms of independent sub-Weibull random variables,

$$
\operatorname{Pr}\left(\left|a^{\top} \varepsilon\right| \geq t\right) \leq 2 \exp \left\{-c_{\alpha, \zeta} \min \left[\frac{t^2}{\sigma^2\|a\|_2^2},\left(\frac{t}{\sigma\|a\|_{\infty}}\right)^\alpha\right]\right\} .
$$

Since $\|a\|_{\infty} \leq\|a\|_2$, the asserted inequality follows. 
\end{proof}

\section{Additional Simulation Results}\label{app:simulations}
\subsection{Additional results for Section 6.1}\label{app:highly}

Across all simulation settings in Section 6.1 and covariance mixes ($w \in \{0.2, 0.5, 0.8\}$), two consistent patterns emerge. First, as the sample size $n$ increases, all methods improve in both prediction error (ASE) and support recovery metrics (FP, FN, MCC). Second, performance improves across all metrics as $w$ decreases when the design shifts from block-dominated to AR(1)-dominated correlation. However, the rates of improvement and the performance gaps between methods hinge on the degree of signal clustering and the strength of within-block collinearity.

\begin{figure}[H]
  \centering
  \includegraphics[scale = 0.57]{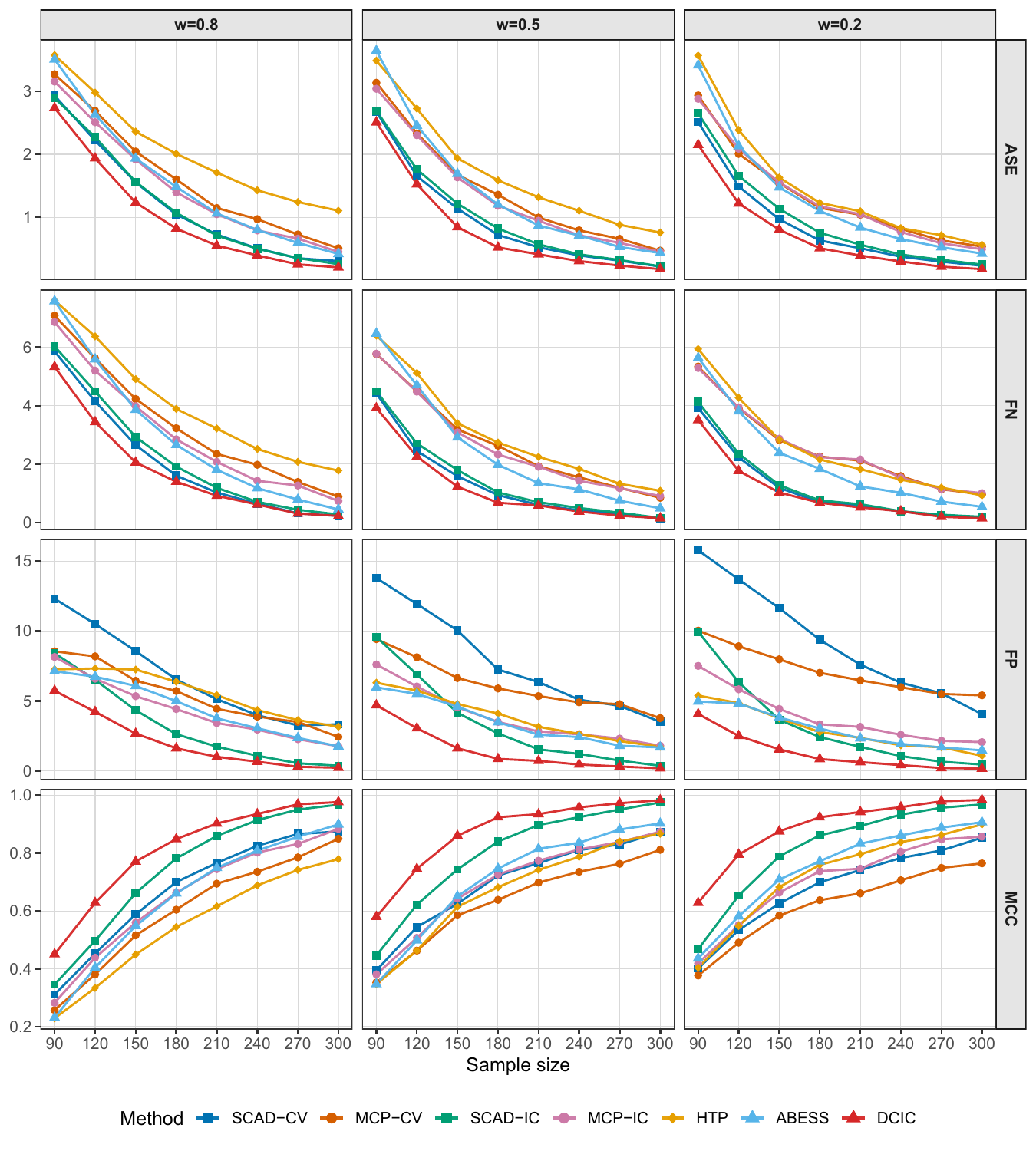}
  \caption{\label{fig:fig1}   Performance metrics under Setting 1 with an increasing sample size.
  The suffix ``-CV'' indicates tuning by 10-fold cross-validation. The suffix ``-IC'' indicates tuning by information criterion.
  }
\end{figure}

In the one-per-block signals setting (Figure~\ref{fig:fig1}), signals are uniformly dispersed across blocks, avoiding within-block competition. When the sample size reaches $n \geq 240$, nearly exact support recovery is attained for DCIC and SCAD-IC, which outperform the other methods significantly. In terms of estimation accuracy, DCIC, SCAD-CV and SCAD-IC deliver substantially lower ASE than the remaining competitors.
DCIC simultaneously keeps both false positives and false negatives at low levels, which is consistent with our theoretical guarantees. By contrast, the best-subset methods (HTP and ABESS) exhibit elevated false negatives, indicating that in the presence of strongly correlated competitor variables they frequently miss true predictors, which also explains their higher ASE.

\begin{figure}[htbp]
  \centering
  \includegraphics[scale = 0.57]{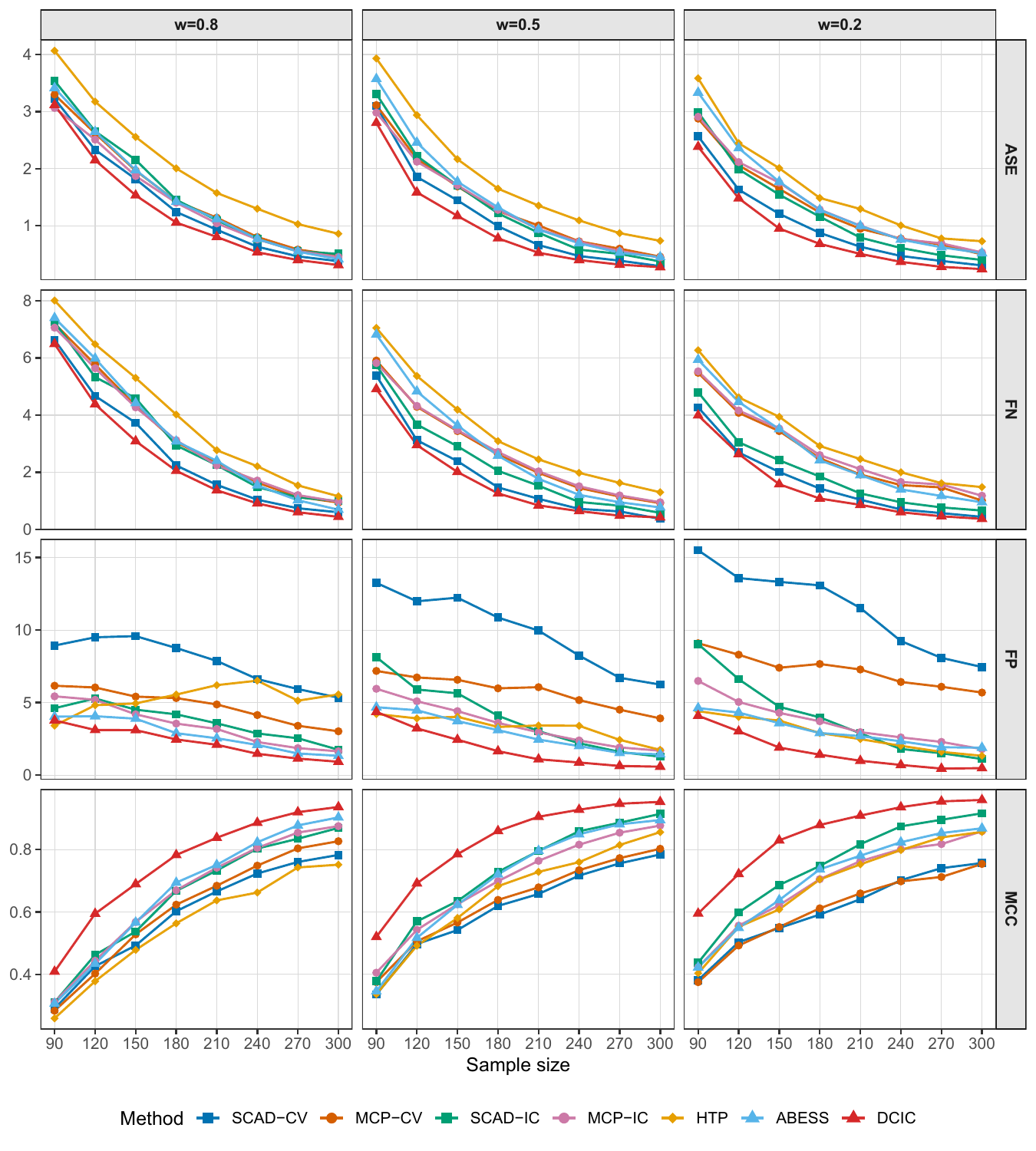}
  \caption{\label{fig:fig2}   Performance metrics under Setting 2 with an increasing sample size.
  The suffix ``-CV'' indicates tuning by 10-fold cross-validation. The suffix ``-IC'' indicates tuning by information criterion.
  }
\end{figure}
In the moderately clustered signals setting (Figure~\ref{fig:fig2}), signals are concentrated in half of the blocks, intensifying within-block competition. DCIC maintains a clear advantage in MCC owing to its good control of both false positives and false negatives. While HTP and ABESS limit false positives, they omit a substantial number of true predictors, returning an overly sparse model. By contrast, SCAD and MCP, especially tuned by cross-validation, exhibit inflated false positives at small $n$, reflecting difficulty in separating signals from strongly correlated surrogates. The ASE of DCIC matches or improves upon that of the best competitors, indicating competitive estimation performance.

In summary, the simulation results illustrate DCIC as a robust method for model selection under strongly correlated designs. Although best-subset methods such as HTP and ABESS have strong theoretical guarantees and often perform competitively in practice \citep{huang2018constructive,Zhu202014241,zhang2025}, under strong correlation, they tend to select overly sparse models, resulting in higher false negatives and weaker support recovery.

\subsection{Illustration of the pathwise computation--statistics trade-off}\label{app:pathwise-sim}
\label{app:path-illustration}
We provide an additional illustration of the pathwise computation--statistics
trade-off underlying the complexity-guided search in Section~\ref{sub:tradeoff}.

The data are generated from the same covariance construction as Setting~1 in
Section~\ref{sub:sim_highly}. Specifically, we take \(n=50\), \(p=300\), and
\(s^*=5\), and use the correlated design with
\(\rho_{\mathrm{in}}=\rho_{\mathrm{out}}=0.9\) and \(w=0.5\). To control the quality of
the complexity ranking, we construct the ranking so that the active variables
are randomly embedded in the first \(\tau_n\) ranked positions, and consider
$
  \tau_n \in \{20,50,100\}.
$
These three choices represent near-oracle, favorable, and uninformative
rankings, respectively.

For each ranking regime, we run the DCIC path algorithm over the same geometric
grid
$
  \lambda_k=\lambda_{\max}\rho^{k-1}, k=1,\ldots,30,
$
with \(\lambda_{\max}=30\) and \(\rho=0.9\). At each path iteration \(k\), we
record the number of models evaluated by the implemented DFS
search, the average squared error (ASE), and the Matthews correlation
coefficient (MCC). The experiment is
repeated 100 times, and Figure~\ref{fig:log} reports the average curves over
these replications.

\begin{figure}[htbp]
  \centering
  \includegraphics[scale = 0.5]{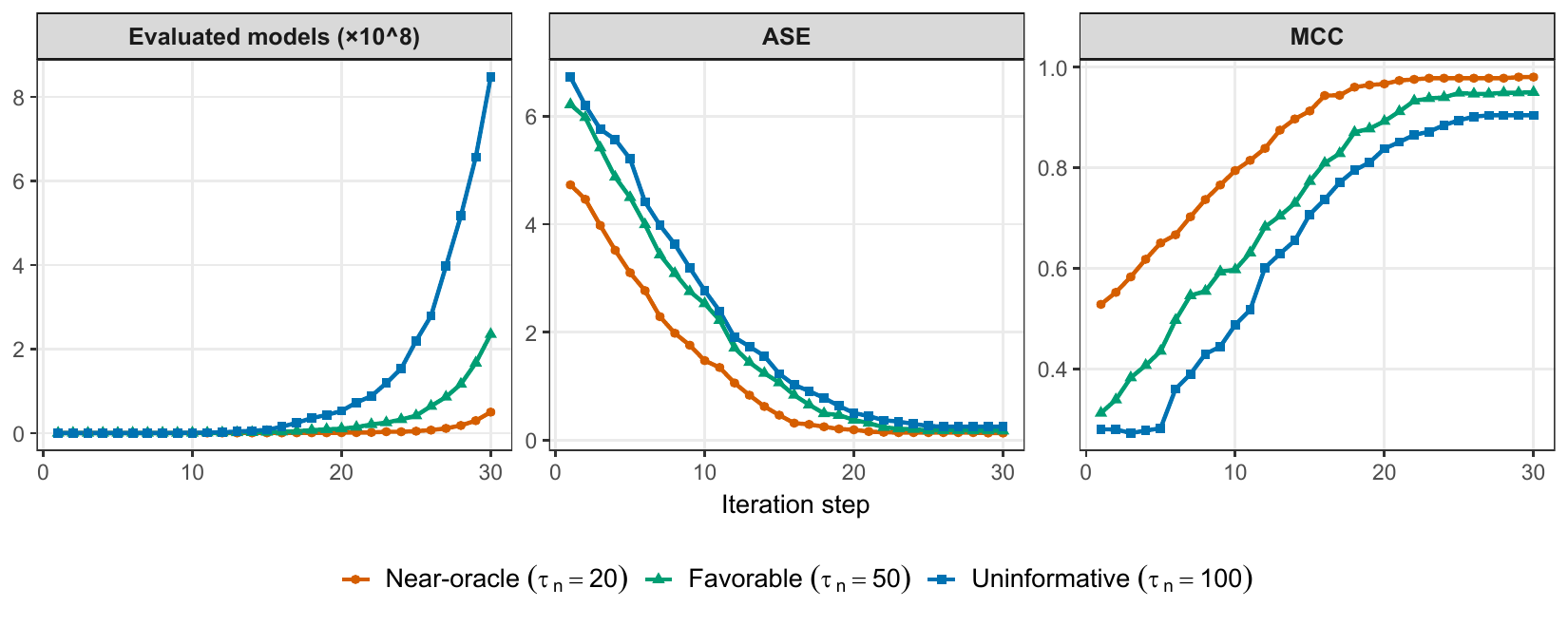}
  \caption{\label{fig:log}
  Pathwise computation--statistics trade-off under different ranking
  informativeness. The curves are averaged over 100 replications.
  }
\end{figure}

Figure~\ref{fig:log} displays a clear pathwise pattern. When the ranking is
strongly informative, the procedure enters a statistically favorable region at
early iterations: ASE decreases rapidly and MCC becomes high while the evaluated
search size remains small. As the ranking becomes less informative, comparable
statistical accuracy is achieved only at later iterations, where the search
region has expanded substantially. In particular, the uninformative regime
requires a much larger evaluated search size before reaching a similar level of
ASE and MCC.

This behavior is consistent with the theoretical picture in
Section~\ref{sub:tradeoff}. A more informative complexity ranking reduces the
rank envelope \(\tau_n\), thereby shifting the statistical entry index \(k_{\rm stat}\)
to earlier grid points, where \(\lambda\) is larger and the search is more
strongly constrained. Conversely, weaker ranking information delays recovery
until smaller values of \(\lambda\), at which point the number of evaluated
models grows rapidly. The experiment therefore illustrates how the complexity
ranking and the \(\lambda\)-path jointly determine the effective
computation--statistics trade-off.

We further examine data-driven constructions of the preliminary ranking under the
same simulation setting. In addition to the oracle-prefix reference rankings with
$\tau_n=20$ and $\tau_n=50$, we consider two practical ranking procedures: a
LASSO-based ranking and a marginal-correlation ranking. The sample size $n$ is
varied from 40 to 200 in increments of 20. Each configuration is
repeated 100 times, and the results are summarized in Figure~\ref{fig:lasso}.

\begin{figure}[htbp]
\centering
\includegraphics[scale=0.52]{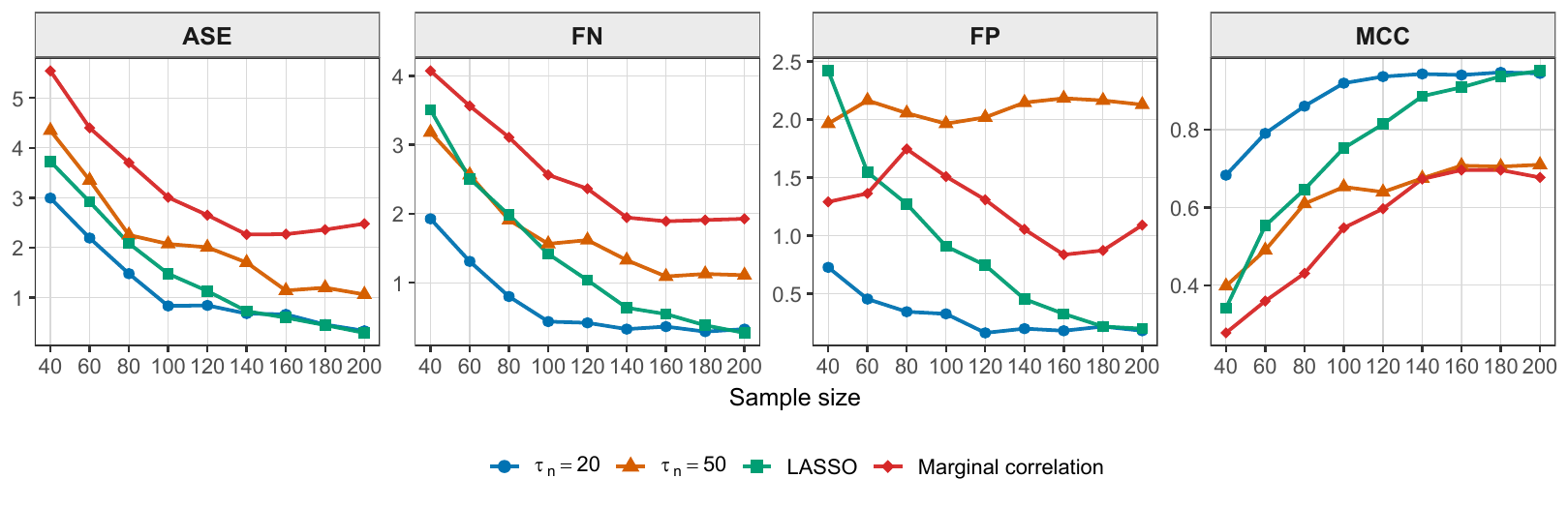}
\caption{\label{fig:lasso}
Performance of different ranking procedures under a strongly correlated design.
}
\end{figure}

The LASSO-based ranking provides a useful data-driven approximation to the
favorable oracle-prefix rankings. For small sample sizes, its performance is
comparable to that of the $\tau_n=50$ reference ranking. As $n$ increases, the
LASSO ranking becomes more informative and moves closer to the performance of the
stronger $\tau_n=20$ reference ranking. By contrast, marginal correlation performs
substantially worse in this strongly correlated setting, reflecting its sensitivity
to correlated proxy variables. These results support the use of LASSO as a
practical ranking device in the complexity-guided DCIC implementation.

To assess whether the preceding comparison is specific to the strongly correlated
design, we repeat the same ranking experiment under a moderately correlated AR(1)
design with covariance $\Sigma_{ij}=0.6^{|i-j|}$, keeping the remaining simulation
parameters unchanged. The results are summarized in Figure~\ref{fig:lasso_AR}.
\begin{figure}[htbp]
\centering
\includegraphics[scale=0.52]{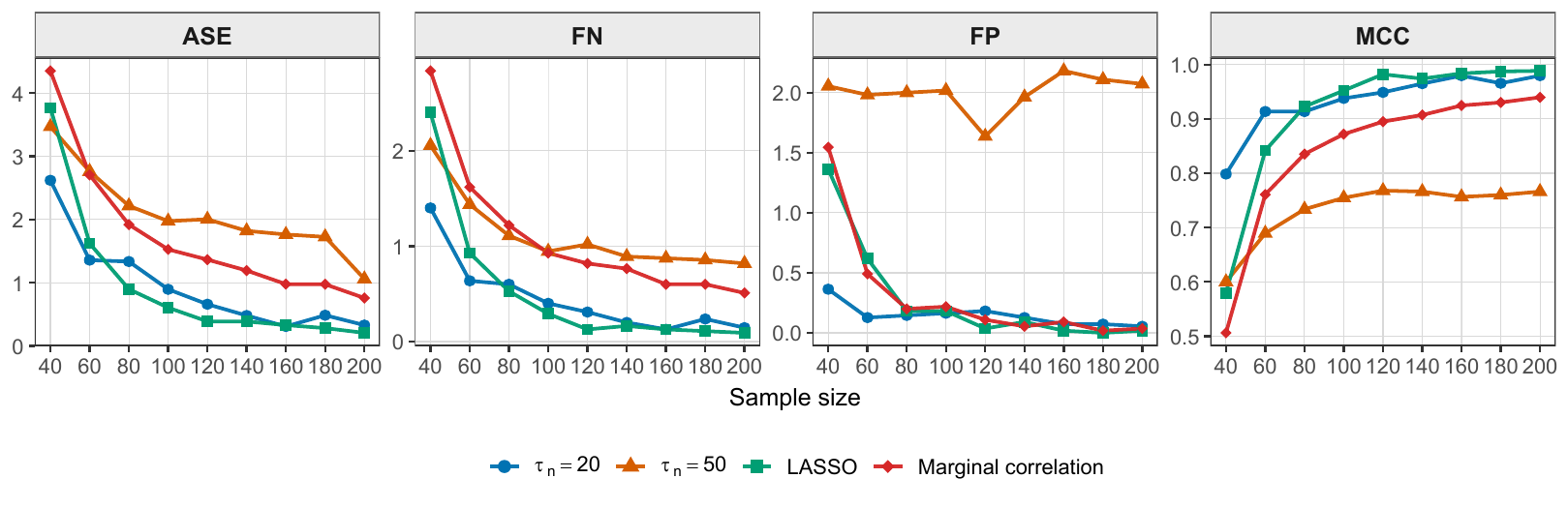}
\caption{\label{fig:lasso_AR}
Performance of different ranking procedures under a moderately correlated AR(1) design.
}
\end{figure}

Under the moderate AR(1) design, the LASSO-based ranking remains
competitive with the favorable $\tau_n=20$ reference ranking across all four metrics.
Marginal correlation also performs reasonably well, in contrast to its behavior under
the strongly correlated design in Figure~\ref{fig:lasso}. Taken together, Figures~\ref{fig:lasso} and~\ref{fig:lasso_AR}
suggest that LASSO provides a stable practical ranking device across both difficult and
moderately correlated regimes, while marginal correlation can be adequate in easier
settings but may deteriorate under strong correlation.

As an additional large-scale runtime illustration, we consider a problem with $n=200$, $p=5000$, and $s^*=20$, and run the proposed complexity-guided search under different rank envelopes $\tau_n$. We set $\sigma=1$, and the nonzero entries of $\beta^*$ are drawn independently from $\operatorname{Unif}[1,2]$.

Unlike Figure~\ref{fig:log}, where the horizontal axis is the path iteration, Figure~\ref{fig:time} reports ASE and MCC against accumulated computational time. Panel~(a) uses the design in Setting~1 of Section~\ref{sub:sim_highly}, with $\rho_{\mathrm{in}}=\rho_{\mathrm{out}}=0.9$ and $w=0.5$. Panel~(b) uses an AR(1) design with covariance $\Sigma_{ij}=0.6^{|i-j|}$. The same qualitative pattern remains visible on the runtime scale: a smaller rank envelope reaches low ASE and high MCC much earlier, whereas a larger rank envelope requires substantially longer runtime and may remain away from the favorable statistical region under the same computational budget. This experiment is intended as a representative runtime illustration of the algorithmic role of $\tau_n$ in the computation--statistics trade-off.

\begin{figure}[H]
\centering
\includegraphics[scale=0.52]{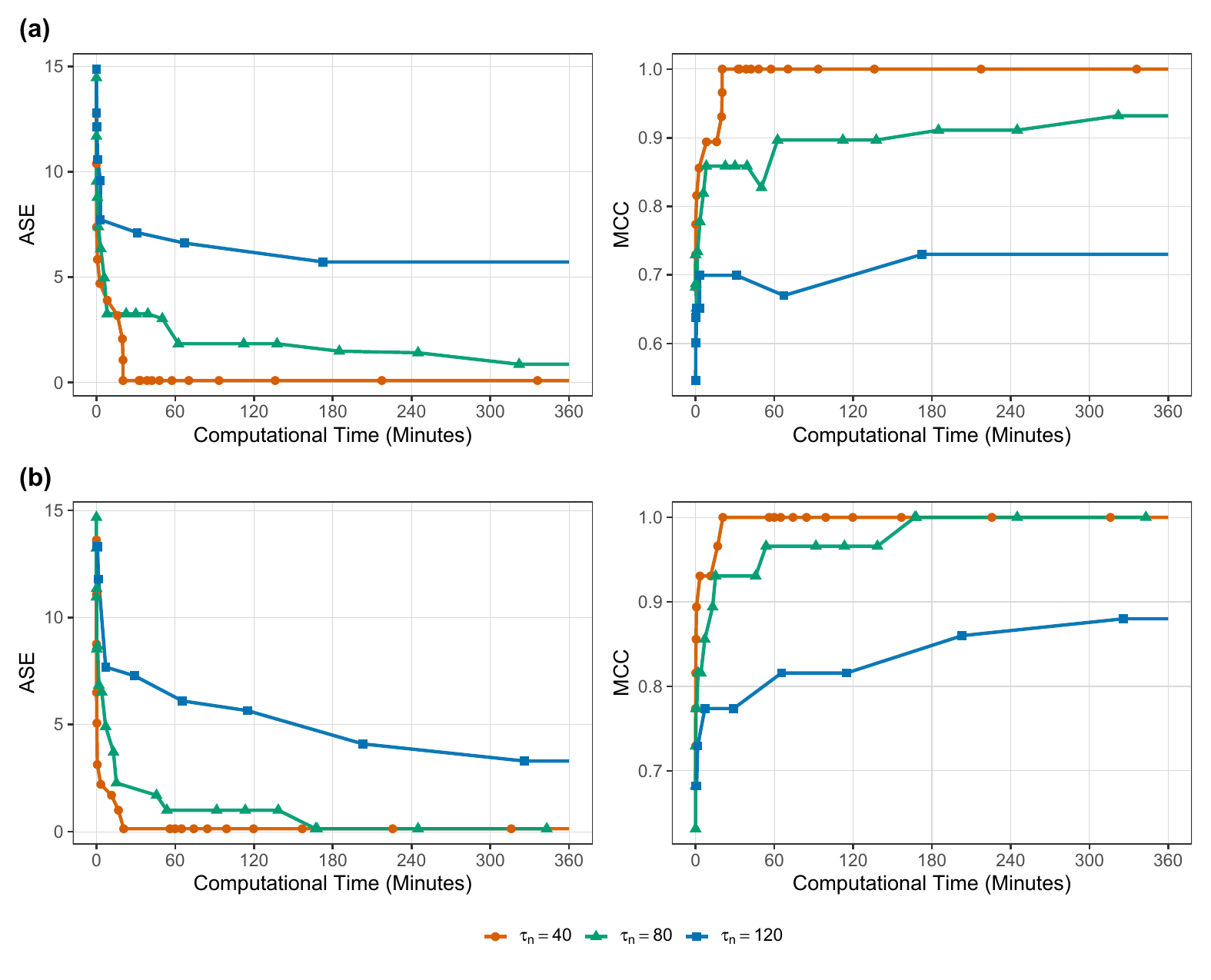}
\caption{\label{fig:time}
Representative large-scale runtime illustration of the pathwise computation--statistics trade-off, with $n=200$, $p=5000$, and $s^*=20$. The curves report ASE and MCC against accumulated computational time under different rank envelopes $\tau_n$. Panel~(a) corresponds to the strongly correlated design in Setting~1 of Section~\ref{sub:sim_highly}, with $\rho_{\mathrm{in}}=\rho_{\mathrm{out}}=0.9$ and $w=0.5$; panel~(b) corresponds to an AR(1) design with $\Sigma_{ij}=0.6^{|i-j|}$.
}
\end{figure}

\subsection{Sensitivity analysis for implementation heuristics}\label{sub:sensitivity}
We examine the sensitivity of the DCIC search to two heuristic
parameters used to accelerate the depth-first search described in
Remark~\ref{rem:accelerate}: the ranked-search prefix size and the RSS-based
budget-tightening factor. The experiments in this subsection use the same block-design
mechanism and active-block structure as Setting~2 in Section~\ref{sub:sim_highly}.
The active variables are distributed over five correlated blocks, with two nonzero
coefficients in each active block. We fix \(w=0.5\) and consider two correlation regimes by setting
\(\rho_{\mathrm{in}}=\rho_{\mathrm{out}}\) to \(0.5\) and \(0.8\), corresponding to
moderate and strong correlation, respectively.

Each configuration is repeated 50 times. We impose a 10-minute time limit on each
replication. If the search does not terminate within the time limit, the last available
solution along the DCIC path is used as the reported estimator. We report both the
statistical performance of the resulting estimators and the runtime.

\subsubsection{RSS-based budget tightening}
We first vary the RSS-based budget-tightening factor. This factor controls how
aggressively the RSS term is used to tighten the complexity budget during the
implemented search. Smaller values retain more candidate branches and are therefore
more conservative, whereas larger values prune more aggressively. We set \(n=100\),
\(p=400\), and \(\sigma=1\), and draw the nonzero coefficients from
\(\mathrm{Unif}[0.5,1]\). The ranked-search prefix size is fixed at its default value
\(p/5\). We compare the budget-tightening factors
$
\{0.6,\ 0.8,\ 1.0\}.
$
\begin{table}[htbp]
\centering
\caption{Sensitivity to the RSS-based budget-tightening factor.}
\label{tab:scale-sensitivity}
\begin{tabular}{llrrrrrr}
\toprule
Regime & Budget factor & ASE & MCC & TP & FN & Truncated & Runtime (s) \\
\midrule
Moderate & \(0.6\) & 0.730 & 0.863 & 7.98 & 2.02 & 12 & 173.52 \\
Moderate & \(0.8\) & 0.738 & 0.863 & 7.98 & 2.02 & 0  & 15.85  \\
Moderate & \(1.0\) & 0.779 & 0.852 & 7.86 & 2.14 & 0  & 0.82   \\
\midrule
Strong & \(0.6\) & 1.171 & 0.647 & 5.35 & 4.65 & 23 & 283.12 \\
Strong & \(0.8\) & 1.092 & 0.676 & 5.76 & 4.24 & 1  & 35.61  \\
Strong & \(1.0\) & 1.201 & 0.647 & 5.34 & 4.66 & 0  & 1.51   \\
\bottomrule
\end{tabular}
\end{table}

Table~\ref{tab:scale-sensitivity} exhibits the expected computation--statistics
trade-off. The conservative factor \(0.6\) retains more branches and achieves
competitive recovery accuracy, but it leads to many truncated replications and much
larger runtime. The aggressive factor \(1.0\) is substantially faster, but it slightly
degrades recovery accuracy, especially in the strong regime. The intermediate factor
\(0.8\) retains nearly the same statistical performance as the conservative choice in
the moderate regime and gives the best recovery performance in the strong regime, while
requiring much less computation than \(0.6\). These results support the default
RSS-based budget-tightening factor \(0.8\) used in the main experiments.

\subsubsection{Ranked-search prefix size}
We next vary the ranked-search prefix size, which controls how many top-ranked
variables are retained in the depth-first search. We set \(n=100\), \(p=300\), and
\(\sigma=2\), and draw the nonzero coefficients from \(\mathrm{Unif}[1,2]\). The
RSS-based budget-tightening factor is fixed at \(0.8\). We compare prefix sizes
$
\{p,\ p/2,\ p/5,\ p/10\}.
$

\begin{table}[H]
\centering
\caption{Sensitivity to the ranked-search prefix size.}
\label{tab:prefix-sensitivity}
\begin{tabular}{llrrrrrr}
\toprule
Regime & Prefix size & ASE & MCC & TP & FN & Truncated & Runtime (s) \\
\midrule
Moderate & \(p\)      & 2.952 & 0.870 & 8.03 & 1.97 & 0 & 22.02 \\
Moderate & \(p/2\)    & 2.952 & 0.870 & 8.03 & 1.97 & 0 & 13.85 \\
Moderate & \(p/5\)    & 2.952 & 0.870 & 8.03 & 1.97 & 0 & 2.97  \\
Moderate & \(p/10\)   & 2.966 & 0.868 & 7.98 & 2.02 & 0 & 0.17  \\
\midrule
Strong & \(p\)      & 5.020 & 0.675 & 5.36 & 4.64 & 3 & 177.45 \\
Strong & \(p/2\)    & 5.020 & 0.675 & 5.36 & 4.64 & 3 & 101.44 \\
Strong & \(p/5\)    & 5.027 & 0.675 & 5.36 & 4.64 & 1 & 38.17 \\
Strong & \(p/10\)   & 5.027 & 0.674 & 5.36 & 4.64 & 0 & 2.22   \\
\bottomrule
\end{tabular}
\end{table}

Table~\ref{tab:prefix-sensitivity} shows that reducing the prefix size substantially
decreases runtime while leaving the statistical performance competitive.
The default choice \(p/5\) provides a useful compromise: it is much faster than using
the full ranked list or \(p/2\), while remaining more conservative than the very small
prefix \(p/10\).

% \begin{flushleft}
% \footnotesize
% Note: ``Truncated'' reports the number of replications not completed within the
% 10-minute cap. For truncated replications, performance metrics are computed from the
% latest checkpoint, while runtime is counted as 600 seconds when computing the capped
% mean runtime.
% \end{flushleft}

\subsection{Additional setting of all-subset selection}
In this section, we report additional simulation results under a standard AR(1) design with moderate correlation \citep{bertsimas2016,Zhu202014241}. Specifically, the rows of the design matrix $X$ are generated independently from $N(0,\Sigma)$ with $\Sigma_{ij}=0.6^{|i-j|}$. We fix $p=1000$, $s^*=10$, and $\sigma=2$, and let the sample size $n$ vary from 90 to 360 in increments of 30. The nonzero coefficients are set to $1$ with random signs, and their indices are equally spaced. In addition to the performance metrics reported in Section~\ref{sub:sim_highly}, we also record the runtime in seconds for each method.

\begin{figure}[H] \centering \includegraphics[scale = 0.6]{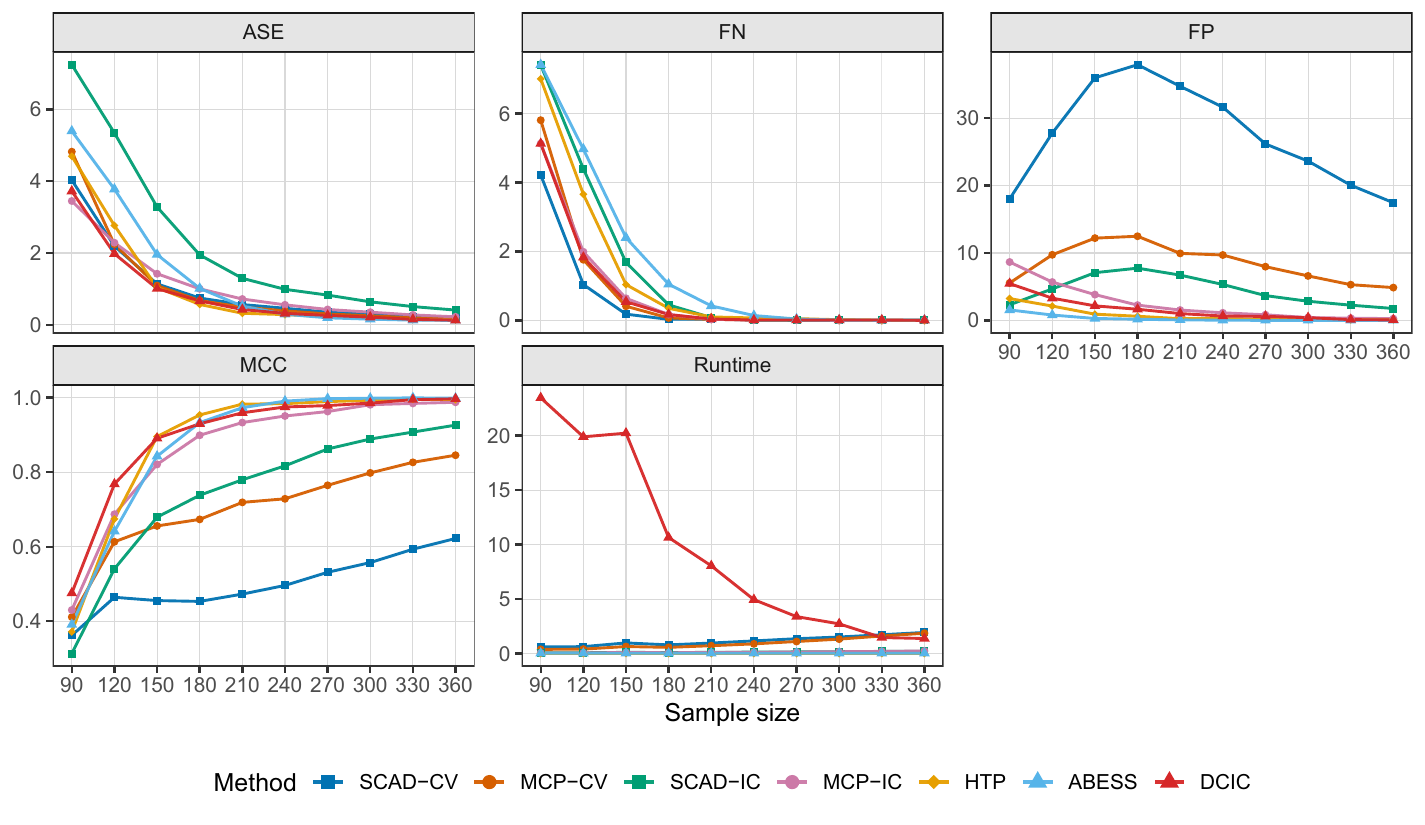} \caption{\label{fig:figAR}
Performance under the AR(1) design with $\Sigma_{ij}=0.6^{|i-j|}$ as the sample size increases.
The suffix ``-CV'' denotes tuning by 10-fold cross-validation, whereas ``-IC'' denotes tuning by an information criterion.
}
\end{figure}

Figure~\ref{fig:figAR} shows that all methods improve as the sample size increases, although the extent of improvement differs substantially across procedures. In terms of support recovery, HTP, ABESS, and DCIC exhibit the strongest overall performance: both false negatives and false positives decrease rapidly with $n$, and their MCC values approach one once the sample size is moderate to large. Among the remaining competitors, MCP-IC performs reasonably well, whereas SCAD-CV continues to select a non-negligible number of irrelevant variables even at the largest sample sizes, and SCAD-IC shows comparatively slower improvement in estimation and recovery accuracy. 

From a computational perspective, HTP and ABESS are consistently the fastest methods across the entire range of $n$. DCIC is more computationally demanding when $n$ is small, but its runtime decreases markedly as $n$ increases. This pattern is consistent with the fact that, under more favorable recovery regimes, the screening and pruning steps discard a substantially larger fraction of inferior candidate models before the final search stage.

\bibliographystyle{plainnat}
\bibliography{ref}
\end{document}